\documentclass{article}

\usepackage[main, final]{neurips_2026}

\usepackage[utf8]{inputenc} 
\usepackage[T1]{fontenc}    
\usepackage{hyperref}       
\usepackage{url}            
\usepackage{booktabs}       
\usepackage{amsfonts}       
\usepackage{nicefrac}       
\usepackage{microtype}      
\usepackage{xcolor}         

\usepackage{amsmath} 
\usepackage{amsthm}
\usepackage{amssymb} 
\usepackage{bm}
\usepackage{blindtext}
\usepackage{xcolor}
\usepackage{tikz} 
\usetikzlibrary{shapes.geometric,arrows.meta,positioning, shapes.misc}  
\usetikzlibrary{overlay-beamer-styles,positioning,decorations.pathreplacing,calligraphy} 
\usetikzlibrary{hobby,shapes.misc} 
\usetikzlibrary{positioning} 
\usetikzlibrary{bayesnet} 
\usepackage{booktabs} 
\usepackage{arydshln} 
\usepackage{cancel}
\usepackage{multirow}
\usepackage{enumerate}
\usepackage{enumitem}
\usepackage{mathtools}
\usepackage{natbib}

\newcommand{\bb}{\mathbb}
\newcommand{\mc}{\mathcal}
\newcommand{\supp}{\text{supp}} 
\def\indep{{\;\perp\!\!\!\perp\;}} 
\providecommand{\argdot}{{\,\vcenter{\hbox{\tiny$\bullet$}}\,}}
\usepackage{microtype}
\usepackage{graphicx}
\usepackage{subfigure}
\usepackage{booktabs} 
\usepackage{hyperref}
\usepackage{amsmath}
\usepackage{amssymb}
\usepackage{mathtools}
\usepackage{amsthm}
\theoremstyle{plain}
\newtheorem{theorem}{Theorem}[section]

\newtheorem{lemma}[theorem]{Lemma}
\newtheorem{corollary}[theorem]{Corollary}
\theoremstyle{definition}
\newtheorem{definition}[theorem]{Definition}
\newtheorem{assumption}[theorem]{Assumption}
\theoremstyle{remark}
\newtheorem{remark}[theorem]{Remark}

\usepackage[textsize=tiny]{todonotes}
\usepackage[capitalize,noabbrev]{cleveref}
\usepackage{wrapfig}
\usepackage[ruled,vlined,linesnumbered]{algorithm2e}
\SetKwInput{KwRequire}{Require}
\SetKw{Return}{Return}
\DontPrintSemicolon
\title{Debiased Counterfactual Generation via Flow Matching from Observations}

\author{
  Hugh Dance \\
  Gatsby Computational Neuroscience Unit \\
  University College London
  \And
  Johnny Xi \\
  Department of Statistics \\
  University of British Columbia
  \And
  Peter Orbanz \\
  Gatsby Computational Neuroscience Unit \\
  University College London
  \And
  Benjamin Bloem-Reddy \\
  Department of Statistics \\
  University of British Columbia
}

\begin{document}

\maketitle

\begin{abstract}
Estimating the full counterfactual distribution under an intervention is central to treatment risk assessment and {\color{black}counterfactual generation} tasks. Existing approaches model the counterfactual distribution as a standalone generative target, without exploiting its relationship to the observational data. In this work, we show that under standard assumptions, observational distributions and marginal counterfactual distributions over outcomes are tightly linked: they have identical support and tail behavior, remain statistically close under weak confounding, and share any outcome features which are invariant to confounders. These properties motivate learning counterfactual distributions not from scratch, but via a \emph{deconfounding flow} from the observational distribution. We formulate this problem via flow-matching and derive a semiparametrically efficient estimator based on a novel efficient influence function correction. We subsequently extend our estimator to target minimal-energy flows in high-dimensions, which we show can be especially simple targets between observational and counterfactual distributions. In experiments, deconfounding flows outperform existing debiased counterfactual distribution estimators, while also mitigating known failure modes of flow-based methods.
\end{abstract}

\section{Introduction} \label{sec:intro}  
Estimating how the distribution of an outcome \(Y\) would change under a hypothetical intervention on a discrete variable \(A\) is a central problem in causal
inference.  {\color{black}In economic policy analysis, for example, one may be interested in how participation in a program (\(A\)) changes not only average income (\(Y\)), but also the dispersion, quantiles, or tails of the income distribution
\citep{dinardo1996labor,firpo2009unconditional,
chernozhukov2013inference}.}  Analogous ``counterfactual'' distributions also arise in generative applications, where one may wish
to sample high-dimensional outputs \(Y\) with a specified semantic attribute \(A=a\),
while preventing generated samples from inheriting dataset-specific correlations between $A$ and
nuisance attributes \(X\) that may harm controllability or representational fairness
\citep{xu2018fairgan,choi2020fair,bianchi2023easily}.


In general, counterfactual distributions of outcomes are challenging to estimate, as one usually has access only to confounded observational samples.
An emerging framework addresses this problem by using causal identification theory and semiparametric debiasing techniques to estimate flexible density-based or sample-based models of the
counterfactual distribution from observational data,  while retaining asymptotic efficiency and consistency under general conditions \citep{kennedy2023semiparametric, melnychuk2023normalizing, martinez2024counterfactual, luedtke2025doublegen}. Such approaches have proven effective, but so far treat the counterfactual distribution as an unstructured target, relying on generic model classes that do not adapt to the geometry of the observed data.

In this work, we take a different modeling strategy: rather than using the observational distribution purely as a source of training data for counterfactual distributions, we use {\color{black} it} as the starting point for the model itself. Our approach is motivated by the fact that, for the same treatment or attribute level \(a\), the (marginal) \emph{counterfactual distribution} \(\bb P_{Y(a)}\) and the \emph{observational conditional distribution} \(\bb P_{Y\mid A=a}\) can both be represented as mixtures of the same conditional distributions, differing only in their covariate mixing distributions. As we later show, this induces identical
support and tail behavior under standard positivity assumptions, statistical closeness under weak confounding, and shared
feature-embeddings for any features which are invariant to variation
in confounders or nuisance variables \(X\). Accordingly, rather than modeling $\bb P_{Y(a)}$ {from
scratch}, we propose to learn a \emph{``deconfounding flow''} $f_a$ that transports \(\bb P_{Y\mid A=a}\) to \(\bb P_{Y(a)}\), which we instantiate via flow
matching \citep{lipman2023flow}.

\begin{figure}
    \centering
    \includegraphics[width=0.8\linewidth]{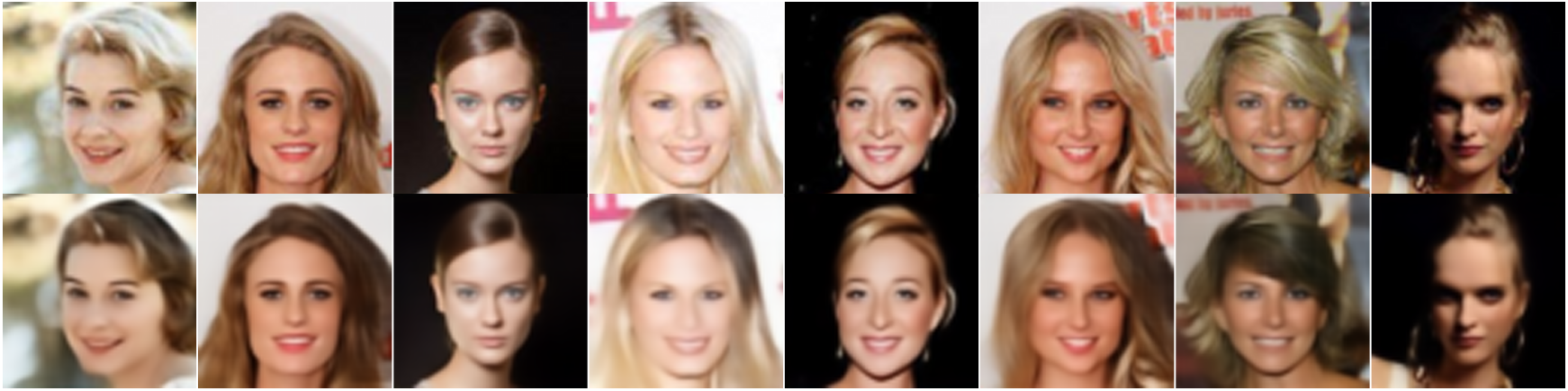}
\caption{
Top: CelebA samples from distribution
\(\bb P_{Y\mid \mathrm{Sex} = \mathrm{Female}}\), which {\color{black}over-represents} $\mathrm{HairColor}=\mathrm{Blonde}$.
Bottom: after applying the deconfounding flow, the samples are distributed according to the learned marginal counterfactual distribution \({\bb P}_{Y(\mathrm{Female})}\), under which $\mathrm{HairColor}$ follows its population distribution.
Learning \({\bb P}_{Y(\mathrm{Female})}\) via a flow from \(\bb P_{Y\mid  \mathrm{Sex} = \mathrm{Female}}\), only requires learning structured image edits rather than the entire image structure, resulting in improved distributional accuracy (see \cref{fig:celeba2}, left). Note, the samples are \emph{not} identified unit-level counterfactuals.
} 
\label{fig:celeba}
\end{figure}

While our idea is conceptually simple, it creates challenges for efficient estimation. Unlike previous counterfactual generative estimators, our learning objective depends on the data distribution through the model itself. This requires us to derive a new efficient influence function for debiased estimation, and to non-trivially extend the standard efficiency results to handle V-processes. Despite this, through careful estimator design, our method is similarly practical to implement as previous debiased methods.

As we show theoretically and empirically, learning a transport from
\(\bb P_{Y\mid A=a}\) can mitigate known learning difficulties in flow-based
models that arise from mismatched support, tail behavior, or multimodality of
the base distribution. Our approach is also very effective for
high-dimensional images \(Y\), where \(X\) may be a confounder or spurious
correlate of \(A\) that affects only structured aspects of the image, such as
color, texture, or lighting, while leaving semantic content intact.
In cases where one wishes to generate image samples $Y$ with attribute
\(A=a\) while ensuring the distribution of $X$ is unaffected by the value of $A$, a deconfounding
flow can reuse the structure in observed samples with attribute $A=a$ and learn only the image edits needed to rebalance $X$ to its population distribution. \Cref{fig:celeba} illustrates this on CelebA, where \(A=\mathrm{Sex}\) and
\(X=\mathrm{HairColor}\).

Taken together, we make the following contributions: 
\begin{enumerate}
\item We characterize several structural similarities between
$\bb P_{Y(a)}$ and $\bb P_{Y\mid A=a}$ under standard identifying assumptions: identical support and tail behavior, statistical closeness
under weak or smooth confounding, and preservation of confounder-invariant
features.
    
    \item We derive a debiased flow-matching estimator
    for learning \emph{``deconfounding flows''} (\textsc{DecFM}) from $\bb P_{Y\mid A=a}$ to $\bb P_{Y(a)}$ based on a novel influence function correction, and
    establish semiparametric efficiency and double-robustness by extending standard results to handle V-processes.

    \item We extend our estimator to target optimal-transport
        flows in higher dimensions, allowing our method to better exploit similarities between
        observational and counterfactual distributions.

    \item We demonstrate improved performance over existing methods on
    synthetic and real causal benchmarks, including high-dimensional image generation tasks.
\end{enumerate}

\section{Background and Related Work}\label{sec:background}  

All measurable spaces are standard Borel. \(\bb P_X\) denotes the law of a random
variable \(X\), and \(\supp(\bb P_X)\) its support. For measurable map \(f\), we write \(f_{\#}\bb P_X\) for
the law of \(f(X)\). We write \(\bb P_{Y\mid X=x}\) for the regular
conditional law of \(Y\) given \(X=x\), and abbreviate \(\bb P_{Y\mid X=x}\) to
\(\bb P_{Y\mid x}\) when unambiguous.

\subsection{Counterfactual Distributions under Interventions}  

Let \(A \in \mc A\) be a finite, discrete variable of interest such that $\bb P_A(a) > 0$ for each $a \in \mc A$,
\(X \in \mc X \subseteq \bb R^d\) observed covariates,
\(Y \in \mc Y \subseteq \bb R^p\) an outcome, and $\bb P_{X,A,Y}$ denote their joint observational distribution. In this work, we are interested in estimating the \emph{distribution of outcomes $Y$ that would arise under a regime that fixes \(A=a\) while keeping the distribution of \(X\) unchanged.}  

In the classical causal setting,
\(A\) is a treatment (e.g., medication), \(X\) are pre-treatment confounders, and one is interested
in the distribution of potential outcomes \(Y(a)\) that would arise if the whole
population were assigned \(A=a\). Under the usual identification assumptions---(i) consistency: \(Y=Y(A)\); (ii) conditional
exchangeability: \(Y(a)\perp\!\!\!\perp A\mid X\); and (iii) positivity:
\(\exists \epsilon >0 :  \bb P(A=a\mid X=x)\ge \epsilon\) for every \(a \in \mathcal A,\;\bb P_X\)-a.e. \citep{rubin2005causal}---this distribution is given by the formula
\begin{align}
\bb P_{Y(a)}(\cdot)
=
\textstyle\int \bb P_{Y\mid X=x,A=a}(\cdot)\, d\bb P_X(x).
\label{eq:int_dist}
\end{align}
The same distribution Eq.~\eqref{eq:int_dist} is also useful in generative applications where \(Y\) is a high-dimensional output (e.g., an image),
\(A\) is a {\color{black}discrete} semantic attribute to control, and \(X\) is a correlated nuisance
attribute. In such settings, the causal relationship between \(X\) and \(A\) may be unclear, or \(X\) may not capture all relevant confounders. Nevertheless, one may wish to generate samples with attribute \(A=a\) while ensuring that the distribution of \(X\) remains at its population distribution \(\bb P_X\), rather than changing with \(a\) according to \(\bb P_{X\mid A=a}\), often for representational fairness reasons (e.g., where \(A\) is a sensitive attribute \citep{choi2020fair}). In this case, Eq.~\eqref{eq:int_dist} remains precisely the target distribution of interest, even though it is not necessarily the distribution induced by
a causal intervention on \(A\). 

{\color{black}
We refer to Eq.~\eqref{eq:int_dist} as the (marginal)
\textbf{counterfactual distribution}, consistent with its definition in the
potential-outcomes framework \citep{rubin2005causal}. We aim to learn a sample-based
(generative) model for this distribution. For scalar outcomes, generated samples
can estimate the counterfactual CDF and derived quantities such as quantiles. In
high-dimensional applications, the samples themselves are of primary interest.
\emph{We do not claim generated samples are ``unit-level counterfactuals'' to factual
observations, as in structural causal modeling \citep{pearl2009causality}.} 
Establishing unit-level counterfactuals requires untestable cross-world assumptions, while the marginal
counterfactual distribution does not.
}

\subsection{Existing Methods for Counterfactual Distribution Estimation} \label{sec:background:counterfactual}  

A straightforward plug-in approach to learning the counterfactual distribution is
to estimate \(\bb P_{Y\mid X,A}\) and \(\bb P_X\) from observational data
with flexible models, then obtain \(\bb P_{Y(a)}\) via Eq.~\eqref{eq:int_dist}
\citep{pawlowski2020deep,sanchez2021vaca,
sanchez2022diffusion,javaloy2023causal,dance2025counterfactual}.
However, this approach can inherit bias from estimating \({\bb P}_{Y\mid X,A}\), which may be substantial when \(X\) is
high-dimensional. Another classical strategy is inverse-propensity weighting
(IPW), which reweights
observed outcomes by the inverse of the propensity score \(\pi_a(x) := \bb P(A=a|X=x)\). However, IPW is highly sensitive to weak overlap \citep{crump2009dealing} and does not yield a generative model for the counterfactual distribution.

More recently, a line of work has used techniques from semiparametric theory to directly target
\(\bb P_{Y(a)}\)
\citep{kennedy2023semiparametric,melnychuk2023normalizing,
martinez2024counterfactual,luedtke2025doublegen}. The basic idea is to specify
a density-based or sample-based model class \(\{\bb P_{\theta_a}:\theta_a\in\Theta\}\) and define
the target parameter as the projection of \(\bb P_{Y(a)}\) onto this class via a distributional
discrepancy \(\mc D\). This yields a first-order optimality condition of the form
\[
m(\theta_{a,0})
:=
\bb E\!\left[
\bb E[g(Y;\theta_{a,0})\mid A=a,X]
\right]
=0,
\]
where \(g\) is typically the gradient of the chosen loss. Based on semiparametric theory, this moment condition is estimated using its efficient influence
function (EIF). The resulting estimator for $m$ uses both an outcome nuisance,
\(\hat{\bb P}_{Y\mid A,X}\), and a propensity nuisance, \(\hat\pi_a(x)\). However, unlike plug-in or IPW estimators, the estimator is consistent if either nuisance estimator converges (i.e., is ``doubly robust'') and can achieve the asymptotic semiparametric efficiency bound \citep{kennedy2023semiparametric}. As such, this framework provides the current state-of-the-art for efficient, generative
counterfactual distribution estimation.  Further details on this framework and implementations are in Appendix \ref{app:details:debiased}.

\section{Flows from Observational to Counterfactual Distributions} \label{sec:geometry}  
Rather than modeling the counterfactual distribution {``from scratch''}, the central idea in this work is to learn it via a flow $f_a$ from the observed conditional distribution, 
\[\bb P_{Y(a)} = (f_a)_{\#} \bb P_{Y \mid a}.\]
We refer to such a flow as a ``deconfounding'' flow throughout. The key observation underlying this approach is that both $\bb P_{Y(a)}$ and $\bb P_{Y \mid a}$ are mixtures over the
same family of conditional distributions, with their mixing weights differing only through the degree of dependence between $X$ and $A$:
\begin{align*}
\bb P_{Y(a)}(\cdot)
&= \textstyle\int \bb P_{Y \mid X=x,A=a}(\cdot)\, d\bb P_X(x), \qquad
\bb P_{Y \mid a}(\cdot)
= \textstyle\int \bb P_{Y \mid X=x,A=a}(\cdot)\, d\bb P_{X \mid a}(x)\;.
\end{align*}
As we detail below, this common structure induces several shared properties and statistical similarities between source and target, with direct implications for the simplicity of deconfounding transports.

\subsection{Shared Tails and Support Under Positivity} \label{sec:geometry:structural}  

We start by establishing basic properties shared by the source and target distributions under the usual positivity assumption in causal inference, and that are  beneficial for learning flows between them. 

\begin{theorem}
[Shared Support and Tail Class]\label{thm:support-tails} Assume \(\exists\epsilon>0\) such that
\(\pi_a(x) \ge \epsilon\) for all \(a\in\mathcal A\),
\(\bb P_X\)-a.e. Then, $\bb P_{Y(a)}$ and $\bb P_{Y|a}$ share the same support and tail-class: for each $a \in \mathcal A$,
\begin{enumerate}
    \item $\supp(\mathbb{P}_{Y(a)})=\supp(\mathbb{P}_{Y\mid A=a})$
    \item For any measurable $s:\mathcal{Y}\to\mathbb{R}_+$,
    \(
\mathbb{E}[s(Y(a))]<\infty
\;\;\Longleftrightarrow\;\;
\mathbb{E}[s(Y)\mid A=a]<\infty.
\)
\end{enumerate}
\end{theorem}

\begin{figure*}
    \includegraphics[width = \textwidth]{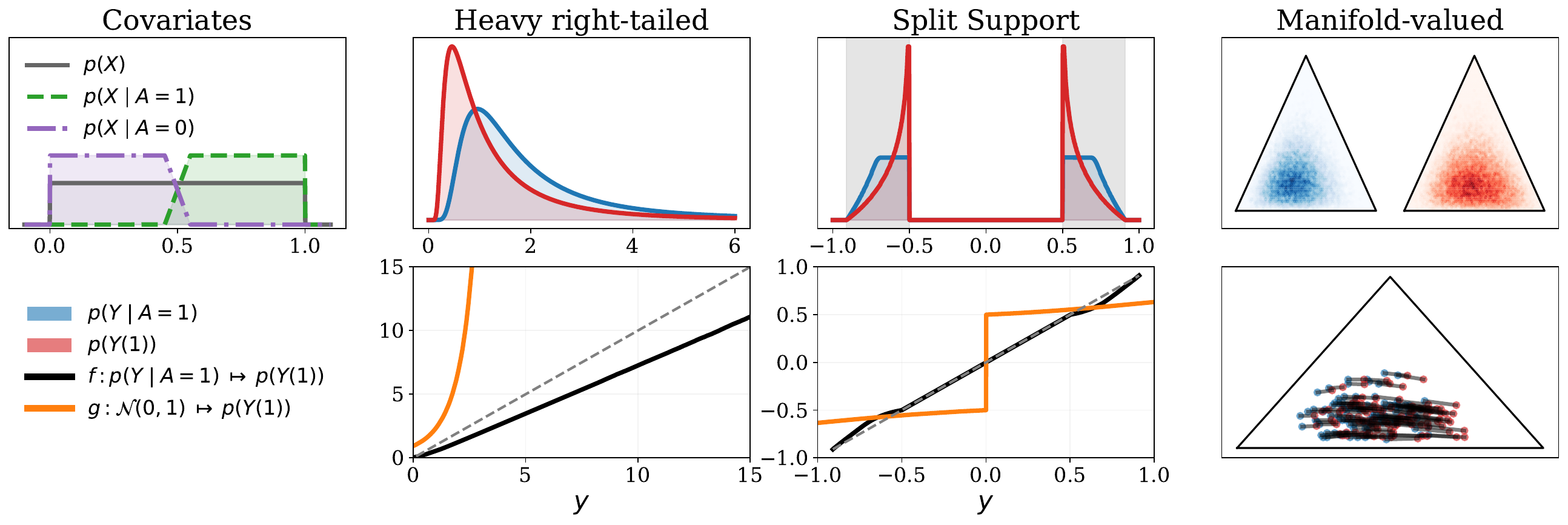}

\caption{
Top: illustrative examples of different conditional and counterfactual distributions
$(\bb P_{Y\mid A=1}, \bb P_{Y(1)})$. Although covariate overlap is weak (top left),
$\bb P_{Y\mid A=1}$ and $\bb P_{Y(1)}$ share the same support and tails (as predicted by \cref{thm:support-tails}), and are statistically much closer than the covariate mixing distributions (as predicted by Eq.~\eqref{eq:dobrushin}). Bottom: the optimal transport maps between these distributions are more
regular and closer to the identity map than transport from $\mc N(0,1)$ (there is no bijective flow from $\mc N_2(0,I)$ in the simplex-valued case).
}
\label{fig:example_distributions_1d}
\end{figure*}

 Shared support and tails are important properties to share in the context of flow-based learning, as they are precisely the properties that popular flow-based methods using smooth diffeomorphisms from a fixed base distribution struggle or are unable to adapt.

In particular, it is known that no Lipschitz transport can generate heavier tails than its base distribution \citep{jaini2020tails}, no continuous transport can split mass into disconnected support components, and no invertible transport can map between distributions supported on sets of different dimension. Since these properties need not be adapted when flowing between $\bb P_{Y|a}$ and $\bb P_{Y(a)}$, learning a deconfounding flow can therefore avoid mis-specification limitations of standard flow-based methods, including those used for counterfactual distributions \citep{melnychuk2023normalizing, luedtke2025doublegen}.
Figure~\ref{fig:example_distributions_1d} illustrates several toy examples where these
geometric constraints impede transporting from $\mc N(0,I)$, but not from $\bb P_{Y\mid a}$.

\subsection{Statistical Similarity Under Mild or Smooth Confounding}  \label{sec:geometry:min_displacement} 
 
The statistical distance between $\bb P_{Y\mid a}$ and $\bb P_{Y(a)}$ can also be bounded by a notion of confounding strength. In particular, since $\bb P_{Y\mid a}$ and $\bb P_{Y(a)}$ are mixtures of the same components, their Total Variation (TV) distance can be bounded using the \emph{Dobrushin Inequality} 
\citep{dobrushin1956central},
\begin{align}  \mathrm{TV}(& \bb P_{Y(a)}, \bb P_{Y|a}) \leq \text{sup}_{x,x' \in \mc X}\mathrm{TV}(\bb P_{Y|x,a},\bb P_{Y|x',a}) \cdot \mathrm{TV}(\bb P_{X}, \bb P_{X|a}). \label{eq:dobrushin}
\end{align}
Thus, the discrepancy between \(\bb P_{Y(a)}\) and \(\bb P_{Y\mid A=a}\)
is small whenever either the association between \(A\) and \(X\) is weak, or
the conditional outcome law varies little with \(X\) within the stratum
\(A=a\). Since each TV term is at most one, the product bound is also no
larger than either component individually. This contraction is illustrated in
\cref{fig:example_distributions_1d}, where the outcome distributions are
substantially closer than their covariate mixing distributions. Under stronger regularity, we can also bound the size of the quadratic-cost optimal transport map \citep{Santambrogio_2015} between these distributions.

\begin{theorem}[Brenier Map Displacement]
\label{thm:Wp_mixture_stability}
Fix $a\in\mathcal A$. Assume: (i) $\bb{P}_{Y|a}$ is absolutely continuous on $\bb{R}^p$; and (ii)
$
W_2\left(\bb P_{Y\mid x,a},\bb P_{Y\mid x',a}\right)
\;\le\;
L_a\, \| x - x' \|_2,
$
for every $x,x'\in\mc X$. Then the (quadratic cost) Brenier map $f_a$ transporting $\bb P_{Y\mid a}$ to
$\bb P_{Y(a)}$ exists and satisfies
\begin{align}
\| f_a - \mathrm{id}\|_{L^2(\bb P_{Y \mid a})}
\;\le\;
L_a\, W_2\left(\bb P_X,\,\bb P_{X\mid a}\right). \label{eq:transport_bound}
\end{align}
\end{theorem}
Thus, under \(W_2\)-Lipschitz outcome variation, covariate shift directly controls
the size of the optimal deconfounding transport. If 
\(\bb P_X\) and \(\bb P_{X\mid A=a}\) are close, or the conditional outcome law
varies slowly with \(X\), then \(f_a\) is near-identity. At a high level, estimating simpler functions is statistically easier than estimating complicated ones. This motivates targeting
low-energy or optimal-transport deconfounding flows, as developed in
\cref{sec:ot}.

\subsection{Shared Confounder-Invariant Features}
\label{sec:geometry:features} 

The shared mixture structure can also preserve important features of
\(Y\) that are insensitive to confounders $X$. Let \(\phi:\mc Y\to\mc Z\) be a measurable feature map. We say that
\(\phi(Y)\) is \emph{confounder-invariant} in arm \(a\) if  $\phi(Y) \perp\!\!\!\perp X \mid A=a$. 

For example, in the image-based settings discussed above,
\(X\) may affect nuisance visual attributes such as color, texture, lighting,
or background, while leaving much of the image geometry unchanged. In this
case, \(\phi\) encodes invariant structure---for
instance, object shape, pose, identity, edge placement, or class-relevant
semantic content. The following result formalizes that every
confounder-invariant feature has the same distribution under
$\bb P_{Y\mid A=a}$ and $\bb P_{Y(a)}$, and under transportability assumptions, there is a transport between them that preserves this feature pointwise.
\begin{theorem}[Feature-preserving deconfounding]\label{thm:invariant_confounding}
Fix $a\in\mathcal A$, let $\phi:\mathcal Y\to\mathcal Z$ be measurable and confounder-invariant, and assume $\pi_a(x) > 0, \bb P_X$-a.e. Then
\(\phi_{\#}\bb P_{Y(a)}=\phi_{\#}\bb P_{Y\mid A=a}\). Moreover, if there exist a jointly measurable family of maps \(\{T_z:\mc Y\to\mc Y \mid z \in \mc Z\}\) such that
\((T_z)_{\#}\bb P_{Y\mid A=a,\phi(Y)=z}=\bb P_{Y(a)\mid \phi(Y(a))=z}\) for \(\phi_{\#}\bb P_{Y\mid A=a}\)-a.e. \(z\), then there is a map $f_a:\mathcal Y\to\mathcal Y$ such that
\[
    (f_a)_{\#}\bb P_{Y\mid A=a}=\bb P_{Y(a)}
    \qquad\text{and}\qquad
    \phi(f_a(Y))=\phi(Y)
    \quad \bb P_{Y\mid A=a}\text{-a.s.}
\]
\end{theorem}
One implication of this result is that when confounders or nuisance attributes $X$ affect only
certain aspects of an image-based outcome, there exists a deconfounding flow that can in principle correct the correlated attributes while preserving shared \textit{image-level} structure or content. Thus the learning problem is closer to structured editing than generation from scratch. As we later show in experiments, deconfounding flows perform especially well in such settings, even with simple model architectures.

\section{Debiased Flow-Matching for Deconfounding Flows}
\label{sec:est}

Having established the motivation for deconfounding flows $f_a$ from
$\bb P_{Y\mid a}$ to $\bb P_{Y(a)}$, we now show how to estimate such a transport from
observational data $(X_i,A_i,Y_i)_{i=1}^n{\sim}_{iid}\bb P$ using flow matching
\citep{lipman2023flow}. Following standard flow matching, we model $f_a$ as the
time-$1$ flow induced by a velocity field
$v_{\theta_a}:\mc Y\times[0,1]\to\mc Y$, parameterized by
$\theta_a\in\bb R^q$, and aim to minimize
\begin{equation}
\mathcal L_a(\theta_a)
:=
\bb E_{(Y_0,Y_1) \sim \bb P_{Y\mid a} \otimes \bb P_{Y(a)},
\, t \sim \mathrm{Unif}(0,1)}
\Big[
\bigl\| v_{\theta_a}(\phi_t(Y_0,Y_1), t)
      - \partial_t \phi_t(Y_0,Y_1) \bigr\|^2
\Big].
\label{eq:fm_population}
\end{equation}
Here $\phi_t(y_0,y_1)=(1-t)y_0+ty_1$ is the linear interpolation path and
$\partial_t\phi_t(y_0,y_1)=y_1-y_0$. Minimizing
Eq.~\eqref{eq:fm_population} targets a velocity $v^*$ whose induced flow
transports $\bb P_{Y\mid a}$ to $\bb P_{Y(a)}$.

{\color{black}The source distribution $\bb P_{Y\mid a}$ can be consistently estimated
by the empirical distribution of observed outcomes in treatment arm $a$,
\[
\widehat{\bb P}_{Y\mid a}
:=
n_a^{-1}\sum\nolimits_{i:A_i=a}\delta_{Y_i},
\qquad
n_a:=\sum\nolimits_{i=1}^n \bm 1\{A_i=a\},
\]
and (following common practice in generative modeling) can be optionally smoothed with Gaussian noise to ensure arbitrarily many distinct samples can be generated (details in Appendix \ref{app:experiments})}. However, samples from the target distribution
$\bb P_{Y(a)}$ are generally not observed, so Eq.~\eqref{eq:fm_population} cannot be
optimized directly. We therefore follow the semiparametric debiasing approach
used by previous counterfactual generative estimators, expressing the gradient
as an observable moment condition and estimating it using an efficient
influence-function (EIF) correction. The main distinction from previous estimators is that our model
itself depends on the data distribution through $\bb P_{Y|a}$, and thus requires a new
EIF and nontrivial extensions of the standard efficiency
arguments.



\subsection{Debiased Estimation via Novel EIF Correction}
\label{sec:est:debiased} 
We now derive a debiased estimator of the gradient of Eq.~\eqref{eq:fm_population} for optimization, using the semiparametric estimation framework discussed in \cref{sec:background} and Appendix \ref{app:details:debiased}. Under Eq.~\eqref{eq:int_dist}, the gradient of Eq.~\eqref{eq:fm_population} can be expressed as the moment functional 
\begin{align}
m_a(\theta_a,\bb P)
& = \bb E_{\bb P}\!\left[
     \bb E_{ \bb P}\!\left[g_{\theta_a}(Y)\mid A=a,X\right]
   \right],
\label{eq:DecFM_moment}
\end{align}
where $g_{\theta_a}$ is an integral of the flow-matching loss score under the source $\bb P_{Y|a}$,
\[
g_{\theta_a}(y)
:=
\textstyle\int r_{\theta_a}(y, \tilde y)\,
\bb P_{Y\mid a}(d\tilde y), \qquad r_{\theta_a}(y, \tilde y)
\!:=\! \textstyle\int_0^1 \nabla_{\theta_a} \|v_{\theta_a}(\phi_t(y,\tilde y),t)
- \partial_t\phi_t(y,\tilde y)\|^2\,dt.
\]
The standard plug-in estimator for Eq.~\eqref{eq:DecFM_moment} involves estimating
\(\mu_{\theta_a}(x)\)
and averages \(\hat\mu_{\theta_a}(X_i)\) over the empirical covariate distribution. 
To remove the resulting first-order bias in this estimator that arises from estimation of $\mu_{\theta_a}$, we instead use the usual
one-step correction via an estimator of the efficient influence function (EIF) $\varphi_{\theta_a}$ of $m(\theta_a,\cdot)$. The resulting estimator is of the usual one-step form,
\begin{equation}
\hat m(\theta_a)
=
\frac1n\sum_{i=1}^n
\Bigl(
\hat\mu_{\theta_a}(a,X_i)
+
\hat\varphi_{\theta_a,a}(X_i,A_i,Y_i)
\Bigr)\;.
\label{eq:DR_m}
\end{equation}
In previous debiased generative estimators, $g_\theta$ is not a functional of the data distribution $\bb P$, and so the EIF $\varphi_{\theta_a}$ takes the usual ``AIPW'' form given in Appendix \ref{app:details:debiased}. However, in our case $g_{\theta_a}$ depends on $\bb P$ through
$\bb P_{Y\mid a}$. Thus, the EIF takes a different form and needs to be derived from scratch.

\begin{theorem}[EIF for Deconfounded Flow-Matching]
\label{thm:eif_fm} Let $\mu_{\theta_a}(x)=\bb E[g_{\theta_a}(Y)\mid A=a,X=x]$
and $\chi_{\theta_a}(y) := \bb E[\bb E[r_{\theta_a}(Y,y)\mid A=a,X]]$. Then, under the overlap and
integrability conditions in Assumption~\ref{ass:eif_regularity} in Appendix \ref{app:math}, the EIF for $m_a(\theta_a)$ is
\begin{align*}
 \varphi_{\theta_a}(\tilde x,\tilde a,\tilde y) \!= & \frac{\bm 1_{\{\tilde a\!=\! a\}}}{\pi_a(\tilde x)}
\!\bigl(g_{\theta_a}\!(\tilde y)-\mu_{\theta_a}\!(\tilde x)\bigr)
+\mu_{\theta_a}\!(\tilde x)-m_a(\theta_a)  
+ \frac{\bm 1_{\{\tilde a\!=\! a\}}}{\bb P_A(a)}
\bigl(\chi_{\theta_a}\!(\tilde y)
\!-\!\bb E[\chi_{\theta_a}\!(Y)| A\!=\! a]\bigr) 
\end{align*}
\end{theorem}

The first two terms are the usual EIF correction when \(g_{\theta_a}\) is fixed; the final \(\chi_{\theta_a}\)-terms are new and 
account for the dependence on the estimated base 
distribution \(\bb P_{Y\mid A=a}\).

\paragraph{A Practical Estimator for Optimization} The standard approach to estimating Eq.~\eqref{eq:DR_m} is to (i) estimate $\mu_{\theta_a}$ and the components of $\varphi_{\theta_a}$ on one half of the data, and (ii) substitute these estimators into Eq.~\eqref{eq:DR_m} and sum over the other half of the data. Naively,
this would require repeatedly estimating several \(\theta_a\)-dependent nuisance
components during optimization. We avoid this by outsourcing all estimation to the distributions \(\hat\pi_a, \hat{\bb P}_{Y\mid A,X}, \hat{\bb P}_A, \hat{\bb P}_{Y\mid a}\) which do not depend on $\theta_a$, and then estimating $\hat{\bb P}_A, \hat{\bb P}_{Y\mid a}$ via the empirical measures on
the \emph{same data} used to evaluate Eq.~\eqref{eq:DR_m}. In this case, 
the  resulting one-step estimator Eq.~\eqref{eq:DR_m} {\color{black}simplifies to}
\begin{align*}
\hat m(\theta_a)
\!:=\!
\frac{1}{n}\sum_{i=1}^n
\Bigg[
\frac{\bm 1_{\{A_i=a\}}}{\hat\pi_a(X_i)}
\!\!\int \!\!r_{\theta_a}(Y_i,\tilde y)
\hat{\bb P}_{Y\mid a}(d\tilde y)
 +
\Bigl(1\!-\!\frac{\bm 1_{\{A_i=a\}}}{\hat\pi_a(X_i)}\Bigr) \!\!\!
\int \!\!
r_{\theta_a}(\hat y,\tilde y)\,
\hat{\bb P}_{Y\mid a,x}(d \hat y)\hat{\bb P}_{Y\mid a}(d\tilde y)
\Bigg]
\label{eq:dr_grad_practical}
\end{align*}
The benefit of outsourcing estimation to components independent of \(\theta_a\) is that
\(\hat m(\theta_a)\) can be approximated for any \(\theta_a\) using Monte Carlo with minibatches of samples from \(\hat{\bb P}_{Y\mid A=a,X}\) and
\(\hat{\bb P}_{Y\mid A=a}\), and hence can be used for stochastic gradient descent
(see Algorithm~\ref{alg:dr-fm} in Appendix \ref{app:algs}). The benefit of using the evaluation fold for
\(\hat{\bb P}_A\) and \(\hat{\bb P}_{Y\mid A=a}\) is that the additional \(\chi_{\theta_a}\)-type nuisance terms appearing in
the full EIF cancel empirically, and so do not need estimating in practice.

\paragraph{Efficiency of Estimator}
By using the same data for estimating \(\hat{\bb P}_A\) and \(\hat{\bb P}_{Y\mid A=a}\) and evaluating the moment, one cannot use the usual sample-splitting
arguments for statistical efficiency of debiased estimators under fully non-parametric nuisance models \citep[e.g.,][]{kennedy2024semiparametric}. However, because \(\hat{\bb P}_A\) and \(\hat{\bb P}_{Y\mid A=a}\) are empirical measures, we can rewrite the estimator as a
collection of V-processes. By carefully extending the usual proof techniques to handle such processes (see Appendix \ref{app:math:lemmas}), in Corollary~\ref{corr:efficiency} we recover the usual
asymptotic expansion,
\begin{align*} & \hat m(\theta_a) - m(\theta_a) = O_{\bb{P}}(n^{-1/2}) \;+ O_{\bb{P}}\big( n^{-\frac 12} (\|\hat\pi_a-\pi_a\|_{L_2(\bb P_X)} + \Delta_{1}) + \|\hat\pi_a-\pi_a\|_{L_2(\bb P_X)}\Delta_{1} \big) \;, \end{align*} 
where $\Delta_{1} := \| W_1\!\big( \hat{\bb P}_{Y\mid a,x}, \bb P_{Y\mid a,x} \big) \|_{L_2(\bb P_X)}$, and the leading term is from the mean-zero empirical average of the EIF.
Hence, our estimator retains the usual double robustness and efficiency
properties: it is consistent if either \(\hat\pi_a\) or
\(\hat{\bb P}_{Y\mid A=a,X}\) is consistent while the other's error is bounded; and it
attains the semiparametric efficiency bound whenever
\(
\|\hat\pi_a-\pi_a\|_{L_2(\bb P_X)}\Delta_1=o_{\bb P}(n^{-1/2}). 
\) We verify double-robustness of our method empirically under nuisance misspecification in Appendix \ref{app:oracle_dr_ablation}.

\paragraph{Nuisance Model Estimation}
In the spirit of deconfounding flows, we also estimate the nuisance flow $\hat{\bb P}_{Y\mid a,x}$ via a flow from $\bb P_{Y \mid a}$, using standard flow-matching for conditional distributions. Although ${\bb P}_{Y\mid a,x}$ and $\bb P_{Y\mid a}$ do not necessarily share exactly the same transport-cost bound in \cref{thm:Wp_mixture_stability}, they do share all confounder invariant features (\cref{thm:invariant_confounding}), and also the same tails and support (\cref{thm:support-tails}) when those properties are not affected by $X$. ${\bb P}_{Y\mid a,x}$ and $\bb P_{Y|a}$ are also closer in distance than the supremum term of the Dobrushin bound in Eq.~\eqref{eq:dobrushin}. The propensity model $\hat\pi_a$ can be estimated with any probabilistic classifier. Further nuisance estimation details are in Appendix \ref{app:experiments}.

\subsection{Extension for Targeting OT Flows In Higher Dimensions} \label{sec:ot} 
In one dimension ($\mc Y\subset\bb R$), the population objective
Eq.~\eqref{eq:fm_population} targets the velocity $v^*$ 
whose induced time-1 flow  coincides with the quadratic cost OT map $f_a$ studied in \cref{sec:geometry:min_displacement}. Moreover, its kinetic energy $E(v^*)$ corresponds to the transport cost,
\[
E(v^*) := \textstyle\int_0^1\!\!\int \|v_t(y)\|^2\,\bb P_t(dy)\,dt = \|f_a-\mathrm{id}\|^2_{L_2(\bb P_{Y\mid a})}.
\]
This energy is minimal among all velocities with flows that transport $P_0:= \bb P_{Y\mid a}$ to $P_1:= \bb P_{Y(a)}$.
Thus, a ``near-identity'' deconfounding flow (e.g., under structured or mild confounding) implies a ``near-zero'' velocity, and regularity properties of $f_a$ are known to directly carry over to $v^*$. For example, if $f_a$ is bi-Lipschitz, then $f_t$ is bi-Lipschitz and $v_t$ is Lipschitz  \citep{ambrosio2005gradient}.

In higher dimensions, however, this equivalence only holds when replacing the
independent coupling $\bb P_{Y\mid a}\otimes\bb P_{Y(a)}$ in
Eq.~\eqref{eq:fm_population} by the {OT}
coupling $\gamma$ between $\bb P_{Y\mid a}$ and $\bb P_{Y(a)}$. 
\begin{align}
\hspace{-5pt}\mathcal L_a^{\gamma}(\theta_a)
:=\!\!\!\!\!
 \underset{(Y_0,Y_1)\sim \gamma}{\bb E}
\!
\bigl\| v_{\theta_a}(\phi(Y_0,Y_1), t)
      - \partial_t \phi(Y_0,Y_1) \bigr\|^2 \label{eq:pop_fm_coupling}
\end{align}
Therefore, we extend the estimator of \cref{sec:est:debiased} to  target such OT flows in higher dimensions, as an additional implementation option.
To do so, we write the coupling $\gamma$ that defines Eq.~\eqref{eq:pop_fm_coupling} as $\bb P_{Y(a)} \otimes \gamma_{0 \mid 1}$, with $ \gamma_{0 \mid 1}$ the conditional distribution of $Y_0 \mid Y_1$. For any fixed  $\gamma_{0 \mid 1}$, we derive a debiased estimator of $\nabla_{\theta_a}\mathcal L_a^\gamma(\theta_a)$ using the equivalent EIF, resulting in an analogous gradient estimator for optimization (see Eq.~\eqref{eq:ot_gradient} in Appendix \ref{app:experiments}).
This estimator recovers the independent-coupling estimator of \cref{sec:est:debiased}  by taking
\(\gamma_{0\mid 1}(\cdot\mid y)=\hat{\bb P}_{Y\mid a}\). Analogous to the usual minibatch-OT flow matching \citep{tong2023improving}, we instead propose to estimate $\gamma_{0\mid 1}$ via an entropic-OT coupling between
\(\hat{\bb P}_{Y\mid a}\) and a plug-in estimate of \(\bb P_{Y(a)}\), recomputed on
minibatches using Sinkhorn's algorithm \citep{tong2023improving}. To reduce sensitivity to plug-in misspecification, we optionally tilt samples from \(\hat{\bb P}_{Y(a)}\) to match doubly-robust moment estimates. Details are in Appendix \ref{app:experiments} and Algorithm~\ref{alg:dr-ot-fm}.

We note that estimating the OT coupling introduces an
additional source of statistical error not corrected at first order by the fixed-\(\gamma\) EIF correction. However, in exchange, we get
a simpler geometric target through the form of straighter, lower-energy paths, which induce near-identity flows under mild confounding. In high-dimensional experiments, we find this estimator improves distributional accuracy, demonstrating that this trade-off is worthwhile. We also verify that the method retains empirical double-robustness under nuisance misspecification in Appendix \ref{app:oracle_dr_ablation}, and empirically recovers shorter, smoother transports in Appendix \ref{app:flow_complexity}.

\section{Experiments}
\label{sec:experiments}  

We now test deconfounding flows (\textsc{DecFM}) across several experiments. We first use controlled experiments to probe the structural advantages
predicted by Section~3, before comparing DecFM with established
counterfactual-distribution estimators and evaluating it on CelebA. {\color{black}As is common in counterfactual distribution estimation evaluation
\citep{kennedy2023semiparametric,melnychuk2023normalizing,
martinez2024counterfactual,luedtke2025doublegen}, our main-text
experiments primarily use binary treatments. Appendix~\ref{app:details:threearm}
additionally evaluates DecFM in a non-binary treatment setting.} Design and implementation details for all experiments are in Appendix \ref{app:experiments}. Code for our method can be found at the GitHub repository \href{https://github.com/HWDance/DeconfoundingFlowMatching}{github.com/hwdance/DeconfoundingFlowMatching}.

\subsection{Observational vs. Gaussian Base Distribution under Different Geometries}\label{sec:exp:1dablation}  
We start with a simple ablation of the effect of flowing from $\bb P_{Y|a}$ compared to flowing from a fixed $\mc N(0,I)$ base distribution under different outcome geometries, to demonstrate the usefulness of the shared tails and support of deconfounding flows (i.e., \cref{thm:support-tails}). We draw covariates $X \sim \mathrm{Unif}[0,1]$, treatment 
\(
A \mid X \sim \mathrm{Bern}\bigl(\sigma(5(X-0.5))\bigr)
\), and then $Y \sim \bb P_{Y\mid X,A}$ using:
(i) a simple conditional Gaussian with covariate-dependent mean, where the Gaussian
base is well matched to the target distribution;
(ii) an inverse-Gamma distribution with covariate-dependent tails (which do not match those of a Gaussian); and
(iii) a mixture of uniform distributions with disconnected support (also mis-matched to the connected support of a Gaussian).  For both flows we use the debiased estimator of \cref{sec:est:debiased} (instantiated via
{Algorithm~\ref{alg:dr-fm}} in Appendix \ref{app:algs}) on $n=1000$ samples from each design, and in all cases parameterize the velocity using a MLP.

\begin{wraptable}[8]{r}{0.5\columnwidth}
\vspace{-1em}
\centering
\setlength{\tabcolsep}{3pt}
\footnotesize
\resizebox{\linewidth}{!}{%
\begin{tabular}{lcc}
\toprule
Outcomes $\downarrow$ / Base $\to$ & $\bb P_{Y\mid a}$ & $\mc N(0,I)$ \\
\midrule
Gaussian &
0.15 $\pm$ 0.05 &
\textbf{0.13 $\pm$ 0.05} \\

Inverse Gamma &
\textbf{0.71 $\pm$ 0.32} &
1.04 $\pm$ 0.46 \\

Mixture of Uniform &
\textbf{0.04 $\pm$ 0.01} &
0.08 $\pm$ 0.01 \\
\bottomrule
\end{tabular}
}
\caption{
$W_1$ error (mean $\pm$ sd, 50 trials) in estimated $\bb P_{Y(a)}$ via flow from
$\bb P_{Y\mid a}$ versus $\mc N(0,I)$.
}
\label{tab:1dablation}
\end{wraptable} 
\textbf{Deconfounding flows avoid tail and support mismatch.} Table~\ref{tab:1dablation} reports the $W_1$ error, averaged over $a \in \{0,1\}$ and 50 trials. In the unimodal Gaussian setting, flowing from $\bb P_{Y|a}$ does not degrade performance when the $\mc N(0,I)$ base is well-suited. In contrast, for the heavy-tailed and split-support outcomes, 
we see a substantial improvement in flowing from $\bb P_{Y\mid a}$, as predicted in \cref{sec:geometry}. 

\subsection{Deconfounding Generation Biases in ColorMNIST}\label{sec:exp:mnist} 

We next study DecFM on ColorMNIST (CMNIST), which provides a controlled high-dimensional setting for us to evaluate (i) whether our approach performs relatively better as $\bb P_{X\mid a}$ approaches $\bb P_X$ (\cref{thm:Wp_mixture_stability}), and (ii) whether structure unaffected by $X$ is preserved by the flow (\cref{thm:invariant_confounding}). We generate \(n=20{,}000\) images from two digit classes, \(A\in\{1,6\}\), with uniformly varying foreground colour
\(X=R/(R+B)\) along the red--blue axis. Thus, $X$ does not affect the geometry of the image. Confounding is induced by drawing \(A\) conditional on colour according to \(\bb P(A=6\mid X=x)=\sigma(w(x-1/2))\), where \(w\) is varied to control confounding strength. We aim to generate samples from the counterfactual distribution \(\bb P_{Y(a)}\), i.e. images with digit \(a\) under the color distribution \(\bb P_X\). We implement deconfounding flows with the independent coupling (DecFM-I) and EOT coupling (DecFM-EOT) using U-Net velocity fields of two widths \citep{ronneberger2015u} and compare against equivalent flows estimated from \(\mc N(0,I)\) with the usual independent coupling.

\textbf{Deconfounding flows improve accuracy and learn simpler transports.}
\cref{fig:cmnist} (left) shows SW2 error against the true counterfactual
distribution (averaged over three seeds and both arms), for varying confounding
strength and architecture. Both deconfounding flows outperform Gaussian-base
flow matching across confounding strengths and U-Net widths, with the largest
relative gains as \(\bb P_{X\mid a}\) approaches \(\bb P_X\), consistent with
\cref{thm:Wp_mixture_stability}. In Appendix~\ref{app:flow_complexity}, we
show that, even in the strongest confounding case, the learned flows from \(\bb P_{Y\mid a}\) are substantially shorter and
smoother than those from \(\mc N(0,I)\), supporting the hypothesis that
deconfounding flows improve accuracy by targeting simpler transports.
DecFM-EOT improves performance further, consistent with the simpler trajectories
it induces (as shown in Appendix~\ref{app:flow_complexity}, as well as in Appendix~\ref{app:details:ablation} for a 2d synthetic setting). In Appendix~\ref{app:details:threearm} we extend CMNIST to more than two
digits, showing similar performance gains.

\textbf{Deconfounding rebalances \(X\) while preserving invariant structure.}
In the strongest-confounding setting, \cref{fig:cmnist} (middle) shows that the
pixel-estimated colour distributions are successfully rebalanced toward
\(\bb P_X\), while \cref{fig:cmnist} (right) shows that the learned flow
preserves image geometry and makes little modification when no confounding is
present, consistent with \cref{thm:invariant_confounding}.


\begin{figure*}
    \centering
    \includegraphics[width=0.53\linewidth]{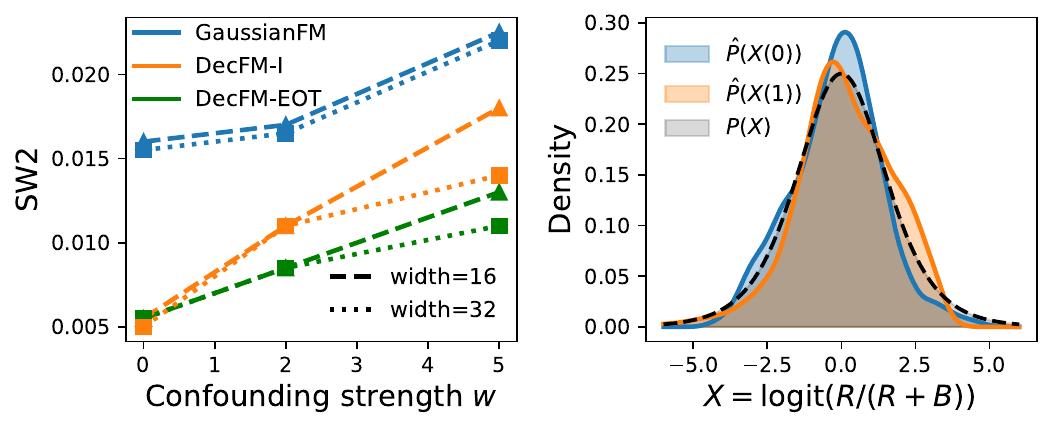}
    \hspace{2pt}
    \raisebox{10pt}{
    \includegraphics[width=0.44\linewidth]{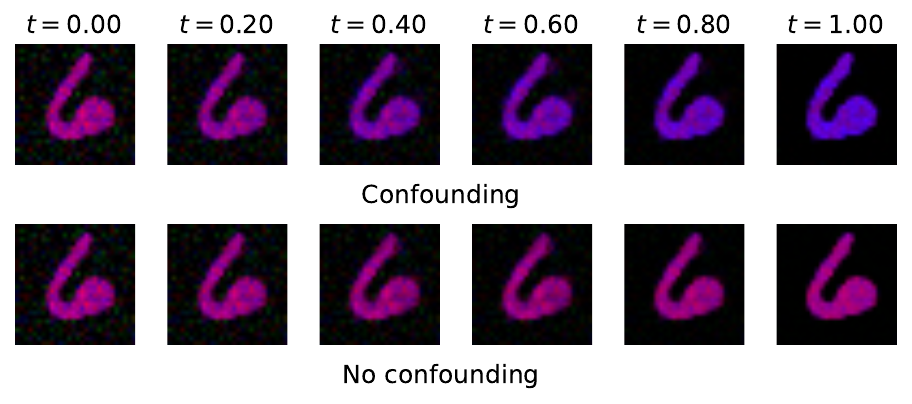}
    }
\caption{{\color{black}\textbf{ColorMNIST}}. \emph{Left:} Avg $SW_2(\bb P_{Y(a)}, \bb P_{\hat \theta_a})$ over 3 seeds vs.\ confounding
        strength for deconfounding flows (Ind + EOT coupling) and \ flow-matching from $\mc N(0,I)$, for different U-Net widths. \emph{Middle:} Learned color distributions $\hat {\bb P}_{X(0)}$ (digit 0) and $\hat {\bb P}_{X(1)}$ (digit 6) by DecFM-EOT vs. true $\bb P(X)$, under strongest confounding design ($w=5$).
        \emph{Right:} Same DecFM-EOT sample trajectory under high (top) and no (bottom) confounding.
        }    \label{fig:cmnist}
\end{figure*}

\subsection{Counterfactual Distribution Estimation on Classic Causal Benchmarks}\label{sec:exp:benchmark}  







Having isolated the benefits of deconfounding flows in controlled settings, we now evaluate the relative performance and runtime of deconfounding flows against established debiased estimators for counterfactual distributions (defined in Appendix \ref{app:details:debiased} and \ref{appx:baselines}) on causal benchmark datasets. Specifically, we consider three datasets with real covariates \(X\) and treatment
\(A\): TWINS~\citep{louizos2017causal}, 401(k)~\citep{engen2000effects},
and ACIC 2019~\citep{acic2019challenge}. We synthetically generate outcomes
using different outcome mechanisms per each dataset (Gaussian, SplitSupport, HeavyTails) described in Appendix \ref{app:experiments}. We also consider 10 ACIC 2016 challenge
datasets~\citep{dorie2019automated}, where true potential outcomes are provided. As all outcomes are scalar, we evaluate full-distribution recovery for all methods by their integrated absolute CDF error (i.e. $W_1$-error); we use only DecFM-I, as no OT extension is needed in one dimension.

\begin{wraptable}[12]{r}{0.60\columnwidth} 
\centering
\vspace{-0.9em}
\setlength{\tabcolsep}{2pt}
\renewcommand{\arraystretch}{0.9}
\scriptsize
\resizebox{\linewidth}{!}{%
\begin{tabular}{l|ccc|c}
\toprule
\textbf{Dataset} $\rightarrow$
 & \textbf{ACIC}
 & \textbf{401k}
 & \textbf{TWINS}
 & \textbf{ACIC16} \\
\midrule
Outcomes $\rightarrow$
 & Gaussian & SplitSupport & HeavyTails
 & Challenge \\ 
\midrule

AIPW-KDE
 & 0.14 $\pm$ 0.01
 & 0.50 $\pm$ 0.01
 & 0.28 $\pm$ 0.02
 & 0.73 $\pm$ 0.14 \\

AIPW-TSE
 & 0.07 $\pm$ 0.02
 & 0.26 $\pm$ 0.06
 & 2.75 $\pm$ 1.82
 & 0.46 $\pm$ 0.38 \\

INF
 & 0.10 $\pm$ 0.03
 & 0.06 $\pm$ 0.07
 & 0.42 $\pm$ 0.05
 & 0.50 $\pm$ 0.11 \\

DR-MKSD
 & 0.18 $\pm$ 0.01
 & 0.08 $\pm$ 0.04
 & 0.34 $\pm$ 0.07
 & 1.04 $\pm$ 0.68 \\

DoubleGen
 & 0.05 $\pm$ 0.02
 & 0.07 $\pm$ 0.01
 & 0.36 $\pm$ 0.14
 & 0.62 $\pm$ 0.24 \\

\midrule
\textbf{DecFM}
 & \textbf{0.05 $\pm$ 0.02}
 & \textbf{0.04 $\pm$ 0.01}
 & \textbf{0.25 $\pm$ 0.08}
 & \textbf{0.35 $\pm$ 0.15} \\
\bottomrule
\end{tabular}
}
\caption{
Mean $\pm$ SD integrated CDF error ($W_1$), 10 seeds.
}
\label{tab:causal_benchmarks}
\end{wraptable} 
\textbf{Deconfounding flows outperform existing methods.}
\cref{tab:causal_benchmarks} reports \(W_1\) error over 10 random seeds for the synthetic experiments and across the 10 ACIC 2016 challenge datasets. DecFM achieves the lowest error across the designs, with competing methods occasionally matching but not exceeding its accuracy, and the largest gains occurring under more challenging outcome geometries. Since DoubleGen \citep{luedtke2025doublegen} was also implemented with flow-matching, it is clear that in 3/4 cases, our performance improvement is a consequence of using \emph{deconfounding} flows specifically, rather than simply due to flow-matching outperforming other modeling approaches. 

\textbf{Runtime and Complexity.}
Appendix~\ref{app:computational_cost} presents runtime and complexity comparisons. In short, DecFM has complexity no worse than other methods ($O(ND)$ per-epoch) and achieves runtime comparable to DoubleGen, faster than INF and DR-MKSD, but slower than closed-form KDE/TSE. OT increases the cost to $O(NBD)$ per epoch, with only modest wall-clock overhead in practice. 






\subsection{Sampling from the Counterfactual Distribution on CelebA}  
\begin{figure}[t]
    \centering
\begin{minipage}[c]{0.49\linewidth}
    \centering
    \small
    \setlength{\tabcolsep}{4pt}
\begin{tabular}{ccccc}
\toprule
$a$ & Method & SW2 (base) & SW2 (target) \\
\midrule
\multirow{3}{*}{$0$}
    & Gaussian-FM & 0.575 $\pm$ 0.016 & 0.126 $\pm$ 0.004 \\
    & DecFM-I      & 0.066 $\pm$ 0.005 & 0.035 $\pm$ 0.001 \\
    & DecFM-EOT    & 0.066 $\pm$ 0.005 & 0.032 $\pm$ 0.002 \\
\midrule
\multirow{3}{*}{$1$}
    & Gaussian-FM & 0.547 $\pm$ 0.002 & 0.149 $\pm$ 0.007 \\
    & DecFM-I      & 0.059 $\pm$ 0.004 & 0.036 $\pm$ 0.007 \\
    & DecFM-EOT    & 0.059 $\pm$ 0.004 & 0.035 $\pm$ 0.004 \\
\bottomrule
\end{tabular}
\end{minipage}
    \begin{minipage}[c]{0.5\linewidth}
        \centering
        \includegraphics[width=\linewidth]{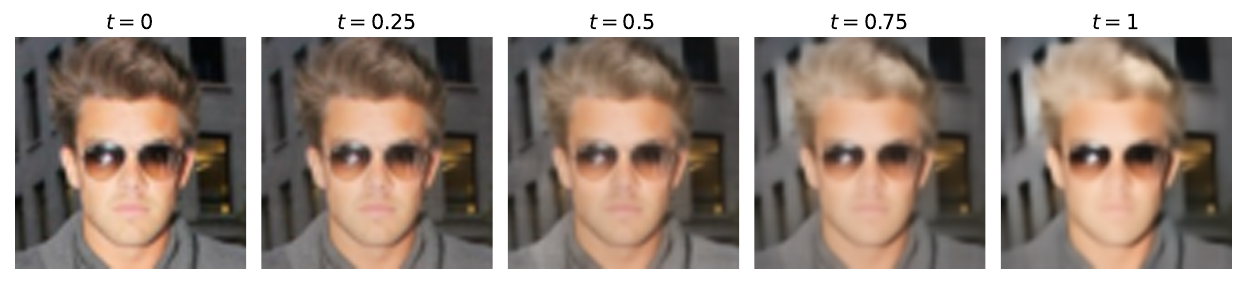}


        \includegraphics[width=\linewidth]{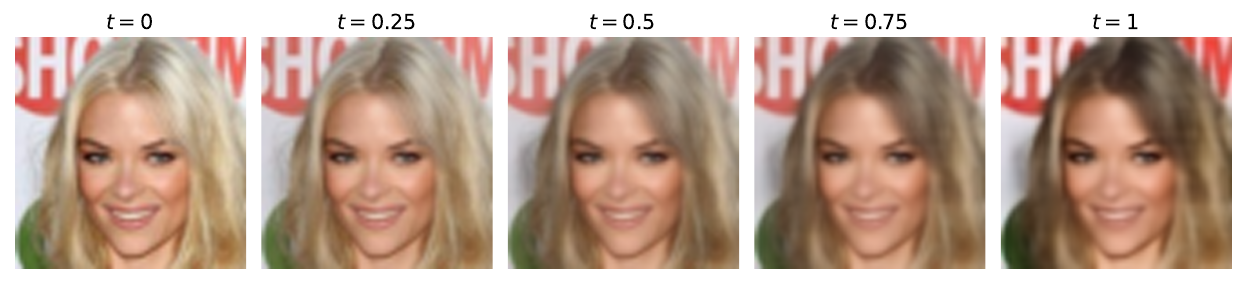}
    \end{minipage}
    \caption{
    {\color{black}\textbf{CelebA attribute rebalancing}}.
Left: Mean \(\pm\) SD Sliced Wasserstein--2 results over 3 seeds. SW2 (base) = distance from the flow source distribution to the
target distribution, SW2 (target) = distance after applying the learned flow.
Right: trajectories from DecFM-EOT.
    } 
    \label{fig:celeba2}
\end{figure}

We finally evaluate deconfounding flows on CelebA, where the aim is to generate images with fixed values of \(A=\mathrm{Sex}\) under a regime in which \(X=\mathrm{HairColor}\) follows its population distribution \(\bb P_X\), rather than its
group-specific distribution \(\bb P_{X\mid A=a}\). We choose hair color as a visually salient but comparatively low-stakes nuisance attribute for evaluation. We construct a training set with a strong association between \(A\) and \(X\) by sampling images without replacement from the four \((A,X)\) cells according to prescribed proportions (i.e., $\bb P(A=1|X=1) = 0.1, \bb P(A=1|X=0)=0.6$, $\bb P(X=1) = 0.3$), yielding \(n\approx20{,}000\) training images due to the limiting size of the smallest cell.  

\textbf{Deconfounding flows make targeted, content-preserving edits.}
As \cref{fig:celeba2} shows, deconfounding flows substantially improve over Gaussian flow matching, with DecFM-EOT performing best. Before applying the flow, \(\bb P_{Y\mid A=a}\) is already a much better image-space source than Gaussian noise (SW2 (base)), despite having the wrong hair-color mixture. Applying DecFM then further reduces SW2 by roughly \(40\)--\(50\%\), showing that the flow learns a meaningful adjustment to the base distribution. Qualitatively, the flow makes structured hair-colour edits while preserving much of the original image content (see \cref{fig:celeba2} (right), and also \cref{fig:celeba}).


\section{Conclusion} \label{sec:impact}  

We introduced \emph{deconfounding flows}, which estimate \(\bb P_{Y(a)}\) by transporting \(\bb P_{Y\mid A=a}\) rather than a fixed base distribution. Because both distributions mix the same conditionals \(\bb P_{Y\mid A=a,X}\) over different covariate distributions, we showed they share support and tails under positivity, are statistically close under mild or smooth confounding, and share confounder-invariant features. We derived a debiased flow-matching estimator (\textsc{DecFM}) based on a new efficient influence function that accounts for the estimated source distribution, and showed it is doubly robust and semiparametrically efficient by extending standard results to V-processes. We extended our method using entropic-OT minibatching to further target straighter, lower-energy transports in high dimensions. Empirically, deconfounding flows matched or outperformed existing debiased estimators on classic causal benchmarks, and improved on Gaussian-base flow matching on ColorMNIST and CelebA, where deconfounding flows learn structured, content-preserving image edits.

Deconfounding flows inherit the usual limitations of observational causal inference: they rely on no unobserved confounding, positivity, and sufficiently accurate nuisance estimation. Our efficiency guarantees also cover only the independent-coupling estimator, not the additional error from estimating the OT coupling. More broadly, the gains depend on \(\bb P_{Y\mid A=a}\) being a useful geometric anchor for \(\bb P_{Y(a)}\). Under weak overlap or highly variable outcome mechanisms, the required transport may remain hard to learn. Future work could develop diagnostics for this regime and characterize which distributional similarities hold in other causal settings, such as under the front-door criterion.

\paragraph{Societal Impact}Beyond the methodological benefits established in this paper, deconfounding
flows may also have positive societal impact by helping mitigate harmful
generative biases, for example by rebalancing generated samples with respect to
sensitive or socially salient attributes.  As with all generative methods,
there is also potential for misuse, for example by producing misleading or
manipulated synthetic content, and so the usual safeguards, provenance tools, and deployment
constraints required for generative models remain necessary here. 

\begin{ack}
HD and PO are supported by the Gatsby Charitable Foundation. BBR and JX acknowledge the support of the Natural Sciences and Engineering Research Council of Canada (NSERC): RGPIN2020-04995, RGPAS-2020-00095. 
\end{ack}

\bibliographystyle{apalike}
\bibliography{references}


\appendix

\onecolumn
\section{Additional Analysis}

\subsection{Background on Debiased Estimation of Counterfactual Distributions}\label{app:details:debiased}

\begin{table}[ht]
\centering
\resizebox{\textwidth}{!}{%
\begin{tabular}{lllll}
\toprule
Method & Model Class & Discrepancy $\mc D$
& Moment Function $g_\theta$
& Nuisance Outcome $\mu$ \\
\midrule
\citet{kennedy2023semiparametric}
& Exponential family
& KL divergence
& $-\partial_\theta \log p_\theta(y)$
& Score regression \\

\citet{kennedy2023semiparametric}
& Truncated series
& $L^2$ density error
& $\partial_\theta p_\theta(y)$
& Basis regression \\

\citet{melnychuk2023normalizing}
& Normalizing flows
& KL divergence
& $-\partial_\theta \log p_\theta(y)$
& Normalizing Flow \\

\citet{martinez2024counterfactual}
& Energy-based models
& Kernel Stein
& $\mc T_{p_\theta} f(y)$
& RKHS regression \\

\citet{luedtke2025doublegen}
& Continuous flows
& Flow matching
& $\bb E_{ty_t, y_0 , t \mid y}[\nabla_\theta v_\theta(y_t,t)^\top
\bigl(v_\theta(y_t,t)-u^\star(y_0,y)\bigr)]$
& Continuous flow \\

\bottomrule
\end{tabular}
}
\caption{
Different generative estimators of counterfactual distributions via debiased estimation of moment equations.}
\label{tab:existing_methods}
\end{table}

Recent literature has combined generative modeling with tools from semiparametric debiased estimation \citep{kennedy2024semiparametric}.
The starting point is to specify a generative model class
\[
\mathcal M := \{\mathbb P_{\theta_a} : \theta_a \in \Theta\},
\]
and define the target parameter $\theta_{a,0}$ as the projection of the
true distribution $\mathbb P_{Y(a)}$ onto this class via a discrepancy
$\mathcal D$:
\begin{equation}
\theta_{a,0} \in \arg\min_{\theta_a \in \Theta}
\mathcal D\left(\mathbb P_{Y(a)}, \mathbb P_{\theta_a}\right).
\label{eq:poprisk}
\end{equation}
By choosing $\mathcal D$ such that
\[
\arg\min_{\theta_a \in \Theta}
\mathcal D\left(\mathbb P_{Y(a)}, \mathbb P_{\theta_a}\right)
=
\arg\min_{\theta_a \in \Theta}
\mathbb E\left[\ell(Y(a); \theta_a)\right]
\]
for a suitable loss $\ell$, first-order optimality of $\theta_{a,0}$ yields a
population estimating equation
\[
m(\theta_{a,0})
:=
\mathbb E\left[g_{\theta_{a,0}}(Y(a))\right]
= 0,
\quad 
g_{\theta_a}(y) := \nabla_{\theta_a}\ell(y; \theta_a).
\]
For example, when $\mathcal D$ is the Kullback--Leibler divergence,
$g(\theta_a,y) = -\partial_{\theta_a}\log p_{\theta_a}(y)$,
where $p_{\theta_a}$ denotes the density of $\mathbb P_{\theta_a}$.
Table~\ref{tab:existing_methods} summarizes the model
classes, discrepancies, and induced moment functions used in the literature. We note that while the recent flow-matching and score matching estimators used in this framework are not gradients of divergences, they can still be characterized as first-order optimality conditions, giving rise to equivalent moment conditions. 

From here, one uses semiparametric theory to derive efficient estimators for $\theta_{a,0}$. The basic idea is to write $m_a(\theta_a)$ as a functional
of the observational distribution $\mathbb P:= \bb P_{Y,A,X}$, 
\[
m(\theta_a,\mathbb P)
=
\mathbb E_{\mathbb P}\left[
\mathbb E_{\mathbb P}\left[g_{\theta_a}(Y)\mid A=a,X\right]
\right],
\]

Assuming pathwise differentiability of $m$, one can take a first-order von Mises expansion
around estimator $\hat {\mathbb P}$,
\[
m(\theta_a,\mathbb P)
=
m(\theta_a,\hat{\mathbb P})
+
\int \varphi_a(o)\, d(\mathbb P-\hat{\mathbb P})(o)
+
R_2(\hat{\mathbb P},\mathbb P).
\]
Here $o:=(x,a,y)$, $\varphi_{\theta_a,a}$ is called the
\emph{efficient influence function} (EIF) of $m(\theta_a,\cdot)$ and $R_2$ is a second-order remainder term. The EIF is the unique, mean-zero function with minimum asymptotic variance that characterizes the first-order sensitivity of $m(\theta_a,\cdot)$ to perturbations of $\mathbb P$ \citep{kennedy2024semiparametric}. Operationally, it captures the leading (first-order) error incurred when approximating $m(\theta_a,\mathbb P)$ by $m(\theta_a,\hat{\mathbb P})$. In the present setting, the EIF takes the explicit augmented inverse propensity weighted (AIPW) form
\[
\varphi_a(X,A,Y)
=
\frac{\mathbf 1\{A=a\}}{\pi_a(X)}
\bigl(g(\theta_a,Y)-\mu_{\theta_a}(a,X)\bigr)
+
\mu_{\theta_a}(a,X)
-
m_a(\theta_a),
\]
where $\pi_a(X)=\mathbb P(A=a\mid X)$ and
$\mu_{\theta_a}(a,X)=\mathbb E[g(\theta_a,Y)\mid A=a,X]$ are called `nuisance models'.

Using this expansion, one constructs a `debiased' estimator for $m(\theta_a,\bb P)$ in two steps. (i) form the plug-in estimator
\begin{align*}
\hat m_{\mathrm{PI}}(\theta_a)
:=
\frac{1}{n}\sum_{i=1}^n \hat\mu_{\theta_a}(a,X_i),
\end{align*}
using an estimator $\hat\mu_{\theta_a}$ for
$\mu_{\theta_a}$, and (ii) remove its first-order bias by adding an empirical approximation of the
EIF term, using the nuisance estimators $\hat \mu_{\theta_a}$ and $\hat \pi_a$. This yields the so-called ``one-step''or ``doubly-robust'' estimator
\begin{align}
\hat m_{\mathrm{DR}}(\theta_a)
:=
\hat m_{\mathrm{PI}}(\theta_a)
+
\frac{1}{n}\sum_{i=1}^n \hat\varphi_{\theta_a,a}(X_i,A_i,Y_i) \label{eq:onestep}
\end{align}
In practice, one finds the target parameter that solves $\hat m_{\mathrm{DR}}(\theta_a)=0$ either in closed form (when available) \citep{kennedy2023semiparametric}, or by minimizing an antiderivative
of $\hat m_{\mathrm{DR}}$ \citep{melnychuk2023normalizing, martinez2024counterfactual,luedtke2025doublegen}. For statistical efficiency reasons, one usually trains $\hat \mu$ and $\hat \pi$ on a separate fold of data than $(X_i,A_i,Y_i)_{i=1}^n$, or uses ``cross-fitting'' estimators which combine multiple sample split estimators to make use of the full data for nuisance training and moment evaluation. 

\paragraph{Fast Rates, Efficiency, and Double Robustness of the Debiased Estimator}
The naive plug-in estimator
$\hat m_{\mathrm{PI}}(\theta_a)=n^{-1}\sum_i \hat\mu_{\theta_a}(a,X_i)$
inherits its convergence rate directly from the nuisance estimator
$\hat\mu_{\theta_a}$, and is therefore typically limited by the accuracy with which
$\mu_{\theta_a}$ can be estimated.

In contrast, the estimation error of the debiased (or doubly-robust) estimator
$\hat m_{\mathrm{DR}}(\theta_a)$ admits an expansion with a second-order remainder,
whose leading term depends on the \emph{product} of the estimation errors of the
nuisance models $\hat\mu_{\theta_a}$ and $\hat\pi_a$. As a consequence, under standard
regularity conditions (e.g.\ Donsker classes) or under cross-fitting,
\[
\hat m_{\mathrm{DR}}(\theta_a) - m_a(\theta_a)
= \mc O_p\left(
\|\hat\mu_{\theta_a}-\mu_{\theta_a}\|\,
\|\hat\pi_a-\pi_a\|
\right),
\]
up to empirical process terms of order $n^{-1/2}$.

This structure yields two important properties.
First, if both nuisance estimators converge at rates faster than
$\mc O(n^{-1/4})$, the remainder term is $o_p(n^{-1/2})$, and the resulting estimator
is known to achieve $\sqrt n$-consistency and semiparametric efficiency \citep{kennedy2024semiparametric}.
Second, the estimator is \emph{doubly robust} in the sense that consistency of
$\hat m_{\mathrm{DR}}(\theta_a)$ only requires one of the nuisance models
$\hat\mu_{\theta_a}$ or $\hat\pi_a$ to be consistently estimated.

\subsection{Additional Ablation of Independent vs OT Coupling}\label{app:details:ablation}

\begin{wrapfigure}[18]{r}{0.52\linewidth} 
    \centering
    \includegraphics[width=\linewidth]{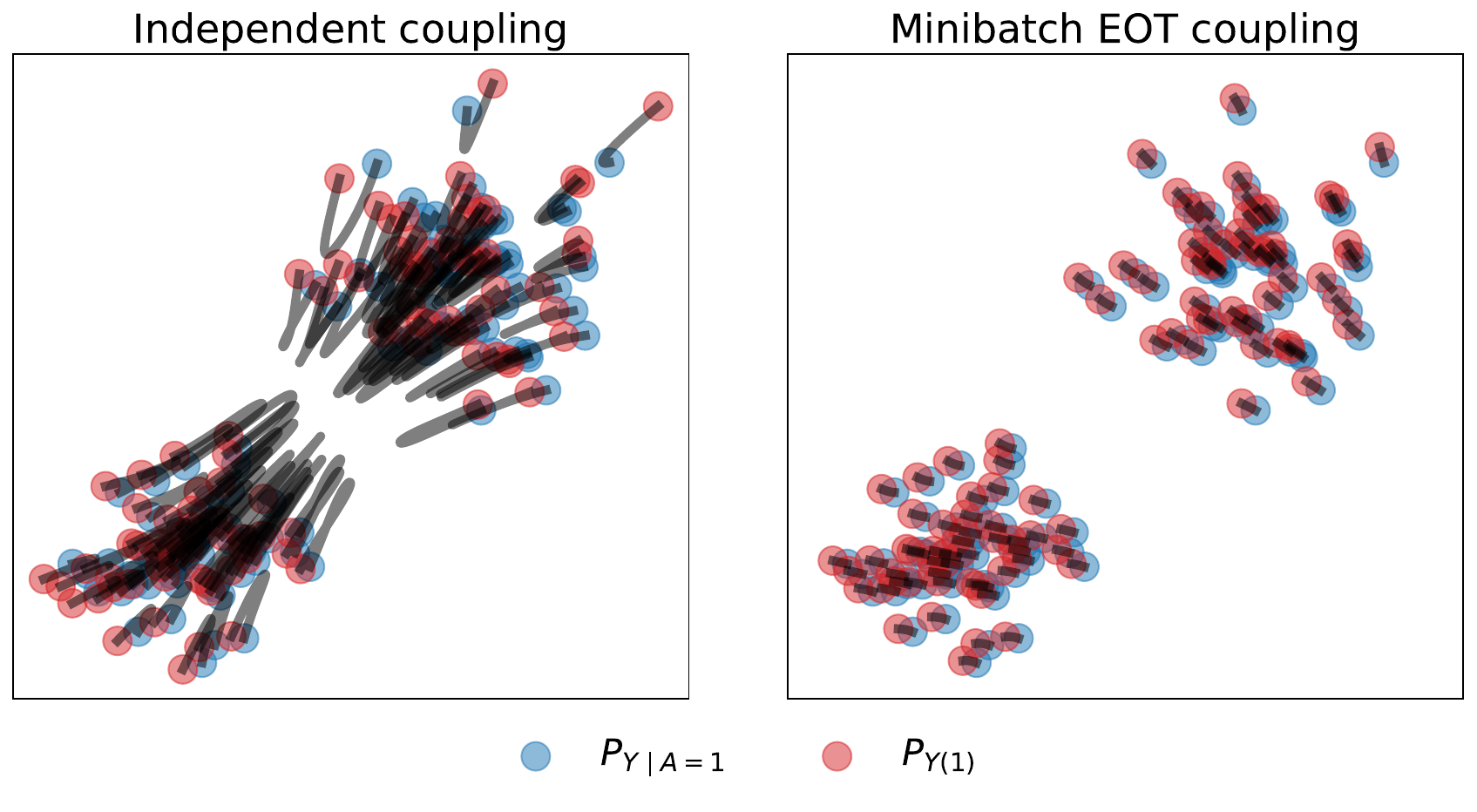}
    \caption{
    Deconfounding flow trajectories learned using independent coupling estimator (left)
    and minibatch EOT coupling estimator (right) for a mixture-of-Gaussians
    outcome under weak confounding.
    }
    \label{fig:ot_trajs}
\end{wrapfigure}

Here we examine the relative performance of our independent and Minibatch EOT estimators for deconfounding flows, in a two-dimensional synthetic design where we can control both the confounding strength and outcome multimodality. Covariates are drawn as
$X \sim \mathcal N(0, I_2)$,
treatment is assigned according to
$A \sim \mathrm{Bernoulli}(\sigma(\omega^\top X))$,
and outcomes are generated as
$Y = X + \xi$, where  $\xi \sim \tfrac 12 \mc N(0,\sigma^2) + \tfrac 12 \mc N(\mu,\sigma^2)$ for varying $\mu \in \{0,5\}$. Note that while there is no causal effect of $A$ on $Y$, learning the counterfactual distribution $\mathbb P_{Y(a)}$ still requires deconfounding since
$\mathbb P_X \neq \mathbb P_{X \mid a}$.

\paragraph{Results}
\cref{fig:coupling_ablation} shows the effect of the coupling choice on distributional accuracy (SW2, sliced Wasserstein-2) across different confounding strength ($\|\omega\|$), noise settings, and velocity architectures (from 10 trials of $n=1000$ observational samples per trial). In the unimodal Gaussian setting, differences between
coupling strategies are relatively modest once a sufficiently expressive MLP model is
used for the velocity. However, the
OT coupling is more robust to model capacity: even a much smaller architecture
(width=4) achieves performance comparable to wider models (width=64), suggesting that OT
induces simpler transport paths that are easier to fit by the estimator.

In contrast, for the multimodal mixture-of-Gaussians outcome, the independent coupling
exhibits a pronounced degradation in performance compared to the OT coupling, and
this gap is not fully addressed by using the larger capacity model.
\begin{wrapfigure}{r}{0.6\linewidth}
    \centering
    \includegraphics[width=\linewidth]{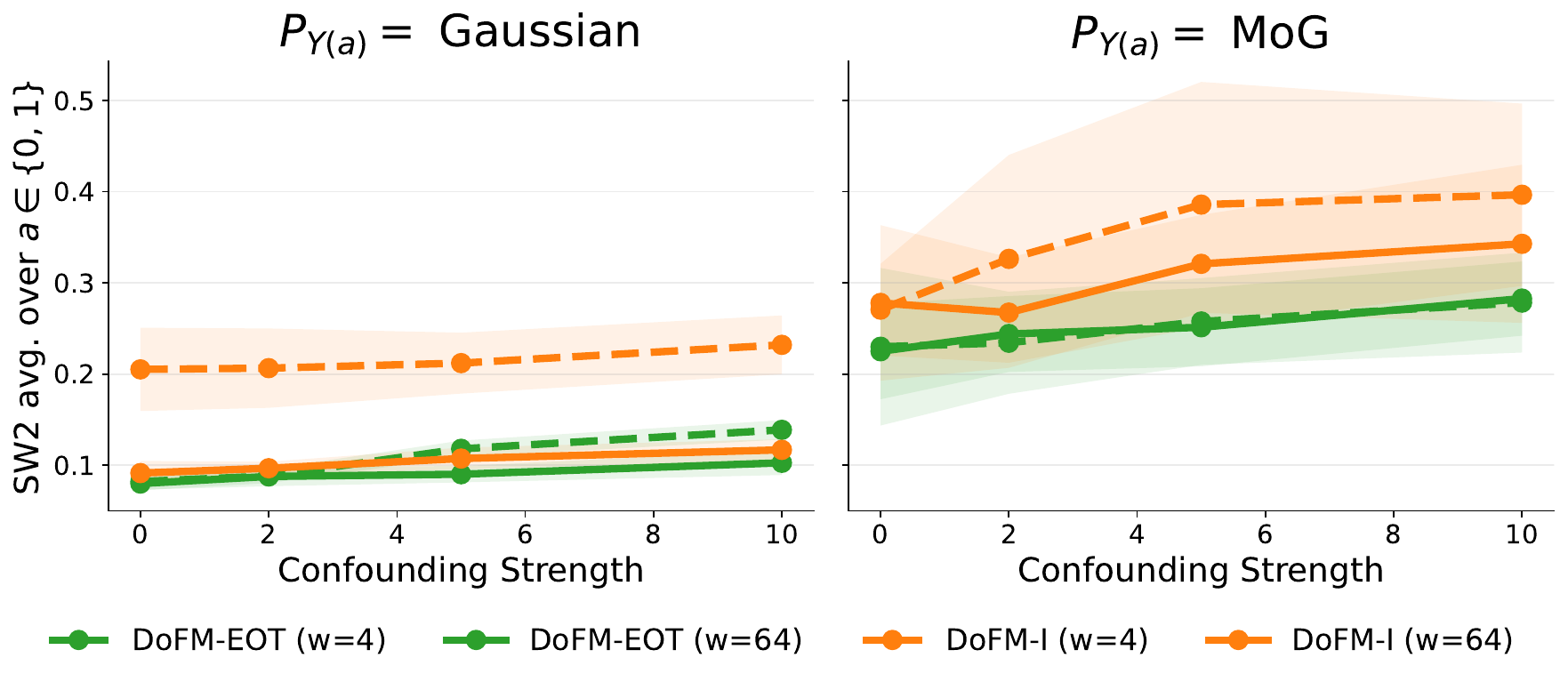}
    \caption{
    {\color{black}Coupling ablation. SW$_2$ error (mean $\pm$ sd) as a function of
    confounding strength for Gaussian (left) and multimodal (right) outcomes. $W = $ Velocity MLP width.}
    }
    \label{fig:coupling_ablation}
\end{wrapfigure}This agrees with the
trajectory visualizations in \cref{fig:ot_trajs}, which show the independent coupling produces
    highly curved transport paths, a consequence of its known trajectory bias toward the mean \citep{guan2025mirror}. In contrast, the OT coupling produces straight paths, which are therefore much easier to fit by the estimator. Overall, this experiment highlights clear benefits of directly
targeting minimal-energy deconfounding flows.  \newline

\subsection{Testing Double Robustness to Nuisance Misspecification}
\label{app:oracle_dr_ablation}
We additionally test robustness to nuisance misspecification using the two-dimensional mixture-of-Gaussians design introduced in the preceding section, with \emph{oracle} nuisance models subject to controlled contamination.

\paragraph{Experiment.}For each contamination level, we replace the fitted nuisance models by contaminated oracle nuisances and train using the same architecture and optimization settings as in the preceding experiment. We compare DecFM with independent coupling (DecFM-I), DecFM with the estimated EOT coupling (DecFM-EOT), and the corresponding \emph{Plugin flow} and \emph{IPW flow}, which correspond to DecFM-I with only the plug-in and inverse-propensity-weighted components of the estimating equation, respectively.

To perturb the propensity while keeping the outcome nuisance correct, we use
\[
\widehat\pi_\alpha(x)
=(1-\alpha)\pi(x)+\alpha/2,
\qquad
\alpha\in\{0,.25,.5,.75,1\}.
\]
Thus increasing $\alpha$ progressively shrinks the propensity toward randomized treatment, with $\widehat\pi_1(x)\equiv 1/2$. Conversely, to perturb the outcome nuisance while keeping the propensity exact, we use
\[
\widehat Q_\beta(\cdot\mid x,a)
=
(1-\beta)\bb P_{Y\mid X=x,A=a}
+
\beta\,(\tau_{\#}\bb P_{Y\mid X=x,A=a}),
\qquad
\beta\in\{0,.25,.5,.75,1\},
\]
where $\tau(y)=y+2\bm 1_2$. Thus increasing $\beta$ progressively replaces the oracle conditional distribution by a translated copy, inducing conditional mean bias $2\beta\bm 1_2$. We vary only one nuisance at a time and report SW2 error as mean $\pm$ standard deviation over 10 random seeds.

\paragraph{Results.}
The results in \cref{tab:oracle_dr_ablation}
 show the expected double-robustness pattern. When the outcome nuisance is correct, DecFM-I and DecFM-EOT remain broadly stable as the propensity is increasingly misspecified, whereas the IPW flow deteriorates steadily. Conversely, when the propensity is correct, DecFM-I remains essentially unchanged under severe outcome-nuisance contamination, while the Plugin flow deteriorates rapidly. DecFM-EOT does show some modest degradation under strong outcome misspecification, which is likely because the estimated EOT coupling depends on the plug-in target. However, our exponential tilting scheme makes this degradation extremely mild in comparison to that of the Plugin flow. The fact that the Plugin flow performs best with no contamination is expected when the outcome nuisance is oracle-correct, as it avoids the additional finite-sample variance of the IPW-augmentation term. Its rapid deterioration under misspecification highlights the robustness gained by the doubly robust estimators.

\begin{table}[h]
\centering
\small
\caption{Oracle nuisance-contamination ablation. Top: propensity nuisance increasingly misspecified while the outcome nuisance is correct. Bottom: outcome nuisance increasingly misspecified while the propensity is correct. Methods that do not depend on the perturbed nuisance are unchanged across contamination levels.}
\label{tab:oracle_dr_ablation}
\begin{tabular}{c|cccc}
\toprule
Contamination & DecFM-I & DecFM-EOT & Plugin flow & IPW flow \\
\midrule
\multicolumn{5}{c}{\textit{Propensity contaminated; outcome nuisance correct}} \\
0.00 & $0.426\pm0.145$ & $0.352\pm0.120$ & $0.270\pm0.073$ & $0.439\pm0.142$ \\
0.25 & $0.345\pm0.109$ & $0.343\pm0.133$ & $0.270\pm0.073$ & $0.476\pm0.104$ \\
0.50 & $0.351\pm0.135$ & $0.330\pm0.119$ & $0.270\pm0.073$ & $0.514\pm0.077$ \\
0.75 & $0.375\pm0.117$ & $0.337\pm0.110$ & $0.270\pm0.073$ & $0.571\pm0.085$ \\
1.00 & $0.371\pm0.124$ & $0.334\pm0.108$ & $0.270\pm0.073$ & $0.617\pm0.081$ \\
\midrule
\multicolumn{5}{c}{\textit{Outcome nuisance contaminated; propensity correct}} \\
0.00 & $0.426\pm0.145$ & $0.352\pm0.120$ & $0.270\pm0.073$ & $0.439\pm0.142$ \\
0.25 & $0.413\pm0.155$ & $0.347\pm0.131$ & $0.481\pm0.032$ & $0.439\pm0.142$ \\
0.50 & $0.428\pm0.161$ & $0.352\pm0.107$ & $0.792\pm0.036$ & $0.439\pm0.142$ \\
0.75 & $0.425\pm0.124$ & $0.377\pm0.086$ & $1.123\pm0.035$ & $0.439\pm0.142$ \\
1.00 & $0.462\pm0.141$ & $0.444\pm0.081$ & $1.456\pm0.043$ & $0.439\pm0.142$ \\
\bottomrule
\end{tabular}
\end{table}

\subsection{Non-binary Treatment Experiment on ColorMNIST}
\label{app:details:threearm}

We additionally evaluate DecFM in a three-treatment setting to verify that the approach extends naturally beyond the binary-treatment experiments in the main paper.

\paragraph{Experiment.}
We construct a three-arm variant of ColorMNIST with treatment
\[
A\in\{0,1,6\},
\qquad
X\sim \mathrm{Unif}(0,1),
\]
where $A$ determines the MNIST digit class and $X$ determines its foreground colour. Specifically, an image with confounder value $X=x$ is coloured according to
\[
\mathrm{RGB}(x)=(x,0,1-x),
\]
interpolating continuously from blue to red on a black background. Treatment assignment depends on $X$ according to
\[
P(A=a_j\mid X=x)
\propto
\exp\{-2(x-c_j)^2\},
\qquad
(c_0,c_1,c_6)=(0,\tfrac12,1),
\]
with a propensity floor of $0.01$ followed by renormalization (the floor is effectively non-binding in this design). For reference, the resulting assignment probabilities are approximately
\[
\begin{array}{c|ccc}
x & P(A=0\mid x) & P(A=1\mid x) & P(A=6\mid x)\\
\hline
0   & 0.574 & 0.348 & 0.078\\
0.5 & 0.274 & 0.452 & 0.274\\
1   & 0.078 & 0.348 & 0.574
\end{array}
\]
so that colour is associated with treatment assignment while also directly determining image appearance.

We generate $n_{\mathrm{obs}}=20{,}000$ observational samples. Reference interventional distributions are generated independently with $X\sim\mathrm{Unif}(0,1)$ and $A$ fixed, using $5{,}000$ samples per arm.

We compare three flow-matching estimators. \textsc{DoubleGen} uses a standard Gaussian base with the usual independent coupling. \textsc{DecFM-I} replaces the Gaussian base with the empirical treatment-conditional distribution $\widehat P_{Y\mid A=a}$ (with Gaussian perturbation $\sigma=0.1$), while retaining the same independent coupling. Thus, the comparison between DoubleGen and DecFM-I isolates the effect of using the observational distribution as the flow base. \textsc{DecFM-EOT} retains the empirical base but replaces the independent coupling with the minibatch EOT coupling, isolating the additional effect of targeting a lower-energy transport. The EOT variant uses 20 Sinkhorn iterations, an OT source batch of 128, one plug-in draw per context, and adaptive entropic regularization
\[
\varepsilon = 0.1\times\text{mean minibatch transport cost}.
\]

The nuisance model $\widehat {\bb P}_{Y\mid A,X}$ is a conditional U-Net with base width $8$, trained by standard independent flow matching from a Gaussian base. The propensity is trained identically to the main CMNIST experiment. The target DecFM model is a single shared three-arm U-Net, also with base width $8$, with the treatment arm supplied as conditioning information. We train for $1{,}000$ epochs using Adam with learning rate $10^{-4}$, batch size $128$, default momentum parameters, and no weight decay, learning-rate schedule, EMA, or gradient clipping. All models use FP32 precision, and sampling uses a 50-step midpoint ODE solver. The plug-in reservoir contains five samples per context, with one sampled per loss evaluation and the reservoir refreshed every 50 epochs.

Performance is measured using sliced Wasserstein--2 distance (SW2) from $5{,}000$ generated samples per arm to the corresponding interventional reference, using 256 random projections.

\begin{table}[h]
\centering
\small
\caption{Three-arm ColorMNIST results. Entries report mean $\pm$ standard deviation over three random seeds. Lower is better for both SW2 and FID.}
\label{tab:cmnist_three_arm}
\resizebox{\linewidth}{!}{%
\begin{tabular}{lccc|ccc}
\toprule
&
\multicolumn{3}{c|}{SW2 $\downarrow$}
& \multicolumn{3}{c}{FID $\downarrow$} \\
\cmidrule(lr){2-4}\cmidrule(lr){5-7}
Method 
& $A=0$ & $A=1$ & $A=6$
& $A=0$ & $A=1$ & $A=6$ \\
\midrule

DoubleGen
& $0.0338 \pm 0.0065$
& $0.0231 \pm 0.0023$
& $0.0290 \pm 0.0002$
& $93.29 \pm 22.30$
& $72.46 \pm 9.90$
& $119.33 \pm 8.58$ \\

DecFM-I
& $0.0224 \pm 0.0051$
& $0.0072 \pm 0.0012$
& $0.0290 \pm 0.0041$
& $36.42 \pm 5.76$
& $43.00 \pm 5.01$
& $41.82 \pm 3.52$ \\

DecFM-EOT
& $\mathbf{0.0183 \pm 0.0063}$
& $\mathbf{0.0057 \pm 0.0012}$
& $\mathbf{0.0178 \pm 0.0050}$
& $\mathbf{29.36 \pm 0.21}$
& $\mathbf{28.12 \pm 3.22}$
& $\mathbf{33.40 \pm 5.84}$ \\

\bottomrule
\end{tabular}%
}
\end{table}

\paragraph{Results.}
As \cref{tab:cmnist_three_arm} shows, the conclusions of the binary ColorMNIST experiment extend to the multi-valued setting. Replacing the Gaussian base with the observational treatment-conditional base reduces average SW2 from $0.0286$ to $0.0195$ and average FID from $95.0$ to $40.4$. EOT further reduces these to $0.0139$ and $30.3$, respectively, achieving the best performance in every arm. Thus, the observational base and EOT coupling provide complementary gains, and the shared-arm DecFM construction extends naturally beyond binary treatments.

\subsection{Computational Complexity and Runtime}
\label{app:computational_cost}

We compare the computational cost of DecFM with the counterfactual-distribution estimators used in \cref{sec:experiments}. Let $N$ denote sample size, $D$ the outcome dimension, and $B$ the minibatch size. We report the dominant dependence on these quantities, treating network widths/depths, the number of random-forest trees, and the number of Sinkhorn iterations as fixed; $\widetilde O(\cdot)$ suppresses logarithmic factors and dependence on fixed nuisance-model hyperparameters. For neural estimators, the reported costs are per training epoch.

\begin{table}[h]
\centering
\caption{Computational complexity of DecFM and the counterfactual-distribution baselines.}
\label{tab:computational_complexity}
\small
\begin{tabular}{llll}
\toprule
Method
& Outcome nuisance
& Propensity nuisance
& Target-model estimation \\
\midrule
KDE
& $\widetilde O(ND)$ RF
& $\widetilde O(ND)$ RF
& $O(ND)$ closed-form \\

TSE
& $\widetilde O(ND)$ RF
& $\widetilde O(ND)$ RF
& $O(ND)$ closed-form \\

INF
& $O(ND)$ / epoch
& $O(ND)$ / epoch
& $O(ND)$ / epoch \\

DR-MKSD
& $O(N^3+N^2D)$ kernel ridge
& $\widetilde O(ND)$ RF
& $O(N^2D)$ / epoch \\

DoubleGen-FM
& $O(ND)$ / epoch conditional flow
& $\widetilde O(ND)$ RF
& $O(ND)$ / epoch DR-FM \\

DecFM-I
& $O(ND)$ / epoch conditional flow
& $\widetilde O(ND)$ RF
& $O(ND)$ / epoch DR-FM \\

DecFM-EOT
& $O(ND)$ / epoch conditional flow
& $\widetilde O(ND)$ RF
& $O(NBD)$ / epoch \\
\bottomrule
\end{tabular}
\end{table}

The default independent-coupling estimator, DecFM-I, therefore has the same asymptotic training complexity as the closest flow-matching baseline, DoubleGen-FM. The two methods use the same type of conditional-flow nuisance model and doubly robust flow-matching optimization; replacing the Gaussian source by the empirical treatment-conditional distribution introduces no additional asymptotic cost.

The optional EOT extension incurs the standard overhead of minibatch optimal transport. For a minibatch of size $B$, constructing the pairwise quadratic-cost matrix and performing a fixed number of Sinkhorn iterations requires $O(B^2D)$ operations. Across $N/B$ minibatches, this gives $O(NBD)$ work per epoch, together with $O(B^2)$ additional memory for the minibatch cost/coupling representation. In contrast, DR-MKSD requires pairwise kernel computations during target-model optimization and dense kernel-ridge nuisance estimation, while the normalizing-flow implementation of INF has comparatively large constant factors associated with density, inverse-flow, and Jacobian evaluations.

\paragraph{Empirical runtime.}
To complement the asymptotic comparison, \cref{tab:runtime_causal} reports wall-clock runtimes on the four causal benchmarks in \cref{sec:exp:benchmark}, measured on a single NVIDIA A100 GPU. Values are mean $\pm$ standard deviation over five runs.

\begin{table}[h]
\centering
\caption{Runtime in minutes on the causal benchmarks using a single NVIDIA A100 GPU (mean $\pm$ standard deviation over five runs).}
\label{tab:runtime_causal}
\begin{tabular}{lcccc}
\toprule
Method & ACIC & 401k & TWINS & ACIC16 \\
\midrule
KDE
& $3.7\pm0.0$ & $3.0\pm1.0$ & $7.8\pm0.2$ & $5.7\pm0.1$ \\
TSE
& $2.2\pm0.0$ & $2.5\pm0.0$ & $3.8\pm0.0$ & $3.5\pm0.4$ \\
INF
& $26.9\pm0.2$ & $50.0\pm9.4$ & $72.5\pm4.9$ & $18.0\pm7.0$ \\
DR-MKSD
& $17.5\pm0.0$ & $56.2\pm19.1$ & $68.5\pm0.1$ & $14.7\pm2.9$ \\
DoubleGen-FM
& $8.8\pm0.0$ & $14.3\pm0.1$ & $17.3\pm0.2$ & $8.6\pm1.2$ \\
DecFM-I
& $8.8\pm0.0$ & $13.3\pm2.2$ & $17.4\pm0.1$ & $8.8\pm2.1$ \\
\bottomrule
\end{tabular}
\end{table}

The observed runtimes closely reflect the complexity comparison. DecFM-I and DoubleGen-FM have essentially identical computational cost, while both are substantially faster than INF and DR-MKSD on these benchmarks. The closed-form density estimators KDE and TSE remain faster, as expected, since they do not require iterative training of a target generative model.

Finally, we isolate the additional cost of the EOT coupling in the high-dimensional ColorMNIST experiment. Across three runs on an A100 GPU, DecFM-I required $495\pm68$ minutes compared with $602\pm53$ minutes for DecFM-EOT. Thus, although EOT changes the formal minibatch complexity from $O(ND)$ to $O(NBD)$, its wall-clock overhead is comparatively modest in our image implementation: the Sinkhorn matrix operations are inexpensive relative to the forward and backward passes through the U-Net velocity fields.

\subsection{Flow Complexity Diagnostics on Image Benchmarks}
\label{app:flow_complexity}
To directly assess whether deconfounding flows actually learn simpler transports in practice, we evaluate the geometry of the learned velocity fields on ColorMNIST and CelebA. 

\paragraph{Experiment.}
For each base/coupling configuration, we evaluate a single fully trained model from the corresponding ColorMNIST or CelebA experiment. For a trajectory $Y_t^a$ in $d$-dimensional outcome space, we compute the normalized path energy
\[
\mathcal E_{\mathrm{path}}
=
\mathbb E\!\left[
\int_0^1 \frac{1}{d}
\left\|v_\theta(Y_t^a,t,a)\right\|_2^2\,dt
\right],
\]
and the normalized material velocity-derivative energy
\[
\mathcal E_{\mathrm{deriv}}
=
\mathbb E\!\left[
\int_0^1 \frac{1}{d}
\left\|
\frac{d}{dt}v_\theta(Y_t^a,t,a)
\right\|_2^2\,dt
\right].
\]
The former measures the magnitude of the transport required along the learned trajectories, while the latter measures how rapidly the velocity changes along them, providing a complementary measure of flow smoothness. We approximate these expectations by Monte Carlo sampling 200 trajectories and evaluating the time integrals on a grid of 50 points.

\begin{table}[h]
\centering
\caption{Flow-complexity diagnostics on the image benchmarks. Lower values correspond to shorter and smoother learned transports.}
\label{tab:image_flow_complexity}
\begin{tabular}{llcc}
\toprule
Dataset & Base & $\mathcal E_{\mathrm{path}}$ & $\mathcal E_{\mathrm{deriv}}$ \\
\midrule
ColorMNIST
& $\mathcal N(0,I)$
& $0.973 \pm 0.002$
& $6.205 \pm 0.029$ \\
& $\widehat P_{Y\mid A=a}$
& $0.094 \pm 0.000$
& $0.738 \pm 0.018$ \\
& $\widehat P_{Y\mid A=a}+\mathrm{OT}$
& $0.014 \pm 0.000$
& $0.105 \pm 0.006$ \\
\midrule
CelebA
& $\mathcal N(0,I)$
& $1.189 \pm 0.006$
& $15.884 \pm 0.072$ \\
& $\widehat P_{Y\mid A=a}$
& $0.266 \pm 0.010$
& $7.127 \pm 0.335$ \\
& $\widehat P_{Y\mid A=a}+\mathrm{OT}$
& $0.119 \pm 0.003$
& $2.259 \pm 0.085$ \\
\bottomrule
\end{tabular}
\end{table}

\paragraph{Results} \cref{tab:image_flow_complexity} shows that, across both datasets, replacing the Gaussian base with the observational conditional distribution substantially reduces both flow complexity measures. The effect is especially pronounced on ColorMNIST, but remains clear on the substantially more complex CelebA benchmark. Introducing the OT coupling reduces both energies substantially further. These diagnostics provide direct empirical evidence that deconfounding flows target shorter and smoother transports than equivalent flows trained from a Gaussian base, with OT inducing even simpler flows, as predicted by the theory in \cref{sec:geometry}.

\section{Experiment and Implementation Details} \label{app:experiments}

\paragraph{Compute resources.}
Non-image experiments were run on a single-GPU and
took minutes per seed. Image experiments were run as single-GPU jobs on an HPC
cluster; compatible GPU types available on the cluster included NVIDIA RTX
4500/A4500-class GPUs, RTX 4000 Ada GPUs, A100s, L40S GPUs, H100s, and H200s.
ColorMNIST and CelebA runs required several GPU-hours per seed depending on
U-Net width and whether minibatch EOT was used. The reported experiments did not
require multi-node or large-scale distributed training.

\paragraph{Existing assets and licenses.}
All datasets used in this paper are existing research benchmarks and are used
only for research evaluation. ColorMNIST is derived from MNIST, whose dataset
license is Creative Commons Attribution--ShareAlike 3.0. CelebA is used under
its non-commercial research terms; we do not redistribute CelebA images, derived
image data, or pretrained image generators. {\color{black}We also use ACIC challenge data~\citep{dorie2019automated,acic2019challenge},
401(k)~\citep{engen2000effects}, and TWINS~\citep{louizos2017causal}
through their standard public benchmark constructions.} For benchmark distributions where no
explicit dataset license is provided by the source distribution, we follow the
published research-use conventions, credit the original sources, and do not
redistribute the underlying datasets.

\subsection{Implementation Details}

\subsubsection{(DecFM) Deconfounding Flows}

In this section, we provide the details on implementation of our method. 

\paragraph{Plug-in estimation of $\hat{\bb P}_{Y\mid X,A}$ via conditional flow matching.}

As a nuisance component of our method, we estimate the conditional outcome
distribution $\bb P_{Y\mid X,A}$ using a plug-in conditional flow-matching model.
Specifically, we learn a conditional transport that maps the empirical
treatment-conditional distribution $\hat{\bb P}_{Y\mid A=a}$ to
$\bb P_{Y\mid X=x,A=a}$, for $(x,a)\in\mc X\times\mc A$.

Concretely, we parameterize a \emph{single} time-dependent velocity field
\[
v_\eta:\mc Y\times[0,1]\times\mc X\times\mc A\to\mc Y,
\]
which takes the treatment $A$ as part of the conditioning context, and
define the neural ODE
\[
\frac{d}{dt}y_t = v_\eta(y_t,t,x,a), \qquad y_0\sim \hat{\bb P}_{Y\mid A=a}.
\]
For fixed $(x,a)$, the corresponding time-$1$ flow map induces a conditional
distribution $\hat{\bb P}_{Y\mid X=x,A=a}$.

Training is performed using standard \emph{conditional flow matching}
\citep{lipman2023flow}. That is, given observational samples
$\{(X_i,A_i,Y_i)\}_{i=1}^n$, we optimize the plug-in conditional flow-matching
objective
\[
\mathcal L^{\mathrm{plug\text{-}in}}(\eta)
\;=\;
\hat{\bb E}\!\left[
\bigl\|v_\eta(Y_t,t,X,A)-u^\star\bigr\|^2
\right],
\]
where the empirical expectation is taken over minibatches constructed as
follows. For each observed datapoint $(X_i,A_i,Y_i)$ in a minibatch, we draw an
independent base sample by resampling an outcome from the same treatment arm as $Y_i$
\[
\tilde Y_i \sim \hat{\bb P}_{Y\mid A=A_i}\,.
\]
We then
sample an interpolation time $t\sim\mathrm{Unif}(0,1)$ and form the linear
interpolation
\[
Y_{t,i} := (1-t)\tilde Y_i + tY_i,
\qquad
u_i^\star := Y_i-\tilde Y_i .
\]
The procedure is summarized in Algorithm~\ref{alg:plugin-cfm}.

After training, samples $\hat Y\sim\hat{\bb P}_{Y\mid X=x,A=a}$ are obtained by
drawing $\tilde Y\sim\hat{\bb P}_{Y\mid A=a}$ and numerically integrating the ODE
from $t=0$ to $1$ under context $(x,a)$. This conditional model is used solely as a
plug-in nuisance estimator within our debiased and doubly robust estimator for the deconfounding flow.

\paragraph{Propensity Score Estimation} In all experiments except CelebA, propensity scores $\pi_a(x)=\bb P(A=a\mid X=x)$ are estimated
using random forests. We fix the minimum number of samples per leaf node to one,
and select the maximum tree depth by cross--validation on the training data, from the grid $[1,3,5,10]$. In CelebA, as $X \in \{0,1\}$ is binary, we just use the empirical estimator of the propensity score.

\paragraph{Sample Splitting for Nuisance Models}
The semiparametric efficiency results in \cref{sec:est:debiased} require that nuisance
models (e.g., propensity scores and conditional outcome models) be trained on
data that are independent of the samples used in the empirical averages
defining the estimating equations. This can be achieved by sample splitting,
where the dataset is partitioned into disjoint folds, with nuisance models
trained on one fold and evaluated on another.

In practice, this requirement can be relaxed using standard cross-fitting
techniques. In the present setting, cross-fitting amounts to pre-estimating a
collection of nuisance models
$\{\hat m^{(k)}\}_{k=1}^K$, each trained on a different data fold, and evaluating
each sample using the model corresponding to the fold in which that sample was
\emph{not} used for training. During stochastic optimization with minibatches,
each data point is passed through its associated nuisance model according to
its fold assignment.

In our experiments, we follow this principle as follows. In the outcome geometry ablation we train nuisance models on a separate hold-out
sample of size $n$ and fix them throughout training, ensuring independence
between nuisance estimation and the empirical estimating equations. In the causal benchmarks, we use two-fold cross-fitting. In the
image-based experiments, we do not apply sample splitting or cross-fitting in
order to reduce computational cost and maximize sample efficiency. As a result, the corresponding results
should be interpreted as empirical performance evaluations rather than
guarantees of semiparametric efficiency.

\paragraph{Parameter Sharing Across Treatment Arms}
For notational clarity, the theoretical development in the main text treats the
deconfounding flow and its associated velocity field separately for each
treatment level $a \in \mathcal A$. In practice, however, we implement a single
parameterized velocity field
\[
v_\theta : \mathcal Y \times [0,1]\times \mathcal A\to \mathcal Y,
\]
which takes the treatment level $a \in \mathcal A$ as an additional input and is shared
across treatment arms. The arm-specific velocities are then recovered as
$v_{\theta,a}(y,t) := v_\theta(y,t,a)$.

This parameter-sharing strategy reduces variance and computational cost by
allowing statistical strength to be shared across arms, and is natural in
settings where the geometric structure of the outcome space is largely common
across treatments. Importantly, this modeling choice does not restrict
expressiveness: since $a$ is provided as an explicit input, the shared network
can represent distinct velocity fields for each treatment level. Our efficiency
results apply to each arm-specific velocity $v_{\theta,a}$, and parameter
sharing should therefore be viewed as a practical architectural choice rather
than an additional modeling assumption.

\paragraph{Details on Minibatch EOT Implementation} Here we derive our gradient estimator for the minibatch-EOT extension introduced in \cref{sec:ot} and describe how the coupling is
approximated in practice at each gradient descent iteration.

\emph{Derivation of estimator.}
Fix $a\in\mathcal A$ and a (possibly entropic) coupling $\gamma^\varepsilon$ between
$\bb P_{Y\mid A=a}$ (base) and $\bb P_{Y(a)}$ (target), and write its disintegration as
$\gamma^\varepsilon(d\tilde y,dy)=\gamma^\varepsilon_{0\mid 1}(d\tilde y\mid y)\,\bb P_{Y(a)}(dy)$.
Using the per-pair score $r_{\theta_a}(\cdot,\cdot)$ defined in \cref{sec:est:debiased}, define
\[
g^{\gamma}_{\theta_a}(y)
:=\int r_{\theta_a}(y,\tilde y)\,\gamma^\varepsilon_{0\mid 1}(d\tilde y\mid y),
\qquad
m(\theta_a):=\bb E\!\left[g^{\gamma}_{\theta_a}(Y(a))\right].
\]
Under the definition of $\bb P_{Y(a)}$ in Eq.~\eqref{eq:int_dist}, $m(\theta_a)$ admits the representation
$m(\theta_a)=\bb E\!\left[\mu_{\theta_a}(X)\right]$ with
$\mu_{\theta_a}(x):=\bb E\!\left[g^{\gamma}_{\theta_a}(Y)\mid A=a,X=x\right]$.
Treating $\gamma^\varepsilon_{0\mid 1}$ as fixed, the efficient influence function for
$m(\theta_a)$ (as a functional of the observational law) is the standard AIPW form
\[
\phi^{\gamma}_{\theta_a,a}(X,A,Y)
=
\frac{\bm 1\{A=a\}}{\pi_a(X)}
\Big(g^{\gamma}_{\theta_a}(Y)-\mu_{\theta_a}(X)\Big)
+\mu_{\theta_a}(X)-m(\theta_a).
\]
From here,  replacing
$\pi_a$ and $\mu_{\theta_a}$ with nuisance estimators and empirically averaging
$\phi^{\gamma}_{\theta_a,a}$ recovers the following gradient estimator in terms of the conditional coupling
\(\gamma^{\varepsilon}_{0\mid 1}(\cdot \mid y)\),
which defines, for each query point \(y\), a conditional distribution over base
samples \(\tilde y\) such that marginally $\tilde y \sim \bb P_{Y|A=a}$:
\begin{align}
&  \hat m(\theta_a):=
\nonumber \\
&\frac{1}{n}\!\sum_{i=1}^n\!
\Bigg[\!
\frac{\bm 1\{A_i=a\}}{\hat\pi_a(X_i)}\!\!
\int \! r_{\theta_a}(Y_i,\tilde y)\,
 \gamma_{0 \mid 1}(d\tilde y |Y_i)+
\Bigl(\!1-\frac{\bm 1\{A_i=a\}}{\hat\pi_a(X_i)}\!\Bigr) \!
\!\iint \!r_{\theta_a}(\hat y,\tilde y)\, \gamma_{0 \mid 1}(d\tilde y |\hat y)
\hat{\bb P}_{Y\mid a,x}(d\hat y)\!
\Bigg] \label{eq:ot_gradient}
\end{align}

\emph{Practical Approximation.} Fix a treatment arm \(a\) and a minibatch of size \(B\). Let
\(\mc S := \{\hat Y_{i,k}\}_{i=1,k=1}^{B,M}\) denote plug-in samples drawn from
\(\hat{\bb P}_{Y\mid X_i,A=a}\), and let
\(\mc T := \{\tilde Y_j\}_{j=1}^B\) denote base samples drawn from
\(\hat{\bb P}_{Y\mid A=a}\).

For each set of such minibatches drawn during an iteration of gradient descent, we estimate the entropic OT coupling between \(\mc S\) and \(\mc T\) with quadratic
cost \(c(y,y')=\|y-y'\|^2\) and regularization parameter \(\varepsilon>0\), by running $K$ iterations of Sinkhorn's algorithm. The
resulting coupling has the form
\[
\hat\gamma^{\varepsilon}(\ell,j)
\;\propto\;
a_\ell
\exp\!\left(
-\frac{\|\hat Y_\ell-\tilde Y_j\|^2}{\varepsilon}
+ v_j
\right),
\]
where \(\{v_j\}_{j=1}^B\) are the target dual potentials obtained by Sinkhorn
iterations, and \(\{a_\ell\}\) are source weights defined below. In practice, we
compute only the target dual potentials \(v_j\), which suffice to define the
conditional distribution.

Given the dual potentials, the conditional coupling appearing in the loss is
approximated via the standard entropic out-of-sample extension
\[
\hat\gamma^{\varepsilon}_{0\mid 1}(\tilde Y_j \mid y)
\;\propto\;
\exp\!\left(
v_j - \frac{\|y-\tilde Y_j\|^2}{\varepsilon}
\right),
\]
from which a single base draw \(\tilde Y\) is (re)-sampled for each query point \(y\). Note that the out-of-sample extension is tractable because we only use the
extension in the \emph{conditioning direction}. In particular, the dual
potentials are never required for new conditioning points \(y\), and by
restricting the left-hand side of
\(\hat\gamma^{\varepsilon}_{0\mid 1}(\tilde Y_j \mid y)\) to base samples
\(\{\tilde Y_j\}\) that were used to construct the minibatch coupling, the
conditional distribution depends only on the target duals \(\{v_j\}\), which
are already available from the Sinkhorn iterations.

As a result, the conditional coupling can be viewed as inducing a
data-dependent \emph{re-ordering} of the base minibatch, where the effective
pairing between \(\tilde Y_j\) and a query \(y\) depends on the term of the loss
being evaluated. This procedure is applied both to observed outcomes \(Y_i\)
and to plug-in samples \(\hat Y_{i,k}\) when evaluating the inner functional
\(
h_\theta(y)=\bb E_{\tilde Y\mid y}[\ell_\theta(y,\tilde Y)].
\)

{\color{black}\emph{Exponential tilting.}} When using minibatch entropic optimal transport, the coupling $\gamma^\varepsilon$
is computed from samples drawn from a plug--in estimate of the counterfactual
distribution $\hat{\bb P}_{Y(a)}$, which in the above implementation is constructed via
a nuisance model for $\bb P_{Y\mid A,X}$. If this plug--in model is misspecified,
the induced empirical coupling may allocate mass to regions of $\mathcal Y$ that
are incompatible with the moment condition defining $m(\theta_a)$. In particular,
because the coupling itself now depends on $\hat{\bb P}_{Y(a)}$, the usual
double--robustness property of the debiased estimator—consistency under correct
specification of either the propensity or outcome model—no longer holds in
general once the coupling is treated as data--dependent.

To reduce sensitivity to this dependence, one can optionally apply an exponential tilting (or
exponential reweighting) scheme to a fixed reservoir of plug--in samples
$\{\hat y_j\}_{j=1}^M \sim \hat{\bb P}_{Y(a)}$ prior to computing the entropic
coupling. Specifically, we define a tilted empirical measure
\[
\hat{\bb P}^{\mathrm{tilt}}_{Y(a)} := \sum_{j=1}^M w_j\,\delta_{\hat y_j},
\qquad
w_j(\lambda)
:= \frac{\exp\{\lambda^\top s(\hat y_j)\}}{\sum_{\ell=1}^M \exp\{\lambda^\top s(\hat y_\ell)\}},
\]
where $s(\cdot)$ denotes a chosen collection of moment features. The tilt
parameter $\lambda$ is selected once, prior to flow training, so that the
weighted empirical moments of the reservoir approximately match the
corresponding moments estimated via the doubly robust estimating equations,
\[
\sum_{j=1}^M w_j(\lambda)\, s(\hat y_j)
\;\approx\;
\hat{\bb E}_{\mathrm{DR}}\!\left[s(Y(a))\right].
\]
The resulting tilted empirical distribution is then used as the target marginal
when computing the minibatch entropic coupling $\gamma^\varepsilon$. Algorithm~\ref{alg:exp-tilt} in Appendix \ref{app:algs} summarizes the procedure. In an additional coupling ablation experiment we ran in Appendix \ref{app:details:ablation}, we choose $m=128$ random moments  $s(y) = \cos(W y + b)$, where the rows of $W \in \bb R^{128 \times \dim(Y)}$
are drawn independently from a mean--zero Gaussian distribution and
$b \in [0,2\pi]^{128}$ is drawn uniformly at random). We do not apply exponential tilting on ColorMNIST/CelebA for computational reasons.

\paragraph{Architectural and Training Details}
We distinguish between non--image and image--based experiments.

\emph{Non--image based Architectures.}
For all non--image experiments, both the plug--in conditional flow and the
deconfounding (target) flow use identical velocity network architectures unless
explicitly varied in ablations. Each velocity field is parameterized by a
multilayer perceptron with one hidden layer of width $64$, using SiLU
activations. The time variable $t$ is concatenated directly to the input, and no
positional encoding is used in low--dimensional settings.

Models are trained for $1000$ epochs using minibatches of size $512$ and the Adam
optimizer with learning rate $10^{-3}$ and default hyperparameters. No learning
rate schedules or additional regularization are applied unless stated otherwise.
For experiments involving minibatch entropic optimal transport, we use $50$
Sinkhorn iterations per optimization step and a fixed entropic regularization
parameter $\varepsilon=0.1$.

\emph{Image-based architectures.}
For image-based experiments, velocity fields are parameterized using a U-Net
architecture with skip connections. The base channel width is denoted by \(c\).

The plug--in conditional flow used to estimate \(\hat{\bb P}_{Y\mid X,A}\)
uses a U-Net that conditions on covariates \(X\), the treatment \(A\),
and time \(t\in[0,1]\). The encoder consists of four convolutional blocks with
channel dimensions \((c,2c,4c,8c)\), where each block applies two \(3\times3\)
convolutions with SiLU activations, and downsampling is performed using
\(2\times2\) max pooling between blocks. The decoder mirrors the encoder using
transposed convolutions for upsampling. At each decoder stage, the upsampled
feature map is concatenated with the corresponding encoder skip connection and a
time embedding, resulting in channel dimensions \((12c,6c,3c)\) at successive
resolutions. A final \(1\times1\) convolution maps the decoder output to the
velocity field.

Conditioning on \(A\) is implemented via class-conditional normalization layers,
while conditioning on \(X\) is incorporated through feature-wise linear
modulation (FiLM) layers~\citep{perez2018film} applied to intermediate feature
maps. Time is encoded using sinusoidal embeddings followed by learned linear
projections, with separate embeddings of dimensions \((c,2c,4c)\) injected at
the corresponding decoder resolutions. The resulting plug--in velocity field has
the form
\[
v_\eta:\mc Y\times[0,1]\times\mc X\times\mc A\to\mc Y.
\]

The deconfounding target flow uses the same U-Net backbone and time-conditioning
scheme, but omits conditioning on \(X\). In our shared-arm implementation, the
target model conditions only on the treatment \(A\) and time \(t\),
resulting in a velocity field of the form
\[
v_\theta:\mc Y\times[0,1]\times\mc A\to\mc Y.
\]

Across all image-based experiments, we use SiLU activations throughout. In ColorMNIST we vary
the base channel width $c\in\{16,32\}$ as part of the experimental design ($c=8$ in the three-arm study of Appendix~\ref{app:details:threearm}), in CelebA we use fixed $c = 64$.
Models are trained using the Adam optimizer with learning rate $10^{-4}$ and
minibatches of size $128$ for ColorMNIST and $64$ for CelebA. For entropic optimal transport, we use $20$  Sinkhorn
iterations per optimization step for ColorMNIST and 10 for CelebA, with the entropic regularization parameter
$\varepsilon$ set to $0.1$ times the mean of the
minibatch cost matrix. We ran on ColorMNIST for 1000 epochs, and CelebA for 500 epochs. To avoid overfitting, during training we inject the base samples from $\bb P_{Y|A=a}$ with $\mc N(0,\sigma^2)$ Gaussian noise with $\sigma = 0.1$.

\subsubsection{Baselines} \label{appx:baselines}

\paragraph{AIPW-KDE (AIPW kernel density estimator) \cite{kim2018causal}.}
We implement the augmented inverse--probability--weighted (AIPW)  kernel density estimator as follows. In particular, let $K_h(u)=h^{-1}K(u/h)$ denote the Gaussian kernel with bandwidth $h$. The doubly robust
(AIPW) estimator of the smoothed counterfactual density $p_{Y(a)}$ at $y$ is
\[
\hat p^{\mathrm{AIPW}}_{Y(a)}(y)
=
\frac{1}{n}\sum_{i=1}^n
\Bigg[
\frac{\bm 1\{A_i=a\}}{\hat\pi_a(X_i)}
\Bigl\{ K_h(Y_i-y)-\hat m_a(y\mid X_i) \Bigr\}
+
\hat m_a(y\mid X_i)
\Bigg],
\]
where $\hat\pi_a(x)$ is the estimated propensity score and
\[
\hat m_a(y\mid x)
:=
\bb E_{\hat p_{Y\mid X=x,A=a}}\!\left[K_h(Y-y)\right]
\]
is a plug--in estimate of the conditional smoothed density obtained from the
outcome regression model. We use the following implementation choices and hyperparameters: (i) The kernel bandwidth is selected using the median heuristic, separately within each treatment arm. (ii) Outcome regression and propensity score models are estimated using random forests, trained identically to DecFM, including cross--validation over maximum tree depth. (iii) Two--fold cross--fitting is used to ensure independence between nuisance estimation and evaluation. (iv) Wasserstein--1 distances are computed by numerical integration of the estimated density.

\paragraph{AIPW-TSE (Truncated Series Estimator) \cite{kennedy2023semiparametric}.}
We implement the Truncated series estimator (TSE) for counterfactual density
estimation, following the doubly robust construction described in
\citet{kennedy2023semiparametric}. The counterfactual density is approximated by
expanding it in an orthonormal cosine basis on a bounded interval.

Let $[y_{\min},y_{\max}]$ denote the maximum and minimum outcome in the observational set, and define the scaled
outcome
\[
\tilde Y := \frac{\min\{\max\{Y,y_{\min}\},y_{\max}\}-y_{\min}}{y_{\max}-y_{\min}} \in [0,1].
\]
We use the cosine basis
\[
\phi_j(\tilde y) = \sqrt{2}\cos(\pi j \tilde y),
\qquad j=1,\dots,d,
\]
with $d=20$ basis functions. The counterfactual density is represented as
\[
p_{Y(a)}(y)
\;\approx\;
\sum_{j=1}^d \theta_{a,j}\,\phi_j(\tilde y),
\]
where the coefficients $\theta_{a,j} = \bb E[\phi_j(\tilde Y(a))]$ are estimated
using a doubly robust estimating equation.

Specifically, for each basis coefficient we use the AIPW estimator
\[
\hat\theta_{a,j}
=
\frac{1}{n}\sum_{i=1}^n
\Bigg[
\frac{\bm 1\{A_i=a\}}{\hat\pi_a(X_i)}
\Bigl(\phi_j(\tilde Y_i)-\hat\mu_{a,j}(X_i)\Bigr)
+
\hat\mu_{a,j}(X_i)
\Bigg],
\]
where $\hat\pi_a$ is the estimated propensity score and
$\hat\mu_{a,j}(x)=\bb E[\phi_j(\tilde Y)\mid X=x,A=a]$ is estimated via outcome
regression.

The estimated density $\hat p_{Y(a)}$ is obtained by combining the estimated
coefficients $\{\hat\theta_{a,j}\}_{j=1}^d$. Wasserstein--1 distances are computed
by numerical integration of the resulting density.

In all experiments, propensity scores and outcome regressions are estimated using
random forests with the same hyperparameter grid as DecFM, and two--fold
cross--fitting is used to ensure independence between nuisance estimation and
evaluation.

\paragraph{INFs (Interventional Normalizing Flows \cite{melnychuk2023normalizing}).}
We implement Interventional Normalizing Flows following
\citet{melnychuk2023normalizing}, which directly model the counterfactual
distribution $\bb P_{Y(a)}$ via a normalizing flow trained using a
KL--divergence–based estimating objective.

For each treatment arm $a\in\mathcal A$, the target interventional distribution is
parameterized by a neural autoregressive flow and trained by minimizing an
empirical KL divergence between the model density and the distribution implied by
the identifying moment condition in \citet{melnychuk2023normalizing}. The target
flow architecture is fixed across experiments and consists of a fully connected
network with a single hidden layer of width $64$.

INFs additionally employ nuisance models for treatment assignment and outcome
modeling. These are implemented as two fully connected networks, where the
propensity network produces a $9$--dimensional embedding of the covariates that
is used as input to the outcome model. The architectures of these nuisance
networks are selected by cross--validation over hidden layer widths
$\{16,32,64\}$ and depths $\{1,2\}$.

In line with the original INFs methodology, we do not apply sample splitting or
cross--fitting. For evaluation, we draw $10{,}000$ samples from each fitted
interventional model and compute Wasserstein--1 distances empirically using these
samples.

\paragraph{DR-MKSD (Doubly-robust Minimum Kernel Stein Discrepancy \cite{martinez2024counterfactual}).}
We implement the doubly robust minimum kernel Stein discrepancy (DR-MKSD) estimator of
\citet{martinez2024counterfactual}, which estimates the counterfactual distribution
$\bb P_{Y(a)}$ via an energy-based model trained to satisfy a Stein moment
condition.

The target distribution is parameterized as an energy-based model (EBM) with
energy function given by a fully connected neural network with a single hidden
layer of width $64$ and SiLU activations. Model parameters are estimated by
minimizing the empirical kernel Stein discrepancy. Optimization details are identical to DecFM.

The estimator relies on nuisance models for both the propensity score and the
conditional outcome. Propensity scores are estimated using random forests,
trained identically to those used for DecFM, including cross-validation over tree
depth. The outcome regression is estimated using kernel ridge regression in an
RKHS, with the regularization parameter selected by cross-validation over a grid
of $10$ values logarithmically spaced between $10^{-3}$ and $10^{-1}$.

We use two-fold cross-fitting to ensure independence between nuisance estimation
and evaluation. Wasserstein--1 distances are computed by drawing $n=10,000$ samples from the
fitted energy-based model via the Unadjusted Langevin Algorithm and estimating the distance empirically.

\paragraph{DoubleGen \cite{luedtke2025doublegen}}
We implement the flow-matching variant of DoubleGen using our own estimator and architectural/training details but with a Gaussian base. Concretely, DoubleGen uses the same velocity architecture, optimizer, training
epochs, nuisance models, and cross-fitting scheme as DecFM, but replaces the
empirical observational base \(\bb P_{Y\mid A=a}\) by a standard Gaussian base.

\subsection{Experimental Designs}

\subsubsection{Synthetic 1d Outcome Designs}

In \cref{sec:exp:1dablation} we consider three one--dimensional outcome designs with a continuous confounder
$X\in[0,1]$, a binary treatment $A\in\{0,1\}$, and scalar outcome $Y\in\bb R$.
Across all designs, confounding is induced through the treatment assignment
mechanism
\[
A \mid X=x \sim \mathrm{Bernoulli}(\pi(x)),
\qquad
\pi(x) = \sigma\bigl(5(x-\tfrac12)\bigr)
\]
where $\sigma(\cdot)$ is the logistic sigmoid.

\paragraph{Gaussian outcome.}
In the Gaussian design, outcomes are conditionally normal:
\[
Y \mid X=x, A=a \sim \mathcal N(10x + a,\;\sigma^2),
\]
with $\sigma=1$ unless stated otherwise. This setting serves as a baseline with
matched support and light tails, for which Gaussian base distributions are
well-aligned with the target geometry.

\paragraph{Inverse--Gamma outcome (heavy tails).}
To study the effect of tail mismatch, we consider a heavy--tailed outcome model:
\[
Y \mid X=x, A=a \sim \mathrm{InverseGamma}(\alpha(x),\beta(x,a)),
\qquad
\alpha(x)=10-8.5x,\;\;
\beta(x,a)=a+1+10x.
\]
Here, treatment affects scale but not tail index, while confounding through $X$
induces substantial heteroskedasticity and heavy tails in both observational and
interventional distributions.

\paragraph{Uniform mixture (split support).}
To isolate support mismatch, we consider a bimodal outcome distribution with
disconnected support:
\[
Y \mid X=x, A=a = a + U,
\]
where $U$ is drawn from a symmetric mixture of uniforms
\[
U \sim \tfrac12\,\mathrm{Unif}[-\beta(x),-\alpha(x)] 
     + \tfrac12\,\mathrm{Unif}[\alpha(x),\beta(x)],
\]
with
\[
\alpha(x)=1-\varepsilon(c x+1),
\qquad
\beta(x)=1+\varepsilon(c x+1).
\]
In this design, the two support components are separated by a gap around zero,
whose width depends on $X$. We use the parameterization $c = 100$, $\epsilon = 0.001$.

\subsubsection{Synthetic 2d Mixture of Normals}

In Appendix \ref{app:details:ablation} we consider an additional two--dimensional synthetic design with continuous covariates
$X\in\bb R^2$, binary treatment $A\in\{0,1\}$, and outcome $Y\in\bb R^2$. The
design is constructed to induce geometric confounding through a linear
projection of $X$, while allowing control over multimodality and support
separation in the outcome space.

Covariates are drawn as
\[
X \sim \mathcal N(0,I_2).
\]
Treatment assignment follows a logistic propensity model
\[
A \mid X=x \sim \mathrm{Bernoulli}\!\left(\sigma(\gamma\, w^\top x)\right),
\]
where $w=(1,-1)/\sqrt{2}$ is a fixed confounding direction and $\gamma>0$ controls the
strength of confounding. In all experiments, we vary $\gamma$ to interpolate
between weak and strong confounding regimes.

Outcomes are generated as 
\[
Y = X + \varepsilon,
\]
so that there is no average or distributional treatment effect. The observed
outcome is given by $Y=Y(A)$.

The noise term $\varepsilon$ induces a bimodal mixture structure:
\[
\varepsilon =
\begin{cases}
\eta, & \text{with probability } \tfrac12,\\
(\delta, \delta) + \eta, & \text{with probability } \tfrac12,
\end{cases}
\qquad
\eta \sim \mathcal N(0,\sigma^2 I_2),
\]
where $\delta>0$ controls the separation between mixture components and $\sigma$
controls within--component variance. In our experiments, we fix $\sigma=0.2$
and vary $\delta$ (referred to as the \emph{mixture width}) to control the degree
of support separation.

\subsubsection{Causal Benchmarks}

In \cref{sec:exp:benchmark} we evaluate our method on a collection of standard causal benchmark datasets,
combining both fully synthetic and semi--synthetic designs.

\paragraph{ACIC 2016 Challenge datasets.}
We consider $10$ datasets from the ACIC 2016 Causal Inference Challenge, each of
which provides observational data generated from a known structural causal
model with heterogeneous treatment effects and complex confounding. For each
dataset, we construct our own random train--test split of the available data,
using $70\%$ of the samples for training and $30\%$ for evaluation. All models
are trained exclusively on the training split, and performance is reported on
the held--out test split.

\paragraph{Semi--synthetic benchmarks (ACIC 2019, 401(k), and Twins).}
For the ACIC 2019, 401(k), and Twins datasets, we adopt a semi--synthetic
construction. In each case, we begin from a real observational dataset with
covariates $X$ and treatment assignment $A$, and then generate potential
outcomes using one of the outcome--generation mechanisms used in
\cref{sec:exp:1dablation} (In \cref{tab:causal_benchmarks}, SplitSupport = Mixture of Uniform, HeavyTail = InverseGamma). This allows us to evaluate distributional
counterfactual estimators under different outcome geometries while preserving
realistic covariate and treatment structures.

For these datasets, we randomly split the data into $70\%$ training and $30\%$
test sets. All nuisance models and deconfounding flows are trained on the
training split, and evaluation is performed on the held--out test split.

\emph{Strengthening confounding via feature selection.}
To ensure challenging confounding regimes, we construct a low--dimensional
summary $\phi(X)$ of the covariates by ranking features according to their
marginal Wasserstein--1 distance between treated and untreated groups.
Specifically, for each covariate $X_j$, we compute a confounding score
\[
s_j
=
{W1}\!\bigl(\bb P(X_j \mid A=1),\,\bb P(X_j \mid A=0)\bigr).
\]
Covariates with larger $s_j$ exhibit stronger distributional separation between
treatment groups and are therefore more strongly confounded with $A$. We retain the top $K=5$ covariates with the largest scores and define
$\phi(X)$ as their normalized sum,
\[
\phi(X) := \bm 1^\top X / \sqrt{K}.
\]
Outcomes are then generated using the mechanisms described in
\cref{sec:exp:1dablation}, with $\phi(X)$ replacing the original one--dimensional
covariate. However, \emph{all covariates are used when training the models}.

\paragraph{Twins dataset.}
For the Twins dataset, we follow the treatment assignment construction of
\cite{louizos2017causal}. We first form an observational dataset by selecting one
twin at random from each same--sex twin pair and restricting attention to
low--birthweight pairs. Treatment assignment is then generated according to a
logistic model depending on gestational age,
\[
A \sim \mathrm{Bernoulli}\!\left(\sigma(w_h z)\right),
\]
where $z$ denotes a standardized gestational age covariate and $w_h$ is a
randomly drawn slope parameter. 

\subsubsection{Image Generation: ColorMNIST}

In \cref{sec:exp:mnist} we consider an image--based debiasing task constructed from the MNIST dataset~\citep{lecun1998gradient},
following a foreground--color confounding design. The goal is to estimate
counterfactual image distributions in the presence of strong confounding between
treatment assignment and a visual attribute.

\paragraph{Data generation.}
We restrict attention to two digit classes $(d_0,d_1) = (1,6)$ and construct RGB
images by recoloring grayscale MNIST digits. Each image is generated from a
latent scalar confounder $X\in[0,1]$, drawn uniformly. Treatment assignment
follows a logistic propensity model
\[
A \mid X=x \sim \mathrm{Bernoulli}\!\left(\sigma\bigl(w(x-\tfrac12)\bigr)\right),
\]
where $w>0$ controls the strength of confounding.

Conditional on $(X,A)$, the observed image $Y$ is generated by selecting a digit
from class $d_A$ and recoloring its foreground using an RGB mapping that depends
on $X$. Foreground recoloring is implemented via a soft mask derived from the
grayscale digit intensity. Let $S\in[0,1]^{28\times28}$ denote the original MNIST
digit. We define a foreground mask
\[
M = \sigma\bigl(k(S-\tau)\bigr),
\]
where $\sigma(\cdot)$ is the logistic sigmoid and $(\tau,k)$ control the
foreground--background separation. Given a foreground RGB color
$c_{\mathrm{fg}}(X)\in[0,1]^3$ and a fixed background color $c_{\mathrm{bg}}$, the
recolored image is constructed pixelwise as
\[
Y = M \odot (c_{\mathrm{fg}}(X)\, S) + (1-M)\odot c_{\mathrm{bg}}.
\]

In our experiments, the foreground color varies smoothly with the confounder
according to
\[
c_{\mathrm{fg}}(X) = (X,\alpha,1-X),
\]
with fixed $\alpha\in[0,1]$ and constant background color $c_{\mathrm{bg}}$.
This construction induces a nonlinear (approximately logistic) distribution over
the relative red--blue intensity ratio, creating a strong but smooth association
between color, treatment assignment, and digit class.

\paragraph{Evaluation setting.}
Models are trained on \(n=20{,}000\) observational samples \((X,A,Y)\). To
evaluate counterfactual performance, we generate \(5000\) reference samples from
each interventional distribution \(\bb P_{Y(a)}\) by fixing the digit class to
\(d_a\) and resampling \(X\) independently from its marginal distribution. The
foreground colors are then assigned using the same \(X\)-dependent recoloring
rule as in the observational data, so that reference images preserve the latent
color variation induced by \(X\) while removing treatment-selection bias. We
compare generated and reference samples using sliced Wasserstein--2 distance,
averaged over \(a\in\{0,1\}\) and over three random seeds.

\subsection{Image Generation: CelebA}

\paragraph{Data Generation.} We use CelebA~\citep{liu2015faceattributes} as a high-dimensional attribute-rebalancing benchmark. Images are
taken from the aligned CelebA release, center-cropped, resized to
\(64\times 64\), converted to RGB tensors, and normalized to \([-1,1]\). We set
\(A=\mathrm{Male}\), so that \(A=1\) denotes male and \(A=0\) denotes female,
and set \(X=\mathrm{Blond\_Hair}\). To make \(X\) a clean binary hair-colour
attribute, we restrict to images with exactly one of \(\mathrm{Blond\_Hair}\)
or \(\mathrm{Brown\_Hair}\) active, and remove images labelled with
\(\mathrm{Black\_Hair}\), \(\mathrm{Gray\_Hair}\), \(\mathrm{Bald}\), or
\(\mathrm{Wearing\_Hat}\). Thus, in this experiment \(X=1\) denotes blond hair
and \(X=0\) denotes brown hair.

We construct a controlled observational sample from the cleaned train and
validation splits of CelebA. Specifically, we combine the official train and
validation splits into a single cleaned pool, form the four \((A,X)\) cell
pools, and then sample observed and reference images without replacement from
these cells. The observed sample has size \(n_{\mathrm{obs}}=20{,}000\) and is
sampled to satisfy
\[
    \bb P(X=1)=0.3,\qquad
    \bb P(A=1\mid X=0)=0.6,\qquad
    \bb P(A=1\mid X=1)=0.1 .
\]
Equivalently, the target cell proportions for the observational sample are
\[
\begin{array}{c|cc}
 & X=0 & X=1 \\
\hline
A=0 & 0.28 & 0.27 \\
A=1 & 0.42 & 0.03
\end{array}
\]
up to integer rounding. This induces a strong association between sex and hair
colour: females are over-represented among blond examples, while males are
under-represented.

\paragraph{Evaluation setting.} For evaluation, we construct disjoint reference samples for each intervention
arm. For each \(a\in\{0,1\}\), the reference set contains
\(n_{\mathrm{ref}}=2{,}000\) held-out images sampled from the same cleaned pool
with \(A=a\) fixed and
\[
    \bb P_{\mathrm{ref}}(X=1\mid A=a)=0.3 .
\]
Thus the reference distribution for arm \(a\) preserves within-cell CelebA image
variation but replaces the group-specific hair-colour distribution
\(\bb P_{X\mid A=a}\) by the population target \(\bb P_X\). Observed and
reference indices are sampled disjointly within each \((A,X)\) cell. We report
sliced Wasserstein--2 distances between generated samples and these reference
distributions, as well as the corresponding distance between the observational
base \(\bb P_{Y\mid A=a}\) and the reference target. Results in Figure~4 are
reported as mean \(\pm\) sample standard deviation over three random seeds.

\subsection{Additional Figures}

\begin{figure}[h]
    \centering
    \includegraphics[width=0.7\linewidth]{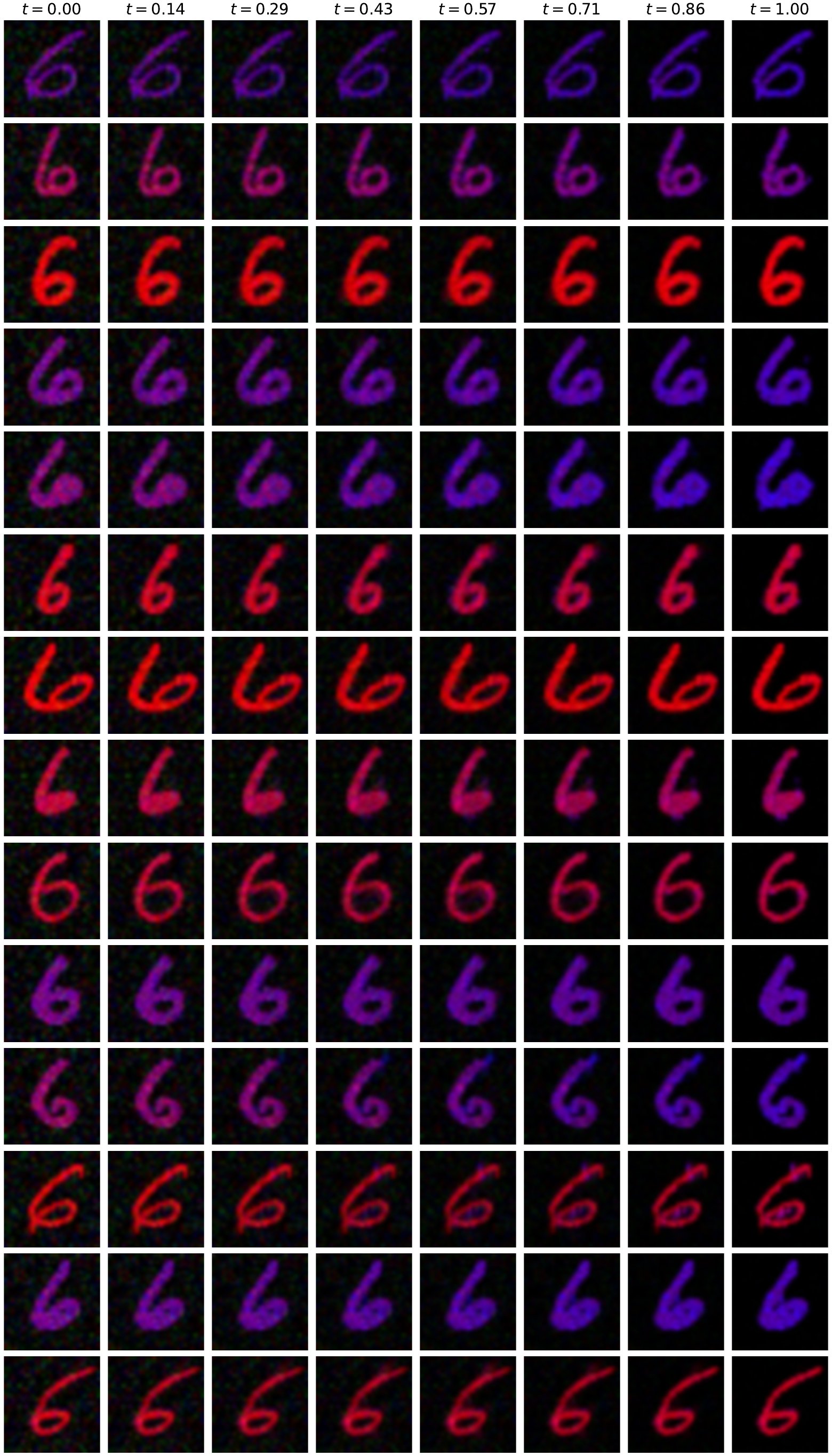}
\caption{Additional sampled ColorMNIST trajectories from DecFM-EOT.}        
\label{fig:cmnist_extra_trajectories}
\end{figure}

\begin{figure}[h]
    \centering
    \includegraphics[width=\linewidth]{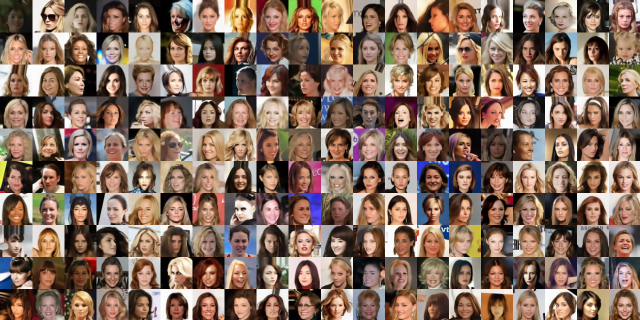}
\caption{200 Generated Samples with $\mathrm{Sex}=\mathrm{Female}$ from DecFM-EOT on CelebA}        
\label{fig:celeba_female_samples}
\end{figure}

\begin{figure}[h]
    \centering
    \includegraphics[width=\linewidth]{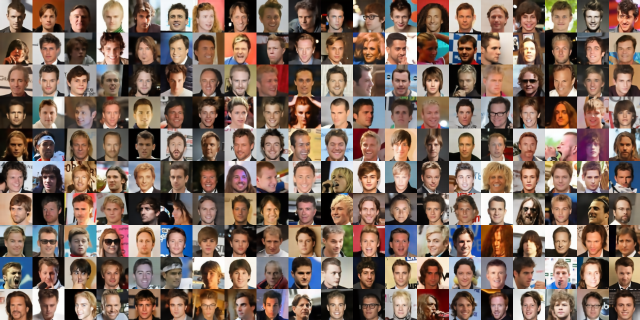}
\caption{200 Generated Samples with $\mathrm{Sex}=\mathrm{Male}$ from DecFM-EOT on CelebA}        
\label{fig:celeba_male_samples}
\end{figure}

\clearpage 

\section{Algorithms}\label{app:algs}

\begin{algorithm}[H]
\caption{Plug-in conditional flow matching nuisance for $\hat{\bb P}_{Y\mid X,A}$}
\label{alg:plugin-cfm}
\KwRequire{Observational data $\mc D=\{(X_i,A_i,Y_i)\}_{i=1}^n$; velocity network $v_\eta(y,t,x,a)$; optimizer \textsc{OPT}; number of iterations $T$; minibatch size $B$}

Partition indices by treatment arm: $\mc I_a := \{i\in[n]: A_i=a\}$ for $a\in\mathcal A$\;

\For{$k=1,\dots,T$}{
    Sample a minibatch of indices $\mc B \subset [n]$ with $|\mc B|=B$\;
    
    Form batch tensors $\{(x_i,a_i,y_i)\}_{i\in\mc B}$\;
    
    \ForEach{$i\in\mc B$}{
        Sample $j \sim \mathrm{Unif}(\mc I_{a_i})$ independently\;
        
        Sample $\tilde y_i  \sim \hat{\bb P}_{Y\mid A=a_i}$\;
    }
    
    Sample times $t_i \stackrel{\text{iid}}{\sim}\mathrm{Unif}(0,1)$ for $i\in\mc B$\;
    
    Compute interpolates $y_{t,i} \gets (1-t_i)\tilde y_i + t_i y_i$\;
    
    Compute targets $u_i^\star \gets y_i - \tilde y_i$\;
    
    Compute loss
    \[
    \widehat{\mathcal L}^{\mathrm{plug\text{-}in}}(\eta)
    \;=\;
    \frac{1}{B}\sum_{i\in\mc B}
    \bigl\|v_\eta(y_{t,i},t_i,x_i,a_i)-u_i^\star\bigr\|^2
    \]
    
    Update parameters $\eta \gets \textsc{OPT}\!\left(\eta,\nabla_\eta \widehat{\mathcal L}^{\mathrm{plug\text{-}in}}(\eta)\right)$\;
}
\Return{$v_\eta$}\;
\end{algorithm}

\begin{algorithm}[H]
\caption{DecFM-I: Debiased Flow Matching with Independent Coupling}
\label{alg:dr-fm} 
\label{alg:i-DecFM}
\KwRequire{Observational data $\{(X_i,A_i,Y_i)\}_{i=1}^n$;
propensity estimator $\hat\pi$;
plug-in conditional sampler $\hat{\bb P}_{Y\mid X,A}$;
velocity field $v_\theta$;
minibatch size $B$;
outer MC size $M$ (plug-in draws);
base MC size $M_0$;
optimizer \textsc{OPT};
iterations $T$.}

Construct empirical bases $\hat{\bb P}_{Y\mid A=a}$ for $a\in\mathcal A$.\;

\For{$t = 1,\dots,T$}{
    Sample minibatch indices $\mc B \subset \{1,\dots,n\}$ with $|\mc B|=B$.\;
    
    Extract batch $(X_b,A_b,Y_b)_{b\in\mc B}$.\;
    
    Sample interpolation times $\tau_b \stackrel{iid}{\sim} \mathrm{Unif}(0,1)$ for $b\in\mc B$.\;
    
    \For{each arm $a\in\mathcal A$}{
        
        \For{each $b\in\mc B$}{
            Draw base samples $\tilde Y_{b,1:M_0} \overset{iid}{\sim} \hat{\bb P}_{Y\mid A=a}$.\;
            
            Draw plug-in samples $\hat Y_{b,1:M} \overset{iid}{\sim} \hat{\bb P}_{Y\mid X=X_b,A=a}$.\;
            
            Set IPW weight $w_b \gets \bm 1\{A_b=a\}/\hat\pi_a(X_b)$.\;
        }
        
        Accumulate DR loss for arm $a$:
        \[
        \widehat{\mc L}_a(\theta)
        =
        \frac{1}{B}\sum_{b\in\mc B}
        \frac{1}{M_0}\sum_{m'=1}^{M_0}
        \Bigg[
        w_b\,\ell_\theta\!\bigl(Y_b,\tilde Y_{b,m'};\tau_b\bigr)
        +
        \frac{1-w_b}{M}\sum_{m=1}^{M}
        \ell_\theta\!\bigl(\hat Y_{b,m},\tilde Y_{b,m'};\tau_b\bigr)
        \Bigg].
        \]
    }
    
    Set $\widehat{\mc L}(\theta)\gets \sum_{a\in\mathcal A}\widehat{\mc L}_a(\theta)$.\;
    
    Update parameters \(
    \theta \gets \textsc{OPT}\!\left(\theta,\nabla_\theta \widehat{\mc L}(\theta)\right).\)\;
}
\Return{$\theta$}\;
\end{algorithm}

\begin{algorithm}[H]
\caption{DecFM-EOT: DR Flow Matching with Minibatch Entropic OT Conditionals.}
\label{alg:dr-ot-fm}
\KwRequire{Data $\mc D=\{(X_i,A_i,Y_i)\}_{i=1}^n$; propensity model $\hat\pi(x)$; plug-in conditional sampler $\hat{\bb P}_{Y\mid X,A}$; base sampler $\hat{\bb P}_{Y\mid A}$; velocity model $v_\theta$; pairwise FM loss $\ell_\theta(y,\tilde y)$; minibatch size $B$; plug-in outer MC size $M$; OT plug-in subsample size $K\le M$; Sinkhorn reg.\ $\varepsilon$ and iters $S$; \textsc{OPT}.}

\For{each training iteration}{
    Sample indices $\mc B \subset [n]$ with $|\mc B|=B$ and form batch $(X,A,Y)$.\;
    
    \For{arm $a\in\mathcal A$}{
        
        Draw base samples $\tilde Y^{\mathrm{base}} \sim \hat{\bb P}_{Y\mid A=a}$ of size $B$ 
        \hfill (target set $\mc T$)\;
        
        Draw plug-in samples $\hat Y_{i,k}\sim \hat{\bb P}_{Y\mid X=X_i,A=a}$ 
        for each $i\in\mc B$, $k=1,\dots,M$\;
        
        Flatten $\hat Y_{\mathrm{flat}}:=\{\hat Y_{i,k}\}_{i\in\mc B,k\le M}$ 
        (size $BM$)\;
        
        Subsample $K$ columns per $i$ to form $\hat Y_{\mathrm{OT}}$ 
        (size $N_s=BK$) \hfill (source set $\mc S$)\;
        
        \tcp{(Optional) compute exponential-tilt source weights $a_\ell$ on $\mc S$}
        
        \If{tilting enabled}{
            Compute features $g(\hat Y_{\mathrm{OT}})$ and logits 
            $r_\ell := \hat\lambda_a^\top g(\hat Y_{\mathrm{OT},\ell})$\;
            
            Set source weights $a_\ell \gets \mathrm{softmax}(r)_\ell$ 
            \hfill (sum to $1$)\;
        }
        \Else{
            Set $a_\ell \gets 1/N_s$ for $\ell=1,\dots,N_s$\;
        }
        
        \tcp{Weighted Sinkhorn: source=$\mc S$ (weights $a$), target=$\mc T$ (uniform)}
        
        Compute target duals 
        $v \gets \mathrm{SinkhornDualTarget}(\mc S,\mc T,a,\varepsilon,S)$ 
        \hfill ($v\in\bb R^{B}$)\;
        
        \tcp{Use the conditional $\hat\gamma^{\varepsilon}_{0\mid 1}(\cdot\mid y)$ to draw base points}
        
        Draw $\tilde Y^{\mathrm{obs}}_i 
        \sim \hat\gamma^{\varepsilon}_{0\mid 1}(\cdot \mid Y_i)$ 
        for each $i\in\mc B$\;
        
        Draw $\tilde Y^{\mathrm{hat}}_{i,k} 
        \sim \hat\gamma^{\varepsilon}_{0\mid 1}(\cdot \mid \hat Y_{i,k})$ 
        for all $i\in\mc B,k\le M$\;
        
        Compute $\ell_{\mathrm{obs},i} 
        \gets \ell_\theta(Y_i,\tilde Y^{\mathrm{obs}}_i)$ 
        for $i\in\mc B$\;
        
        Compute $\ell_{\mathrm{hat},i} 
        \gets \frac{1}{M}\sum_{k=1}^M 
        \ell_\theta(\hat Y_{i,k},\tilde Y^{\mathrm{hat}}_{i,k})$ 
        for $i\in\mc B$\;
        
        Set IPW weights 
        $w_i \gets \bm 1\{A_i=a\}/\hat\pi_a(X_i)$ 
        for $i\in\mc B$\;
        
        Accumulate DR loss contribution
        \[
        \widehat{\mc L}_a(\theta)
        \;\gets\;
        \frac{1}{B}\sum_{i\in\mc B}
        \Bigl(
        \ell_{\mathrm{hat},i} 
        + w_i\bigl(\ell_{\mathrm{obs},i}-\ell_{\mathrm{hat},i}\bigr)
        \Bigr).
        \]
    }
    
    Set $\widehat{\mc L}(\theta)\gets 
    \sum_{a\in\mathcal A}\widehat{\mc L}_a(\theta)$ and update 
    $\theta \gets 
    \textsc{OPT}(\theta,\nabla_\theta \widehat{\mc L}(\theta))$.\;
}
\Return{$\theta$}\;
\end{algorithm}

\begin{algorithm}[H]
\caption{Exponential Tilting of Plug--in Counterfactual Samples for Entropic OT Coupling}
\label{alg:exp-tilt}
\KwRequire{Observational data $\{(X_i,A_i,Y_i)\}_{i=1}^n$; treatment $a$;
plug--in sampler for $\hat{\bb P}_{Y(a)}$; moment map $s(\cdot)$;
propensity model $\hat\pi_a$; nuisance outcome model $\hat\mu$;
reservoir size $M$; regularization $\lambda_{\mathrm{reg}}$}

Draw a global plug--in reservoir $\{\hat y_j\}_{j=1}^M \sim \hat{\bb P}_{Y(a)}$\;

Compute moment features $G_j \gets s(\hat y_j)$ for $j=1,\dots,M$\;

Estimate doubly robust target moments
\[
\hat m_a \;\gets\; \frac{1}{n}\sum_{i=1}^n
\Bigl[
s(\hat Y_i(a))
+ \frac{\bm 1\{A_i=a\}}{\hat\pi_a(X_i)}\bigl(s(Y_i)-s(\hat Y_i(a))\bigr)
\Bigr]
\]

Solve for tilting parameter
\[
\hat\lambda
\;\in\;
\arg\min_{\lambda}
\Bigl\|
\sum_{j=1}^M w_j(\lambda)\,G_j - \hat m_a
\Bigr\|_2^2
\;+\;
\lambda_{\mathrm{reg}}\|\lambda\|_2^2,
\quad
w_j(\lambda)=\frac{\exp(\lambda^\top G_j)}{\sum_{\ell}\exp(\lambda^\top G_\ell)}
\]

Form tilted empirical distribution
\[
\hat{\bb P}^{\mathrm{tilt}}_{Y(a)} := \sum_{j=1}^M w_j(\hat\lambda)\,\delta_{\hat y_j}
\]

Use $\hat{\bb P}^{\mathrm{tilt}}_{Y(a)}$ as the target marginal when computing
the minibatch entropic OT coupling $\gamma^\varepsilon$\;
\end{algorithm}

 \clearpage

\section{Mathematical Appendix}\label{app:math}

\subsection{Technical Lemmas}\label{app:math:lemmas}

\subsubsection{Lemmas for Theorems in \cref{sec:geometry}}

\begin{lemma}[Bounded RN Derivative]\label{lemma:rn}
Assume $\exists \epsilon > 0$ such that $\pi_a(x) := \bb P(A=a|X=x) \geq \epsilon\; \forall a \in \mathcal A,\; \bb P_X$-a.e.
Fix $a\in\mathcal A$ and write $r_a(x):=\bb P(A=a) / \pi_a(X)$. Then $\bb P_{Y(a)} \sim \bb P_{Y\mid a}$ and, for $\bb P_{Y\mid a}$-a.e.\ $y$, the Radon-Nikodym derivative $g_a(y) := \frac{d \bb P_{Y(a)}}{d \bb P_{Y \mid a}}(y)$ exists and satisfies
\begin{align}
    g_a(y)
& = \bb E \left[r_a(X)  \; |\; A=a,\, Y=y\right] \\
\frac{\bb P(A=a)}{1-\varepsilon}
& \;\le\;
g_a(y)
\;\le\;
\frac{\bb P(A=a)}{\varepsilon}
\qquad \text{for $\bb P_{Y\mid a}$-a.e.\ $y$.}
\end{align}
\end{lemma}

\begin{proof}
Let $B\in\mathcal B(\bb X)$ be measurable. By the law of total probability,
\[
\bb P(X\in B, A=a)=\int_B \pi_a(x)\,\bb P_X(dx),
\qquad
\bb P(A=a)=\int_{\bb X}\pi_a(x)\,\bb P_X(dx).
\]
If $\bb P_X(B)=0$, then $\bb P(X\in B, A=a)=0$ hence $\bb P_{X\mid a}(B)=0$, so
$\bb P_{X\mid a}\ll \bb P_X$.

Conversely, suppose $\bb P_{X\mid a}(B)=0$. Then $\bb P(X\in B,A=a)=0$, i.e.
$\int_B \pi_a(x)\,\bb P_X(dx)=0$. Since $\pi_a(x)\ge \varepsilon>0$ $\bb P_X$-a.s., this implies
$\bb P_X(B)=0$. Hence $\bb P_X\ll \bb P_{X\mid a}$.
Therefore $\bb P_X\sim \bb P_{X\mid a}$.

Moreover, for $\bb P_X$-a.e.\ $x$,
\[
\frac{d\bb P_{X\mid a}}{d\bb P_X}(x)
=\frac{\pi_a(x)}{\bb P(A=a)}
\quad\text{and}\quad
r_a(x):=\frac{d\bb P_X}{d\bb P_{X\mid a}}(x)
=\frac{\bb P(A=a)}{\pi_a(x)}.
\]
By overlap, $r_a(x)$ is essentially bounded:
\[
\frac{\bb P(A=a)}{1-\varepsilon}\le r_a(x)\le \frac{\bb P(A=a)}{\varepsilon}
\qquad \text{for $\bb P_{X\mid a}$-a.e.\ $x$.}
\]

Let $h:\bb Y\to\bb R$ be bounded measurable. Using the representation from Eq.~\eqref{eq:int_dist},
$\bb P_{Y(a)}(\cdot)=\int \bb P_{Y\mid X=x,A=a}(\cdot)\,\bb P_X(dx)$,
Tonelli/Fubini yields
\begin{align*}
\int h(y)\,\bb P_{Y(a)}(dy)
&=\int_{\bb X}\int_{\bb Y} h(y)\,\bb P_{Y\mid X=x,A=a}(dy)\,\bb P_X(dx)\\
&=\int_{\bb X}\int_{\bb Y} h(y)\,\bb P_{Y\mid X=x,A=a}(dy)\, r_a(x)\,\bb P_{X\mid a}(dx)\\
&=\int_{\bb X\times \bb Y} h(y)\, r_a(x)\,\bb P_{X,Y\mid A=a}(dx,dy)\\
&=\int_{\bb Y} h(y)\left(\int_{\bb X} r_a(x)\,\bb P_{X\mid Y=y,A=a}(dx)\right)\bb P_{Y\mid a}(dy)\\
&=\int_{\bb Y} h(y)\,\bb E[r_a(X)\mid A=a,Y=y]\,\bb P_{Y\mid a}(dy).
\end{align*}
Since $h$ was arbitrary, the Radon--Nikodym theorem gives $\bb P_{Y(a)}\ll \bb P_{Y\mid a}$ and
\[
\frac{d \bb P_{Y(a)}}{d \bb P_{Y \mid a}}(y)
=\bb E[r_a(X)\mid A=a,Y=y]
=\bb E\left[\frac{\bb P(A=a)}{\bb P(A=a\mid X)}\ \Big|\ A=a,Y=y\right]
\]
for $\bb P_{Y\mid a}$-a.e.\ $y$.

By our earlier result, $r_a(X)\in\left[\frac{\bb P(A=a)}{1-\varepsilon},\frac{\bb P(A=a)}{\varepsilon}\right]$ a.s.
Conditioning preserves bounds, hence for $\bb P_{Y\mid a}$-a.e.\ $y$,
\[
\frac{\bb P(A=a)}{1-\varepsilon}
\le \bb E[r_a(X)\mid A=a,Y=y]
\le \frac{\bb P(A=a)}{\varepsilon}.
\]

Define $\tilde r_a(x):=\frac{d\bb P_{X\mid a}}{d\bb P_X}(x)=\frac{\pi_a(x)}{\bb P(A=a)}$,
which is strictly positive $\bb P_X$-a.s.\ (indeed $\tilde r_a(x)\ge \varepsilon/\bb P(A=a)$).
Repeat the calculation in Step 2 but now start from
\[
\int h(y)\,\bb P_{Y\mid a}(dy)
=\int_{\bb X}\int_{\bb Y} h(y)\,\bb P_{Y\mid X=x,A=a}(dy)\,\bb P_{X\mid a}(dx)
\]
and substitute $\bb P_{X\mid a}(dx)=\tilde r_a(x)\,\bb P_X(dx)$. This yields
\[
\int h(y)\,\bb P_{Y\mid a}(dy)
=\int_{\bb Y} h(y)\,\bb E[\tilde r_a(X)\mid Y(a)=y]\,\bb P_{Y(a)}(dy),
\]
so $\bb P_{Y\mid a}\ll \bb P_{Y(a)}$.
(Here $\bb E[\cdot\mid Y(a)=y]$ is well-defined on standard Borel spaces.)
Therefore $\bb P_{Y(a)}\sim \bb P_{Y\mid a}$.
\end{proof}

\subsubsection{Lemmas for Theorems in \cref{sec:est}}

\begin{lemma}[AIPW for counterfactual mean of a scalar outcome, \cite{kennedy2024semiparametric}]
\label{lem:aipw-cf}
Let \(P_0\) be the true data distribution and suppose the target
\(\bb P_{Y(a)}\) is defined by Eq.~\eqref{eq:int_dist}. Assume overlap, i.e.,
there exists \(\epsilon>0\) such that
\[
\pi_a(X):=P_0(A=a\mid X)\ge \epsilon
\qquad P_0\text{-a.s.}
\]
Then, the efficient influence function of the functional \(
T(\bb P_0)=\bb E_{\bb P_0}[Y(a)]\),
at $\bb P_0$ is
\[
\varphi_T(O)
=
\frac{\mathbf 1\{A=a\}}{\pi_a(X)}\bigl(Y-\zeta(X)\bigr)
+\zeta(X)-T(\bb P_0),
\]
where $\zeta(x)=\bb E_{\bb P_0}[Y\mid X=x,A=a]$ and
$\pi_a(x)=\bb P_0(A=a\mid X=x)$.
\end{lemma}

\begin{lemma}[AIPW for stratum mean, \cite{kennedy2024semiparametric}]
\label{lem:stratum}
Let $\bb P_0$ be the true distribution and let
\[
T(\bb P_0)=\bb E_{\bb P_0}[Y\mid A=a],
\]
with scalar $Y$ and $p_a=\bb P_0(A=a)>0$.
Then the efficient influence function of $T$ at $\bb P_0$ is
\[
\varphi_T(O)
=
\frac{\mathbf 1\{A=a\}}{p_a}
\bigl(Y-\bb E_{\bb P_0}[Y\mid A=a]\bigr).
\]
\end{lemma}

\begin{lemma}[EIF for a two-measure product functional]
\label{lem:nested}
Let $(\mc Y,\mc A)$ be a measurable space and let
$k:\mc Y\times\mc Y\to\bb R$ be measurable.
Fix a pair $(\bb P_{1,0},\bb P_{2,0})\in\mc P(\mc Y)\times\mc P(\mc Y)$ and define,
for $(\bb P_1,\bb P_2)$ in a neighbourhood of $(\bb P_{1,0},\bb P_{2,0})$,
\[
T(\bb P_1,\bb P_2)
:=
\iint k(y,y')\,\bb P_1(dy)\,\bb P_2(dy').
\]
Assume that, at $\bb P_{1,0}$ and $\bb P_{2,0}$, the linear functionals \(
\bb P \mapsto \int \phi(y)\,\bb P(dy)\)
are regular and pathwise differentiable for any measurable $\phi$ satisfying
$\int \phi^2\,d\bb P_{1,0}<\infty$ or $\int \phi^2\,d\bb P_{2,0}<\infty$,
respectively.
Define the partial integrals
\[
\phi_1(y):=\int k(y,y')\,\bb P_{2,0}(dy'),
\qquad
\phi_2(y'):=\int k(y,y')\,\bb P_{1,0}(dy).
\]
Then $T$ is regular and pathwise differentiable at
$(\bb P_{1,0},\bb P_{2,0})$, with efficient influence function
\[
\varphi_T
=
\varphi_{\Lambda_1}+\varphi_{\Lambda_2},
\]
where $\varphi_{\Lambda_1}$ is the EIF of the functional
$\bb P_1\mapsto\int \phi_1\,d\bb P_1$ at $\bb P_{1,0}$ and
$\varphi_{\Lambda_2}$ is the EIF of the functional
$\bb P_2\mapsto\int \phi_2\,d\bb P_2$ at $\bb P_{2,0}$.
\end{lemma}

\begin{proof}
Recall that the efficient influence function $\varphi_T$ of a regular
functional $T$ at a distribution $\bb P_0$ is the (unique) mean-zero
function satisfying
\[
\left.\frac{d}{d\varepsilon}\right|_{\varepsilon=0} T(\bb P_\varepsilon)
=
\bb E\!\left[\varphi_T(O)\,s(O)\right]
\]
for every regular parametric submodel $\{\bb P_\varepsilon\}$ through
$\bb P_0$ with score $s(O)$. We establish the result by applying this
pathwise derivative characterization to the product functional
\[
T(\bb P_1,\bb P_2)=\iint k(y,y')\,\bb P_1(dy)\,\bb P_2(dy').
\]

Fix a true pair $(\bb P_{1,0},\bb P_{2,0})$ and let
$\{(\bb P_{1,\varepsilon},\bb P_{2,\varepsilon})\}$ be a regular
parametric submodel through $(\bb P_{1,0},\bb P_{2,0})$ at
$\varepsilon=0$, with score $s(O)$. Write
$T_\varepsilon:=T(\bb P_{1,\varepsilon},\bb P_{2,\varepsilon})$.
Then
\[
T_\varepsilon
=
\int \phi_{1,\varepsilon}(y)\,\bb P_{1,\varepsilon}(dy),
\qquad
\phi_{1,\varepsilon}(y):=\int k(y,y')\,\bb P_{2,\varepsilon}(dy').
\]

Differentiating at $\varepsilon=0$ and applying the chain rule yields
\[
\left.\frac{d}{d\varepsilon}\right|_{\varepsilon=0}T_\varepsilon
=
\left.\frac{d}{d\varepsilon}\right|_{\varepsilon=0}
\int \phi_1(y)\,\bb P_{1,\varepsilon}(dy)
+
\left.\frac{d}{d\varepsilon}\right|_{\varepsilon=0}
\int \phi_{1,\varepsilon}(y)\,\bb P_{1,0}(dy),
\]
where
\[
\phi_1(y):=\int k(y,y')\,\bb P_{2,0}(dy').
\]

The first term corresponds to perturbing $\bb P_1$ while holding
$\bb P_2$ fixed. By the assumed pathwise differentiability of linear
functionals at $\bb P_{1,0}$, it equals
\[
\bb E\!\left[\varphi_{\Lambda_1}(O)\,s(O)\right],
\]
where $\varphi_{\Lambda_1}$ is the EIF of the functional
$\bb P_1\mapsto\int \phi_1\,d\bb P_1$ at $\bb P_{1,0}$.

The second term corresponds to perturbing $\bb P_2$ while holding
$\bb P_1$ fixed. Writing
\[
\phi_2(y'):=\int k(y,y')\,\bb P_{1,0}(dy),
\]
and using pathwise differentiability at $\bb P_{2,0}$, this term equals
\[
\bb E\!\left[\varphi_{\Lambda_2}(O)\,s(O)\right],
\]
where $\varphi_{\Lambda_2}$ is the EIF of the functional
$\bb P_2\mapsto\int \phi_2\,d\bb P_2$ at $\bb P_{2,0}$.

Summing the two contributions yields
\[
\left.\frac{d}{d\varepsilon}\right|_{\varepsilon=0}T_\varepsilon
=
\bb E\!\left[(\varphi_{\Lambda_1}+\varphi_{\Lambda_2})(O)\,s(O)\right],
\]
which identifies $\varphi_T=\varphi_{\Lambda_1}+\varphi_{\Lambda_2}$ as
the efficient influence function of $T$ at
$(\bb P_{1,0},\bb P_{2,0})$.
\end{proof}

\begin{lemma}[Empirical mean of independent random function (scalar/vector)]
\label{lem:rand_func_L2}
Fix $n \in \bb{N}$, let $\{Z_i\}_{i=1}^n$ be i.i.d.\ with distribution $P$, and let $S_n \in \mc{S}_n$ another random variable such that $\mc{T}_n:=\sigma(S_n)$ is 
independent from $\sigma(\{Z_i\}_{i=1}^n)$.
Let $\psi_n: \mc S_n \times \mc Z\to\bb R^d$ be a measurable function, and assume that
$\bb E[\|\psi_n(S_n, Z)\|_2^2 \mid \mc{T}_n]<\infty$ almost surely.
Define $P_n:=n^{-1}\sum_{i=1}^n\delta_{Z_i}$. Then 
\[
\left\|(P_n- P)\psi_n(S_n,\argdot)\right\|_2
=
O_{\bb P}\left(\frac{\|\psi_n\|_{L_2(\bb P)}}{\sqrt{n}}\right),
\]
\end{lemma}

\begin{proof}
Write
\[
(P_n- P)\psi_n(S_n,\argdot)
=
\frac{1}{n}\sum_{i=1}^n\Big(\psi_n(S_n, Z_i)-P\psi_n(S_n,\argdot)\Big),
\]
where $S_n \in \mc{S}_n$ is $\mc{T}_n$-measurable and $P\psi_n(S_n,\argdot):=\bb{E}[\psi_n(S_n,Z)\mid \mc{T}_n]\in\bb R^d$. 

Define $U_i:=\psi_n(S_n,Z_i)- P\psi_n(S_n,\argdot)$. Observe that $\bb E[{P}_n \psi_n(S_n,\argdot) \mid \mc T_n] = \bb{E}[\psi_n(S_n, Z)\mid \mc{T}_n] = P\psi_n(S_n,\argdot)$, so that $\bb{E}[(P_n - P)\psi_n(S_n,\argdot)] = 0$. 
Moreover, 
$\{U_i\}_{i=1}^n$ are conditionally i.i.d.\ given $\mathcal T_n$. Hence
\[
\bb E\left[\left\|(P_n-P)\psi_n(S_n,\argdot)\right\|_2^2\middle|\mathcal T_n\right]
=
\bb E\left[\left\|\frac{1}{n}\sum_{i=1}^n U_i\right\|_2^2\middle|\mathcal T_n\right]
=
\frac{1}{n^2}\sum_{i=1}^n \bb E\left[\|U_i\|_2^2\middle|\mathcal T_n\right]
=
\frac{1}{n}\bb E\left[\|U_1\|_2^2\middle|\mathcal T_n\right].
\]
Moreover,
\[
\bb E\left[\|U_1\|_2^2\middle|\mathcal T_n\right]
=
\bb E\left[\|\psi_n(S_n, Z)-\bb E[\psi_n(S_n,Z)\mid\mathcal T_n]\|_2^2\middle|\mathcal T_n\right]
=
\textrm{Var}\left(\psi_n(S_n,Z)\mid\mathcal T_n\right),
\]
and since $\textrm{Var}(V\mid\mathcal T)\le \bb E[\|V\|_2^2\mid\mathcal T] \leq \bb{E}[\|V\|_2^2]$ for any
square-integrable $V$,
\[
\textrm{Var}\left(\psi_n(S_n,Z)\mid\mathcal T_n\right)
\le
\bb E\left[\|\psi_n(S_n,Z)\|_2^2\middle|\mathcal T_n\right]
\leq
\bb{E}\left[\|\psi_n(S_n,Z)\|_2^2 \right] = \|\psi_n\|_{L_2(\bb P)} \;.
\]
Now, fix $M>0$ and set
\(
t:=M\,\frac{\|\psi_n\|_{L_2(\bb P)}}{\sqrt{n}}
\). By the conditional Markov (Chebyshev) inequality applied to the nonnegative random variable
$\|(P_n-P)\psi_n(S_n,\argdot)\|_2$, we have
\[
\bb P\left(\|(P_n-P)\psi_n(S_n,\argdot)\|_2^2 > t^2 \ \middle|\ \mathcal T_n\right)
\le
\frac{\bb E[\|(P_n-P)\psi_n(S_n,\argdot)\|_2^2\mid\mathcal T_n]}{t^2}
\le
\frac{1}{M^2}.
\]
Taking expectations over $\mathcal T$ yields the unconditional bound
\[
\bb P\left(\left\|(P_n-P)\psi_n(S_n,\argdot)\right\|_2
>
M\,\frac{\|\psi_n\|_{L_2(\bb P)}}{\sqrt{n}}\right)
\le
\frac{1}{M^2},
\]
which implies
$\|(P_n-P)\psi_n(S_n,\argdot)\|_2
=
O_{\bb{P}}\big(\|\psi_n\|_{L_2(\bb P)}/\sqrt{n}\big)$.
\end{proof}
\begin{lemma}[Second-order V-statistic with an independent random kernel (vector-valued)]
\label{lem:Vstat_indep_kernel_rig_vec}
Let $\{Z_i\}_{i=1}^n$ be i.i.d.\ from $\bb P$, and let $\mathcal T$ be a $\sigma$-algebra
independent of $\{Z_i\}_{i=1}^n$.
Let $k_n:\mc Z\times\mc Z\to\bb R^d$ be a measurable kernel that is $\mathcal T$-measurable
(hence fixed conditional on $\mathcal T$), and assume
$\bb E[\|k_n(Z_1,Z_2)\|_2^2]<\infty$ almost surely and $\bb E[\|k_n(Z_1,Z_1)\|_2^2]<\infty$ almost surely. 
Define $\bb P_n:=n^{-1}\sum_{i=1}^n\delta_{Z_i}$ and the (vector-valued) V-statistic
\[
V_n(k_n)
:=
\bb P_n^{\otimes 2}k_n
=
\frac{1}{n^2}\sum_{i=1}^n\sum_{j=1}^n k_n(Z_i,Z_j)\in\bb R^d.
\]
Then,
\[
\left\|V_n(k_n)-\bb P^{\otimes 2}k_n\right\|_2
=
O_p\left(
\frac{\|k_n\|_{L_2(\bb P^{\otimes 2})}}{\sqrt n}
\right),
\qquad
\|k_n\|_{L_2(\bb P^{\otimes 2})}^2:=\bb P^{\otimes 2}[\|k_n\|_2^2].
\]
Moreover, if $k_n$ is \emph{first-order degenerate} (conditional on $\mathcal T$), i.e.
\[
\bb E[k_n(z,Z')\mid \mathcal T]=\bb P^{\otimes 2}k_n
\quad\text{and}\quad
\bb E[k_n(Z',z)\mid \mathcal T]=\bb P^{\otimes 2}k_n
\qquad \text{for all } z,
\]
then the improved rate holds:
\[
\left\|V_n(k_n)-\bb P^{\otimes 2}k_n\right\|_2
=
O_p\left(
\frac{\|k_n\|_{L_2(\bb P^{\otimes 2})}}{n}
\right).
\]
\end{lemma}

\begin{proof}
We work conditionally on $\mathcal T$ throughout. Since $k_n$ is $\mathcal T$-measurable,
all objects defined from $k_n$ below are deterministic given $\mathcal T$, while
$\{Z_i\}_{i=1}^n$ remain i.i.d.\ from $\bb P$ conditional on $\mathcal T$.
For readability we keep the index $n$ explicit.

Let $\theta_n:=\bb P^{\otimes 2}k_n\in\bb R^d$ and define the first-order projections
\[
k_{n,1}(z):=\bb E[k_n(z,Z')]-\theta_n,
\qquad
k_{n,2}(z):=\bb E[k_n(Z',z)]-\theta_n,
\]
and the degenerate remainder
\[
\tilde k_n(z,z')
:=
k_n(z,z')-\theta_n-k_{n,1}(z)-k_{n,2}(z').
\]
Note that $\bb E[\tilde k_n(z,Z')]=0$ for all $z$ and
$\bb E[\tilde k_n(Z,z')]=0$ for all $z'$ (as vector equalities), where all expectations
are under $\bb P$ (and conditional on $\mathcal T$ if desired).

Expanding $k_n=\theta_n+k_{n,1}+k_{n,2}+\tilde k_n$ yields the decomposition
\[
V_n(k_n)-\theta_n
=
\underbrace{\bb P_n k_{n,1}}_{A_n}
+
\underbrace{\bb P_n k_{n,2}}_{B_n}
+
\underbrace{\bb P_n^{\otimes 2}\tilde k_n}_{C_n},
\]
where $A_n,B_n,C_n\in\bb R^d$.

\paragraph{$L_2$ bounds for $A_n$ and $B_n$.}
Conditional on $\mathcal T$, $k_{n,1}$ is deterministic and $Z_i$ are i.i.d., hence
\[
\bb E\left[\|A_n\|_2^2\middle|\mathcal T\right]
=
\bb E\left[\left\|\frac{1}{n}\sum_{i=1}^n \big(k_{n,1}(Z_i)-\bb Pk_{n,1}\big)\right\|_2^2\middle|\mathcal T\right]
=
\frac{1}{n}\bb E\left[\|k_{n,1}(Z)\|_2^2\middle|\mathcal T\right],
\]
where we used $\bb Pk_{n,1}=0$ and the standard identity
$\bb E\|\sum_{i=1}^n U_i\|_2^2=n\,\bb E\|U_1\|_2^2$ for i.i.d.\ mean-zero vectors.
Likewise,
\[
\bb E\left[\|B_n\|_2^2\middle|\mathcal T\right]
=
\frac{1}{n}\bb E\left[\|k_{n,2}(Z)\|_2^2\middle|\mathcal T\right].
\]

By Jensen and $\|u-v\|_2^2\le 2(\|u\|_2^2+\|v\|_2^2)$, for each $z$,
\begin{align*}
\|k_{n,1}(z)\|_2^2
& =
\|\bb E[k_n(z,Z')]-\theta_n\|_2^2
\\ & \le
2\|\bb E[k_n(z,Z')]\|_2^2+2\|\theta_n\|_2^2
\\ & \le
2\bb E[\|k_n(z,Z')\|_2^2]+2\bb E[\|k_n(Z,Z')\|_2^2].
\end{align*}
Taking expectation over $z\sim\bb P$ (conditional on $\mathcal T$) yields
\[
\bb E\left[\|k_{n,1}(Z)\|_2^2\middle|\mathcal T\right]\le 4\,\bb P^{\otimes 2}[\|k_n\|_2^2],
\qquad
\bb E\left[\|k_{n,2}(Z)\|_2^2\middle|\mathcal T\right]\le 4\,\bb P^{\otimes 2}[\|k_n\|_2^2].
\]
Hence,
\[
\bb E\left[\|A_n\|_2^2\middle|\mathcal T\right]\le \frac{4}{n}\bb P^{\otimes 2}[\|k_n\|_2^2],
\qquad
\bb E\left[\|B_n\|_2^2\middle|\mathcal T\right]\le \frac{4}{n}\bb P^{\otimes 2}[\|k_n\|_2^2].
\]

\paragraph{$L_2$ bound for $C_n$.}
Split $C_n$ into off-diagonal and diagonal parts:
\[
C_n
=
\frac{1}{n^2}\sum_{i\neq j}\tilde k_n(Z_i,Z_j)
+
\frac{1}{n^2}\sum_{i=1}^n \tilde k_n(Z_i,Z_i)
=:C_{n,\mathrm{off}}+C_{n,\mathrm{diag}}.
\]

For the diagonal term, by Jensen,
\[
\bb E\left[\|C_{n,\mathrm{diag}}\|_2^2\middle|\mathcal T\right]
\le
\frac{1}{n^4}\,n\,\bb E\left[\|\tilde k_n(Z,Z)\|_2^2\middle|\mathcal T\right]
=
\frac{1}{n^3}\bb E\left[\|\tilde k_n(Z,Z)\|_2^2\middle|\mathcal T\right].
\]
Using $\|u_1+\cdots+u_m\|_2^2\le m\sum_{\ell=1}^m\|u_\ell\|_2^2$ on
$\tilde k_n(Z,Z)=k_n(Z,Z)-\theta_n-k_{n,1}(Z)-k_{n,2}(Z)$ gives
\[
\|\tilde k_n(Z,Z)\|_2^2
\le
4\Big(\|k_n(Z,Z)\|_2^2+\|\theta_n\|_2^2+\|k_{n,1}(Z)\|_2^2+\|k_{n,2}(Z)\|_2^2\Big).
\]
Taking conditional expectations and using
$\|\theta_n\|_2^2=\|\bb E[k_n(Z,Z')]\|_2^2\le \bb P^{\otimes 2}[\|k_n\|_2^2]$ together with the
previous bounds on $k_{n,1},k_{n,2}$ yields
\[
\bb E\left[\|\tilde k_n(Z,Z)\|_2^2\middle|\mathcal T\right]
\le
c\Big(\bb E[\|k_n(Z,Z)\|_2^2]+\bb P^{\otimes 2}[\|k_n\|_2^2]\Big)
\]
for a universal constant $c>0$.
By the diagonal integrability assumption, the right-hand side is finite almost surely so
\[
\bb E\left[\|C_{n,\mathrm{diag}}\|_2^2\middle|\mathcal T\right]
=
O(n^{-3}),
\]
which is negligible at the $n^{-1}$ scale.

For the off-diagonal term, degeneracy implies that for $i\neq j$,
\[
\bb E[\tilde k_n(Z_i,Z_j)\mid Z_i,\mathcal T]=0
\quad\text{and}\quad
\bb E[\tilde k_n(Z_i,Z_j)\mid Z_j,\mathcal T]=0.
\]
A second-moment expansion shows only index-pairings with $(p,q)=(i,j)$ or $(p,q)=(j,i)$ contribute, hence
\[
\bb E\left[\|C_{n,\mathrm{off}}\|_2^2\middle|\mathcal T\right]
\le
\frac{2}{n^4}\,n(n-1)\,\bb E\left[\|\tilde k_n(Z_1,Z_2)\|_2^2\middle|\mathcal T\right]
\le
\frac{2}{n^2}\bb E\left[\|\tilde k_n(Z_1,Z_2)\|_2^2\middle|\mathcal T\right].
\]
Finally, since $\tilde k_n=k_n-\theta_n-k_{n,1}-k_{n,2}$, the same 4-sum inequality and the bounds above yield
\[
\bb E\left[\|\tilde k_n(Z_1,Z_2)\|_2^2\middle|\mathcal T\right]
\le
c'\,\bb P^{\otimes 2}[\|k_n\|_2^2]
\]
for a universal constant $c'>0$. Hence
\[
\bb E\left[\|C_{n,\mathrm{off}}\|_2^2\middle|\mathcal T\right]
\le
\frac{2c'}{n^2}\bb P^{\otimes 2}[\|k_n\|_2^2],
\qquad
\bb E\left[\|C_n\|_2^2\middle|\mathcal T\right]
\le
\frac{\tilde c}{n^2}\bb P^{\otimes 2}[\|k_n\|_2^2],
\]
for a universal constant $\tilde c>0$.

Using $\|u+v+w\|_2^2\le 3(\|u\|_2^2+\|v\|_2^2+\|w\|_2^2)$ we get
\begin{align*}
\bb E\left[\|V_n(k_n)-\theta_n\|_2^2\middle|\mathcal T\right]
& \le
3\bb E[\|A_n\|_2^2\mid\mathcal T]+3\bb E[\|B_n\|_2^2\mid\mathcal T]+3\bb E[\|C_n\|_2^2\mid\mathcal T] \\
& \le
\frac{12}{n}\,\bb P^{\otimes 2}[\|k_n\|_2^2] + O(n^{-2}) \;.
\end{align*}

Fix $M>0$ and set $t:=M\,\|k_n\|_{L_2(\bb P^{\otimes 2})}/\sqrt n$.
Applying conditional Markov to the nonnegative random variable
$\|V_n(k_n)-\bb P^{\otimes 2}k_n\|_2^2$ gives
\begin{align*}
\bb P\left(\|V_n(k_n)-\bb P^{\otimes 2}k_n\|_2>t\ \middle|\ \mathcal T\right)
& =
\bb P\left(\|V_n(k_n)-\bb P^{\otimes 2}k_n\|_2^2>t^2\ \middle|\ \mathcal T\right) \\
& \le
\frac{\bb E[\|V_n(k_n)-\bb P^{\otimes 2}k_n\|_2^2\mid\mathcal T]}{t^2} \\
& \le
\frac{12}{M^2}.
\end{align*}
Taking expectations over $\mathcal T$ yields the unconditional bound
\[
\bb P\left(\|V_n(k_n)-\bb P^{\otimes 2}k_n\|_2
>
M\,\frac{\|k_n\|_{L_2(\bb P^{\otimes 2})}}{\sqrt n}\right)
\le
\frac{12}{M^2},
\]
which implies
\[
\|V_n(k_n)-\bb P^{\otimes 2}k_n\|_2
=
O_p\left(\frac{\|k_n\|_{L_2(\bb P^{\otimes 2})}}{\sqrt n}\right).
\]

\paragraph{Degenerate case.}
Assume now that $k_n$ is first-order degenerate conditional on $\mathcal T$, i.e.
$\bb E[k_n(z,Z')\mid\mathcal T]=\theta_n$ and $\bb E[k_n(Z',z)\mid\mathcal T]=\theta_n$
for all $z$. Then $k_{n,1}\equiv 0$ and $k_{n,2}\equiv 0$, so $A_n=B_n\equiv 0$ and
\[
V_n(k_n)-\theta_n=C_n=\bb P_n^{\otimes 2}\tilde k_n.
\]
The bounds above already give
\[
\bb E\left[\|C_n\|_2^2\middle|\mathcal T\right]
\le
\frac{\tilde c}{n^2}\,\bb P^{\otimes 2}[\|k_n\|_2^2].
\]
Fix $M>0$ and set $t:=M\,\|k_n\|_{L_2(\bb P^{\otimes 2})}/n$.
Applying conditional Markov to $\|C_n\|_2^2$ yields
\[
\bb P\left(\|V_n(k_n)-\bb P^{\otimes 2}k_n\|_2>t\ \middle|\ \mathcal T\right)
\le
\frac{\tilde c}{M^2}.
\]
Taking expectations over $\mathcal T$ gives
\[
\|V_n(k_n)-\bb P^{\otimes 2}k_n\|_2
=
O_p\left(\frac{\|k_n\|_{L_2(\bb P^{\otimes 2})}}{n}\right),
\]
as claimed.
\qedhere
\end{proof}

\begin{lemma}[Empirical Process Decomposition for Conditional Distribution Plugin]
\label{lem:emp_base_plugin_ratio}
Fix $a$ with $p_a:=\bb P(A=a)>0$. Let $O=(X,A,Y)$ and let $\{O_i\}_{i=1}^n$ be i.i.d.\ from $\bb P$.
Define $\bb P_n:=n^{-1}\sum_{i=1}^n \delta_{O_i}$, $\hat p_a:=\bb P_n[\bm 1\{A=a\}]$,
and the empirical conditional measure
\[
\hat{\bb P}_{Y\mid A=a}[g]
:=
\frac{\bb P_n[\bm 1\{A=a\}\,g(Y)]}{\hat p_a},
\qquad
\bb P_{Y\mid a}[g]
:=
\frac{\bb P[\bm 1\{A=a\}\,g(Y)]}{p_a}.
\]
Let $\psi:\mc O\times\mc Y\to\bb R$ be measurable and let $\hat\psi:\mc O\times\mc Y\to\bb R$ be measurable and $\mathcal T_n$-measurable (e.g.\ obtained by cross-fitting). Define
$\Delta:=\hat\psi-\psi$ and assume
\(
\psi,\Delta \in L_2(\bb P\otimes \bb P_{Y\mid a})
\).
Lastly, define the partial expectations of $\hat\psi, \psi$,
\[
\hat\varphi(O):=\hat{\bb P}_{Y\mid A=a}\big[\hat\psi(O,\cdot)\big],
\qquad
\varphi(O):=\bb P_{Y\mid a}\big[\psi(O,\cdot)\big].
\]
Then, the following holds
\begin{align}
(\bb P_n-\bb P)(\hat\varphi-\varphi)
&= O_p\left(\frac{\|\Delta\|_{L_2(\bb P\otimes \bb P_{Y\mid a})}}{\sqrt n}\right).
\end{align}
\end{lemma}

\begin{proof}
Let $\mathcal T_n:=\sigma(\{O_i:i\in\mathcal I_{\mathrm{tr}}\})$ be the $\sigma$-algebra generated by the training fold.
All statements below are conditional on $\mathcal T_n$.
In particular, $\hat\psi$ (hence $\Delta:=\hat\psi-\psi$) is $\mathcal T_n$-measurable and therefore fixed given $\mathcal T_n$,
while the evaluation observations remain i.i.d.\ from $\bb P$ conditional on $\mathcal T_n$.

To start, note for each $O$ that
\begin{align*}
\hat\varphi(O)-\varphi(O)
&=
\hat{\bb P}_{Y\mid A=a}\big[\psi(O,\cdot)+\Delta(O,\cdot)\big]
-
\bb P_{Y\mid a}\big[\psi(O,\cdot)\big]
\\
&=
\underbrace{\bb P_{Y\mid a}\big[\Delta(O,\cdot)\big]}_{H_\Delta(O)}
+
\underbrace{\big(\hat{\bb P}_{Y\mid A=a}-\bb P_{Y\mid a}\big)\big[\psi(O,\cdot)\big]}_{T_\psi(O)}
+
\underbrace{\big(\hat{\bb P}_{Y\mid A=a}-\bb P_{Y\mid a}\big)\big[\Delta(O,\cdot)\big]}_{T_\Delta((O)}.
\end{align*}

\paragraph{\underline{$H_{\Delta}$}.} For the first term $H_{\Delta}$, by Jensen's inequality we have
\[
\|H_\Delta\|_{L_2(\bb P)}^2
=
\bb E\left[\left(\int \Delta(O,\tilde y)\,\bb P_{Y\mid a}(d\tilde y)\right)^2\right]
\le
\bb E\left[\int \Delta(O,\tilde y)^2\,\bb P_{Y\mid a}(d\tilde y)\right]
=
\|\Delta\|_{L_2(Q)}^2.
\]
Applying Lemma~\ref{lem:rand_func_L2} conditional on $\mathcal T_n$ with $\psi_n=H_\Delta$ (which is $\mathcal T_n$-measurable and hence fixed given $\mathcal T_n$) gives
\[
(\bb P_n-\bb P)H_\Delta
=
O_p\left(\frac{\|H_\Delta\|_{L_2(\bb P)}}{\sqrt n}\right)
=
O_p\left(\frac{\|\Delta\|_{L_2(Q)}}{\sqrt n}\right),
\]

where $Q:=\bb P\otimes\bb P_{Y\mid a}$.

\paragraph{\underline{$T_\psi$}.}
Define the centered kernel $\bar \psi$,
\[
\bar\psi(o,y)
:=
\psi(o,y)-\bb P_{Y\mid a}\big[\psi(o,\cdot)\big],
\qquad
\bb P_{Y\mid a}\big[\bar\psi(o,\cdot)\big]=0
\ \ \forall o,
\]
and note that we can write
\[
T_\psi(o)
=
\big(\hat{\bb P}_{Y\mid A=a}-\bb P_{Y\mid a}\big)\big[\psi(o,\cdot)\big]
=
\hat{\bb P}_{Y\mid A=a}\big[\bar\psi(o,\cdot)\big]
=
\frac{1}{\hat p_a}\,
\bb P_n\left[\,\bm 1\{A=a\}\,\bar\psi(o,Y)\right]\]
where \(
\hat p_a:=\bb P_n[\bm 1\{A=a\}]
\). Since $\bm 1\{A=a\}\in L_2(\bb P)$ and $p_a:=\bb P(A=a)>0$,
\[
\hat p_a-p_a=(\bb P_n-\bb P)\bm 1\{A=a\}=O_p(n^{-1/2}),
\qquad\text{and hence}\qquad
\frac{1}{\hat p_a}=O_p(1).
\]
Therefore it suffices to bound
\[
(\bb P_n-\bb P)\Big[o\mapsto \bb P_n\left(\bm 1\{A=a\}\,\bar\psi(o,Y)\right)\Big].
\]

Define the kernel $k:\mc O \times\mc O\to\bb R$ by
\[
k(o,o')
:=
\bm 1\{a'=a\}\,\bar\psi(o,y'),
\qquad
o=(x,a,y),\ o'=(x',a',y').
\]
Then
\[
(\bb P_n-\bb P)T_\psi
=
\frac{1}{\hat p_a}\,
\Big((\bb P_n-\bb P)\otimes\bb P_n\Big)k.
\]

\medskip
Adding and subtracting $\bb P$ in the second argument gives
\[
\Big((\bb P_n-\bb P)\otimes\bb P_n\Big)k
=
\Big((\bb P_n-\bb P)\otimes(\bb P_n-\bb P)\Big)k
+
\Big((\bb P_n-\bb P)\otimes\bb P\Big)k.
\]
The second term vanishes since, for every fixed $o$,
\[
\bb E[k(o,O')]
=
\bb E\left[\bm 1\{A'=a\}\,\bar\psi(o,Y')\right]
=
p_a\,\bb E\left[\bar\psi(o,Y')\mid A'=a\right]
=
p_a\,\bb P_{Y\mid a}\left[\bar\psi(o,\cdot)\right]
=
0.
\]
Hence
\[
\Big((\bb P_n-\bb P)\otimes\bb P_n\Big)k
=
\Big((\bb P_n-\bb P)\otimes(\bb P_n-\bb P)\Big)k.
\]

\medskip
Define
\[
k_2(o'):=\bb E[k(O,o')],
\qquad
\tilde k(o,o'):=k(o,o')-k_2(o').
\]
Then $\tilde k$ is first-order degenerate in both arguments:
\[
\bb E[\tilde k(o,O')]
=
\bb E[k(o,O')]-\bb E[k_2(O')]
=
0,
\qquad
\bb E[\tilde k(O,o')]
=
\bb E[k(O,o')]-k_2(o')
=
0.
\]
Moreover, since $k_2$ depends only on the second argument and
$(\bb P_n-\bb P)[1]=0$,
\[
\Big((\bb P_n-\bb P)\otimes(\bb P_n-\bb P)\Big)k_2
=
(\bb P_n-\bb P)[1]\cdot(\bb P_n-\bb P)[k_2]
=
0,
\]
and therefore
\[
\Big((\bb P_n-\bb P)\otimes(\bb P_n-\bb P)\Big)k
=
\Big((\bb P_n-\bb P)\otimes(\bb P_n-\bb P)\Big)\tilde k.
\]

\medskip
Since $\tilde k\in L_2(\bb P^{\otimes 2})$ and is first-order degenerate in both arguments, we have 
\[
\Big((\bb P_n-\bb P)\otimes(\bb P_n-\bb P)\Big)\tilde k
= \Big(\bb P_n^{\otimes 2}-\bb P ^{\otimes 2}\Big)\tilde k
\]
Thus, the degenerate case of Lemma~\ref{lem:Vstat_indep_kernel_rig_vec} (applied conditionally on
$\mathcal T_n$) yields
\[
\Big((\bb P_n-\bb P)\otimes(\bb P_n-\bb P)\Big)\tilde k
=
O_p\left(\frac{\|\tilde k\|_{L_2(\bb P^{\otimes 2})}}{n}\right).
\]
Consequently,
\[
(\bb P_n-\bb P)T_\psi
=
\frac{1}{\hat p_a}\,
O_p\left(\frac{\|\tilde k\|_{L_2(\bb P^{\otimes 2})}}{n}\right).
\]

\medskip
We have
\begin{align*}
\|k\|_{L_2(\bb P^{\otimes 2})}^2
& =
\bb E\big[\bm 1\{A'=a\}\,\bar\psi(O,Y')^2\big] \\
& =
p_a\,\bb E\big[\bar\psi(O,Y)^2\mid A=a\big]
 \\
 & \le
p_a\,\bb E\big[\psi(O,Y)^2\mid A=a\big]
=
p_a\,\|\psi\|_{L_2(Q)}^2,
\end{align*}
where $Q:=\bb P\otimes\bb P_{Y\mid a}$.
Moreover, by Jensen's inequality,
\[
\|k_2\|_{L_2(\bb P)}^2
=
\bb E\left[\big(\bb E[k(O,O')]\big)^2\right]
\le
\bb E\left[\bb E[k(O,O')^2]\right]
=
\|k\|_{L_2(\bb P^{\otimes 2})}^2.
\]
Using $\|u-v\|_2^2\le 2(\|u\|_2^2+\|v\|_2^2)$,
\[
\|\tilde k\|_{L_2(\bb P^{\otimes 2})}
=
\|k-k_2\|_{L_2(\bb P^{\otimes 2})}
\lesssim
\|k\|_{L_2(\bb P^{\otimes 2})}
\lesssim
\|\psi\|_{L_2(Q)}.
\]
Combining the above bounds and using $1/\hat p_a=O_p(1)$,
\[
(\bb P_n-\bb P)T_\psi
=
O_p\left(\frac{\|\psi\|_{L_2(Q)}}{n}\right).
\]

\paragraph{\underline{$T_{\Delta}$}.}
The argument for $T_\Delta$ is identical with $\psi$ replaced by $\Delta$ throughout, yielding
\(
T_\Delta
=
O_p\left(\frac{\|\Delta\|_{L_2(Q)}}{n}\right)
\).

Putting the bounds together,
\[
(\bb P_n-\bb P)(\hat\varphi-\varphi)
=
O_p\left(\frac{\|\Delta\|_{L_2(Q)}}{\sqrt n}\right)
+
O_p\left(\frac{1}{n}\right),
\]
and since $\|\Delta\|_{L_2(Q)}=O_p(1)$ by assumption, the $n^{-1}$ term is asymptotically negligible, yielding the claimed rate.
\end{proof}

\subsection{Main Results}

\subsubsection{Proofs for Theorems in \cref{sec:geometry}}
\begin{proof}[\textbf{Proof of \cref{thm:support-tails}}]
For the proof of (i) (equal support), note that by the positivity assumption, the mixing laws $\mathbb P_X$ and $\mathbb P_{X\mid a}$ are mutually absolutely continuous (see Lemma~\ref{lemma:rn}). Therefore, both Radon-Nikodym derivatives exist and are given by
\[
\frac{d\mathbb P_X}{d\mathbb P_{X\mid a}}(x)=\frac{\mathbb P(A=a)}{\mathbb P(A=a\mid X=x)}
=: r_a(x)\in (0,\infty)
\]
Now, note that for any measurable $B \in \mc B(\bb R^p)$ we can write
\begin{align*}
\mathbb P_{Y(a)}(B)& =\int \mathbb P(Y\in B\mid A=a,X=x)\,\mathbb P_X(dx) \\
\mathbb P_{Y\mid a}(B)& =\int \mathbb P(Y\in B\mid A=a,X=x)\,\mathbb P_{X\mid a}(dx).
\end{align*}
By Lemma~\ref{lemma:rn}, the equivalence of the mixing distributions implies $\mathbb P_{Y(a)}\ll\mathbb P_{Y\mid a}$ and $\mathbb P_{Y\mid a}\ll\mathbb P_{Y(a)}$ with
\[
\frac{d\mathbb P_{Y(a)}}{d\mathbb P_{Y\mid a}}(y)
= \mathbb E\left[r_a(X)\,\big|\, A=a,\,Y=y\right]
=: g_a(y)\in (0,\infty)\quad \mathbb P_{Y\mid a}\text{-a.s.}
\]
Equivalence of measures implies equality of supports:
$\supp(\mathbb P_{Y(a)})=\supp(\mathbb P_{Y\mid a})$.

\medskip

For the proof of (ii) (matching tail class), fix a measurable $\varphi:\mathcal Y\to[0,\infty]$ and set
$h(x):=\mathbb E[\varphi(Y)\mid A=a,X=x]\in[0,\infty]$.
Using the definition of $r_a$, we have
\[
\mathbb E[\varphi(Y(a))]
=\int h(x)\,r_a(x)\,\mathbb P_{X\mid a}(dx).
\]

Under the strong positivity assumption
$\exists\,\varepsilon>0$ such that
$\varepsilon\le \mathbb P(A=a\mid X=x)$
$\mathbb P_X$-almost everywhere and for all $a \in \mathcal A$. Thus, the Radon–Nikodým derivative admits the
bounds
\[
\mathbb P(A=a)
\;\le\;
r_a(x)
\;\le\;
\frac{\mathbb P(A=a)}{\varepsilon}
\qquad\text{for }\mathbb P_{X\mid a}\text{-a.e.\ }x .
\]
Hence, for every nonnegative $h$,
\[
\mathbb P(A=a)\int h\,d\mathbb P_{X\mid a}
\;\le\;
\int h\,d\mathbb P_X
\;\le\;
\frac{\mathbb P(A=a)}{\varepsilon}\int h\,d\mathbb P_{X\mid a}.
\]
Applying this to $h(x)=\mathbb E[\varphi(Y)\mid A=a,X=x]$ gives
\[
\mathbb P(A=a)\,\mathbb E[\varphi(Y)\mid A=a]
\;\le\;
\mathbb E[\varphi(Y(a))]
\;\le\;
\frac{\mathbb P(A=a)}{\varepsilon}\,\mathbb E[\varphi(Y)\mid A=a].
\]

In particular,
\[
\mathbb E[\varphi(Y(a))]<\infty
\quad\Longleftrightarrow\quad
\mathbb E[\varphi(Y)\mid A=a]<\infty.
\]
\end{proof}
\begin{remark}
If only positivity is assumed,
the lower bound on $r_a$ remains valid
and the above inequalities yield the one-sided implications
\[
\mathbb E[\varphi(Y)\mid A=a]=\infty\Rightarrow
\mathbb E[\varphi(Y(a))]=\infty,
\qquad
\mathbb E[\varphi(Y(a))]<\infty\Rightarrow
\mathbb E[\varphi(Y)\mid A=a]<\infty .
\]
Thus the two distributions share the same tail class under overlap,
and at least one-sided integrability under mere positivity.
\end{remark}

\begin{proof}[\textbf{Proof of \cref{thm:Wp_mixture_stability}}]
Recall the mixture representations from the main text,
\begin{align*}
\bb P_{Y(a)}(B)
&= \int_{\mc X} K_a(x,B)\, \bb P_X(dx),\\
\bb P_{Y\mid a}(B)
&= \int_{\mc X} K_a(x,B)\, \bb P_{X\mid a}(dx),
\end{align*}
where \(K_a(x,\cdot):=\bb P(Y\in\cdot\mid X=x,A=a)\).
Write \(\mu:=\bb P_X\) and \(\nu:=\bb P_{X\mid a}\), so that
\(\bb P_{Y(a)}=\mu K_a\) and \(\bb P_{Y\mid a}=\nu K_a\).

We first bound the Wasserstein-\(2\) distance between these two distributions.
For each \((x,x')\in\mc X\times\mc X\), let \(\pi_{x,x'}\) be an optimal
coupling between \(K_a(x,\cdot)\) and \(K_a(x',\cdot)\) for the quadratic cost
\(\|y-y'\|_2^2\), i.e.
\[
\pi_{x,x'}\in\Gamma\big(K_a(x,\cdot),K_a(x',\cdot)\big),
\qquad
\int_{\mc Y\times\mc Y} \|y-y'\|_2^2\,\pi_{x,x'}(dy,dy')
=
W_2^2\big(K_a(x,\cdot),K_a(x',\cdot)\big).
\]
Such optimal couplings exist under the finite second-moment conditions implicit
in the \(W_2\)-Lipschitz assumption. Since the spaces are standard Borel and
\(K_a\) is a probability kernel, we may choose
\((x,x')\mapsto\pi_{x,x'}\) to be a measurable optimal-coupling kernel by
standard measurable-selection arguments
\citep[e.g.][]{aliprantis2006infinite}.

Now fix any coupling \(\gamma\in\Gamma(\mu,\nu)\) on \(\mc X\times\mc X\).
Define a measure \(\Pi\) on \(\bb R^p\times\bb R^p\) by mixing
\(\pi_{x,x'}\) along \(\gamma\):
\[
\Pi(C)
:=
\int_{\mc X\times\mc X}\pi_{x,x'}(C)\,\gamma(dx,dx'),
\qquad
C\in\mathcal B(\bb R^p\times\bb R^p).
\]
Since \(\pi_{x,x'}\) has marginals \(K_a(x,\cdot)\) and \(K_a(x',\cdot)\),
and \(\gamma\) has marginals \(\mu\) and \(\nu\), \(\Pi\) has marginals
\(\mu K_a\) and \(\nu K_a\). Hence
\(\Pi\in\Gamma(\mu K_a,\nu K_a)\).

By Tonelli's theorem,
\begin{align*}
\int_{\mc Y\times\mc Y} \|y-y'\|_2^2\,\Pi(dy,dy')
&=
\int_{\mc X\times\mc X}
\left(
\int_{\mc Y\times\mc Y} \|y-y'\|_2^2\,
\pi_{x,x'}(dy,dy')
\right)\gamma(dx,dx')\\
&=
\int_{\mc X\times\mc X}
W_2^2\big(K_a(x,\cdot),K_a(x',\cdot)\big)\,\gamma(dx,dx').
\end{align*}
By the assumed \(W_2\)-Lipschitz regularity of the conditional outcome
distributions,
\[
W_2\big(K_a(x,\cdot),K_a(x',\cdot)\big)
\le
L_a\,\|x-x'\|_2,
\]
and hence
\[
\int_{\mc Y\times\mc Y} \|y-y'\|_2^2\,\Pi(dy,dy')
\le
L_a^2
\int_{\mc X\times\mc X}\|x-x'\|_2^2\,\gamma(dx,dx').
\]
Taking the infimum over all \(\gamma\in\Gamma(\mu,\nu)\) yields
\[
W_2^2\big(\bb P_{Y(a)},\bb P_{Y\mid a}\big)
\le
L_a^2\,W_2^2\big(\bb P_X,\bb P_{X\mid a}\big).
\]

Since \(\bb P_{Y\mid a}\) is absolutely continuous on \(\bb R^p\), Brenier's
theorem guarantees the existence of a \(\bb P_{Y\mid a}\)-a.e. unique
quadratic-cost optimal transport map
\(f_a:\bb R^p\to\bb R^p\) transporting \(\bb P_{Y\mid a}\) to
\(\bb P_{Y(a)}\). Moreover, since \(f_a\) is optimal,
\[
\|f_a-\mathrm{id}\|_{L^2(\bb P_{Y\mid a})}^2
=
W_2^2\big(\bb P_{Y\mid a},\bb P_{Y(a)}\big).
\]
Combining this identity with the previous bound and taking square roots gives
\[
\|f_a-\mathrm{id}\|_{L^2(\bb P_{Y\mid a})}
\le
L_a\,W_2\big(\bb P_X,\bb P_{X\mid a}\big),
\]
which concludes the proof.
\end{proof}

\begin{proof}[\textbf{Proof of \cref{thm:invariant_confounding}}]
For convenience, define
\[
K_a(x,\cdot):=\bb P_{Y\mid X=x,A=a}(\cdot),
\]
and let
\[
Q_a:=\phi_\#\bb P_{Y\mid A=a}.
\]
Since $\phi(Y)\indep X\mid A=a$, we have
\[
\phi_\#K_a(x,\cdot)=Q_a
\qquad
\bb P_{X\mid A=a}\text{-a.e. }x.
\]
Moreover,
\[
\frac{d\bb P_{X\mid A=a}}{d\bb P_X}(x)
=
\frac{\pi_a(x)}{p_a},
\qquad p_a:=\bb P(A=a)>0.
\]
By the assumed strict positivity of $\pi_a$, this derivative is strictly
positive $\bb P_X$-a.e., and hence
\[
\bb P_X\ll\bb P_{X\mid A=a}.
\]
Therefore
\[
\phi_\#K_a(x,\cdot)=Q_a
\qquad
\bb P_X\text{-a.e. }x.
\]

Now, for any $B\in\mathcal B(\mathcal Z)$ (the Borel $\sigma$-algebra of $\mc Z$),
\[
\begin{aligned}
\phi_\#\bb P_{Y(a)}(B)
&=
\int K_a(x,\phi^{-1}(B))\,d\bb P_X(x)\\
&=
\int Q_a(B)\,d\bb P_X(x)
=
Q_a(B).
\end{aligned}
\]
which immediately implies that
\(
    \phi_{\#}\bb P_{Y(a)}
    =
    \phi_{\#}\bb P_{Y\mid A=a}
    =
    Q_a .
\)

It remains to prove the feature-preserving transport claim. For notational convenience in this part, define
\(
    \mu_0 := \bb P_{Y\mid A=a}\) and
    \(\mu_1 := \bb P_{Y(a)} .
\)
We have just shown that
\(\phi_{\#}\mu_0=\phi_{\#}\mu_1=Q_a .
\)
Since all spaces are standard Borel, we can disintegrate these distributions along the value of the feature map \(\phi\) as
\[
    \mu_0(dy)
    =
    \int \mu_{0,z}(dy)\,Q_a(dz),
    \qquad
    \mu_1(dy)
    =
    \int \mu_{1,z}(dy)\,Q_a(dz),
\]
where
\(
    \mu_{0,z}
    =
    \bb P_{Y\mid A=a,\phi(Y)=z}\) and \(
    \mu_{1,z}
    =
    \bb P_{Y(a)\mid \phi(Y(a))=z}.
\)
By the defining property of conditional laws given \(\phi(Y)\), these
versions may be chosen to be supported on the corresponding fibers:
\[
    \mu_{0,z}(\phi^{-1}(z))=1,
    \qquad
    \mu_{1,z}(\phi^{-1}(z))=1,
    \qquad Q_a\text{-a.e. }z.
\]
In other words, after conditioning on the event \(\phi(Y)=z\), the remaining
variation in \(Y\) lies entirely within the fiber
\(\phi^{-1}(z)=\{y:\phi(y)=z\}\).

Now, recall that by assumption, there exists a measurable family of maps
\(
    \{T_z:\mathcal Y\to\mathcal Y\mid z\in\mathcal Z\},
\)
such that, for \(Q_a\)-a.e. \(z\),
\(
    (T_z)_{\#}\mu_{0,z}=\mu_{1,z}.
\)
Since \(\mu_{1,z}\) is supported on \(\phi^{-1}(z)\), this also implies
\[
    \phi(T_z(y))=z
    \qquad \mu_{0,z}\text{-a.e. }y, \qquad  \text{\(Q_a\)-a.e. \(z\)}.
\]
Now, define
\(
    f_a(y) := T_{\phi(y)}(y),
\)
with an arbitrary measurable definition on any null set where the above
conditional maps are not specified. By the assumed joint measurability of the
family \(z\mapsto T_z\), the map \(f_a\) is measurable.

To verify the pushforward property, let \(h:\mathcal Y\to\mathbb R\) be
bounded and measurable. Using disintegration,
\[
\begin{aligned}
\int h(f_a(y))\,\mu_0(dy)
&=
\int
\int h(T_z(y))\,\mu_{0,z}(dy)\,Q_a(dz) \\
&=
\int
\int h(y')\,\mu_{1,z}(dy')\,Q_a(dz) \\
&=
\int h(y')\,\mu_1(dy').
\end{aligned}
\]
Hence \((f_a)_{\#}\mu_0=\mu_1\), which by definition is the same as
\(
    (f_a)_{\#}\bb P_{Y\mid A=a}
    =
    \bb P_{Y(a)}.
\)

Finally, since \(\mu_{0,z}\) is supported on \(\phi^{-1}(z)\) and
\(T_z(y)\in\phi^{-1}(z)\) for \(\mu_{0,z}\)-a.e. \(y\), we have
\[
\begin{aligned}
\mu_0\{y:\phi(f_a(y))\neq \phi(y)\}
&=
\int
\mu_{0,z}\{y:\phi(T_z(y))\neq z\}\,Q_a(dz) =0.
\end{aligned}
\]
Therefore
\(
    \phi(f_a(Y))=\phi(Y)
    ,\; \bb P_{Y\mid A=a}\text{-a.s.}
\), which proves the claim.
\end{proof}

\subsubsection{Assumptions for Theorems and Proofs in \cref{sec:est}}

\paragraph{\underline{Notation, Definitions and Assumptions}}

\begin{definition}[Standing Objects for EIF analysis]\label{def:eif_terms},
Fix a treatment level $a \in \mathcal A$ and a parameter $\theta_a \in \Theta$.
Throughout the results in this section, we use the following notation.

    \begin{itemize}
\item $O := (X,A,Y) \sim \bb P$ denotes a generic observational draw.

\item  For $O \sim \bb P$, we denote the shorthand $\bb P(f) := \bb E[f(O)]$ for any measurable $f$. 

\item $\mc D_1 = (O_i)_{i=1}^n \sim \bb P$ and $\mc D_2 = (\tilde O_i)_{i=1}^n \sim \bb P$ are two independent datasets of i.i.d. observations.

\item $\bb P_{Y\mid a}$ denotes the true conditional outcome law, and
$\hat{\bb P}_{Y\mid A=a} := n_a^{-1}\sum_{i:A_i=a}\delta_{Y_i}$
its empirical analogue on $\mc D_1$.

\item $p_a := \bb P(A=a)$ and $\hat p_a$ denotes its estimator using $\mc D_1$

\item $\pi_a(x) := \bb P(A=a \mid X=x)$ and $\hat\pi_a$ denotes its estimator using $\mc D_2$. 
\item $r_{\theta_a} : \mc Y \times \mc Y \to \bb R$ denotes the score kernel
associated with the moment condition $m(\theta_a,\bb P)$
(cf.\ \cref{sec:est:debiased} of the main text).

\item $h_{a,\theta_a}(x,\tilde y)
:= \bb E\left[r_{\theta_a}(Y,\tilde y)\mid A=a,X=x\right]$,
and $\hat h_{a,\theta_a}$ its estimator using $\mc D_2$.

\item $\mu_{h_{\theta_a}}(x)
:= \bb E[h_{\theta_a}(x,Y)\mid A=a]$,
and $\hat\mu_{\hat h_{\theta_a}}(x)
:= \int \hat h_{a,\theta_a}(x,\tilde y)\,\hat{\bb P}_{Y\mid A=a}(d\tilde y)$.

\item $\chi_{\theta_a}(y)
:=
\bb E\left[h_{a,\theta_a}(X,y)\right]
=
\int h_{a,\theta_a}(x,y)\,\bb P_X(dx)$ and $
\hat \chi_{\theta_a}(y):=\hat {\bb P}_{X}\left[\hat h_{a,\theta_a}(\cdot ,y)\right]$ the corresponding plug-in estimator, where $\hat {\bb P}_{X}$ is the empirical distribution of $\{X_i\}$ on $\mc D_2$. 

\item $
m(\theta_a,\bb P)
:=
\bb E\left[\chi_{\theta_a}(Y)\mid A=a\right]
=
\int \chi_{\theta_a}(y)\,\bb P_{Y\mid a}(dy)$ is the moment condition and $
m(\theta_a,\hat{\bb P})
:=
\int \hat \chi_{\theta_a}(y)\,\hat{\bb P}_{Y\mid A=a}(dy)
=
\hat{\bb P}_{Y\mid A=a}\left[\hat \chi_{\theta_a}\right]$ is its plug-in estimator.

\item  
\(
\psi_{1,\theta_a}(O,y)
:=
\frac{\bm 1\{A=a\}}{\pi_a(X)}
\Big(r_{\theta_a}(Y,y)-h_{a,\theta_a}(X,y)\Big)
+
h_{a,\theta_a}(X,y)
-
\bb P_X[h_{a,\theta_a}(\cdot, y)]\) is the EIF of $\bb E[r_{\theta_a}(Y(a),y)]$ and $\hat \psi_{1,\theta_a}$ its plug-in estimator using $\hat \pi_a, \hat h_{a,\theta_a}$ and $\hat {\bb P}_{X}$.

\item  Define \(
\psi_{2,\theta_a}(O,y)
:=
\frac{\bm 1\{A=a\}}{p_a}\left(\chi_{\theta_a}(Y) - \chi_{\theta_a}(y)\right)
\) and $\hat\psi_{2,\theta_a}$ its plug-in estimator using $\hat {\bb P}_{X}$ and $\hat p_a$. 
\item  Define $\psi_{\theta_a}:=\psi_{1,\theta_a}+\psi_{2,\theta_a}$ and  $\hat \psi_{\theta_a}:=\hat \psi_{1,\theta_a}+\hat \psi_{2,\theta_a}$.

\item $\varphi_{\theta_a}: o \mapsto  \bb P_{Y \mid a}[\psi_{\theta_a}(o,\cdot)]$ is the EIF
of $m(\theta_a,\bb P)$ under $\bb P$ and $\hat\varphi_{\theta_a}: o \mapsto  \hat {\bb P}_{Y \mid a}[\hat \psi_{\theta_a}(o,\cdot)]$ is its plug-in estimator.

\item 
\(
\hat m_{\mathrm{DR}}(\theta_a)
:=
m(\theta_a,\hat{\bb P})
+
\bb P_n\left[\hat\varphi_{\theta_a}\right],
\) is the one-step estimator, where $\bb P_n$ is the empirical distribution on $\mc D_1$.
\end{itemize}
\end{definition}

For our asymptotic estimation results, we work under the following conditions.

\begin{assumption}[EIF Regularity Conditions]\label{ass:eif_regularity}
\hfill
\begin{enumerate}
\item \textbf{Overlap and stable inversion.}
There exists $\epsilon>0$ such that $\pi_a(X)\ge \epsilon$ almost surely.
Moreover, the estimator $\hat\pi_a$ is clipped so that
$\hat\pi_a(X)\ge \epsilon/2$ almost surely.

\item \textbf{Nondegenerate treatment mass.}
$p_a>0$ and $\hat p_a:=\bb P_n[\bm 1\{A=a\}]\to_p p_a$.

\item \textbf{Fold independence (cross-fitting).}
$(\hat p_a, \hat {\bb P}_{Y \mid a})$ are estimated using $\mc D_1$ (evaluation fold). $(\hat\pi_a,\hat h_{a,\theta_a}, \hat {\bb P}_X)$ are estimated using $\mc D_2$ (nuisance fold).

\item \textbf{Square integrability.} The nuisances satisfy $h_{a,\theta_a}, \hat h_{a,\theta_a}\in L_2(\bb P_X \otimes\bb P_Y) \cap L_2(\bb P_{X,Y})$ and the score function satisfies $r_{\theta_a}(y,y') \in L_2(\bb P_Y \otimes \bb P_Y) \cap L_2(\bb P_{Y,Y})$ for all $\theta_a \in \Theta$.
\end{enumerate}

\end{assumption}
\begin{remark}
The above assumptions imply that $\psi_{\theta_a},\ \hat\psi_{\theta_a}\in L_2(\bb P \otimes \bb P_{Y\mid a}) \quad\text{w.p.1,}$ and $\varphi_{\theta_a},\ \hat\varphi_{\theta_a}\in L_2(\bb P) \quad\text{w.p.1}$.
\end{remark}

\subsubsection{EIF of Moment Condition}

\begin{theorem}[Efficient influence function for the deconfounding flow-matching moment; formal version of \cref{thm:eif_fm}]\label{thm:eif_fm_app}
Fix $a\in\mathcal A$ and $\theta_a\in\Theta$. Suppose Assumption~\ref{ass:eif_regularity}(1) and \ref{ass:eif_regularity}(4) holds.
Then the moment functional $m(\theta_a,\bb P)$ defined in \Cref{def:eif_terms} is regular and pathwise
differentiable at $\bb P$, with efficient influence function
\[
\varphi_{\theta_a}(O)
=
\bb P_{Y\mid a}\!\left[\psi_{\theta_a}(O,\cdot)\right]
=
\int \psi_{\theta_a}(O,y)\,\bb P_{Y\mid a}(dy),
\]
where $\psi_{\theta_a}=\psi_{1,\theta_a}+\psi_{2,\theta_a}$ and $\psi_{1,\theta_a},\psi_{2,\theta_a}$
are as in \Cref{def:eif_terms}. Equivalently,
\[
\varphi_{\theta_a}(O)
=
\int \psi_{1,\theta_a}(O,y)\,\bb P_{Y\mid a}(dy)
\;+\;
\int \psi_{2,\theta_a}(O,y)\,\bb P_{Y\mid a}(dy).
\]
\end{theorem}

\begin{proof}[Proof of \Cref{thm:eif_fm}]
Fix $a$ and suppress the subscript $a$ on $\theta$ for readability. Now, recall that under the definition of $\bb P_{Y(a)}$ in Eq.~\eqref{eq:int_dist}
\[
m_a(\theta)=\nabla_{\theta}\mc L_a(\theta)
=\bb E\left[\bb E\left[g_\theta(Y)\mid A=a,X\right]\right] = \bb E\left[g_\theta\big(Y(a)\big)\right],
\qquad
g_\theta(y):=\int r_\theta(y,\tilde y)\, \bb P_{Y\mid a}(d\tilde y).
\]
To derive the EIF, it is helpful to write $m_a(\theta)$ as a functional of both
the counterfactual marginal $\bb P_{Y(a)}$ and the conditional observational law
$\bb P_{Y\mid a}$. Specifically, define the functional
\[
T_\theta(\bb P_1,\bb P_2)
:=\iint r_\theta(y,\tilde y)\,\bb P_1(dy)\,\bb P_2(d\tilde y),
\qquad (\bb P_1,\bb P_2)\in\mathcal P(\mc Y)\times\mathcal P(\mc Y),
\]
and note that for $P_1=\bb P_{Y(a)}$ and $P_2=\bb P_{Y\mid a}$, we have \(m_a(\theta)=T_\theta\big(\bb P_{Y(a)},\bb P_{Y\mid a}\big)\). Along any regular parametric submodel $\{\bb P_\varepsilon\}$ of the
observed-data law $\bb P$, define the induced measures
\[
\bb P_{1,\varepsilon}:=\bb P_{Y(a)}(\bb P_\varepsilon),
\qquad
\bb P_{2,\varepsilon}:=\bb P_{Y\mid a}(\bb P_\varepsilon),
\]
and write $m_a(\theta;\varepsilon):=T_\theta(\bb P_{1,\varepsilon},\bb P_{2,\varepsilon})$.
By the functional chain rule and Lemma~\ref{lem:nested},
\[
\left.\frac{d}{d\varepsilon}\right|_{\varepsilon=0} m_a(\theta;\varepsilon)
=
\left.\frac{d}{d\varepsilon}\right|_{\varepsilon=0}
T_\theta(\bb P_{1,\varepsilon},\bb P_{2,0})
+
\left.\frac{d}{d\varepsilon}\right|_{\varepsilon=0}
T_\theta(\bb P_{1,0},\bb P_{2,\varepsilon}),
\]
where $\bb P_{1,0}=\bb P_{Y(a)}$ and $\bb P_{2,0}=\bb P_{Y\mid a}$.

By a simple application of the chain-rule for EIFs (see Lemma~\ref{lem:nested}),
the EIF of $m_a(\theta)$ is given by the sum of the EIF of
$T(\cdot, \bb P_{Y\mid a})$ for fixed $\bb P_{Y\mid a}$ and the EIF of
$T(\bb P_{Y(a)}, \cdot)$ for fixed $\bb P_{Y(a)}$. Thus, the EIF of $m_a(\theta)$ is given by the sum of the EIF of $T(\cdot, \bb P_{Y|a})$ for fixed $\bb P_{Y|a}$ and the EIF of $T(\bb P_{Y(a)}, \cdot)$ for a fixed $\bb P_{Y(a)}$. To recover these EIFs, we introduce the linear functionals
\begin{align*}
\Lambda_1(\bb P_{Y(a)})
& =
\int g_\theta(y)\,\bb P_{Y(a)}(dy)
=
\bb E[g_\theta(Y(a))] \\
\Lambda_2(\bb P_{Y | A=a})
& =
\int \chi_\theta(\tilde y)\,\bb P_{Y\mid a}(d\tilde y)
=
\bb E[\chi_\theta(Y)\mid A=a].
\end{align*}
where $\chi_\theta(\tilde y):=\int r_\theta(y,\tilde y)\,\bb P_{Y(a)}(dy)$. Note that, for any probability measures $\bb P_1,\bb P_2\in\mc P(\mc Y)$,
\[
T_\theta(\bb P_1,\bb P_{Y\mid a})=\Lambda_1(\bb P_1),
\qquad
T_\theta(\bb P_{Y(a)},\bb P_2)=\Lambda_2(\bb P_2),
\]

The first functional $\Lambda_1(\bb P_{Y(a)})=\bb E[g_\theta(Y(a))]$ is a
counterfactual mean of the vector-valued outcome $g_\theta(Y)$.
Its efficient influence function is therefore
given by the standard AIPW expression for a counterfactual mean
(e.g., \citet{kennedy2024semiparametric}),
\[
\varphi_{\Lambda_1(\bb P_{Y(a)})}(O)
=
\frac{\bm 1\{A=a\}}{\pi_a(X)}\bigl(g_\theta(Y)-\mu_\theta(X)\bigr)
+\mu_\theta(X)-\bb E[g_\theta(Y(a))].
\]

The second functional $\Lambda_2(P_{Y\mid a})=\bb E[\chi_\theta(Y)\mid A=a]$
is a stratum mean of the scalar outcome $\chi_\theta(Y)$. Its efficient
influence function is therefore
(e.g., \citet{kennedy2024semiparametric})
\[
\varphi_{\Lambda_2(P_{Y\mid a})}(O)
=
\frac{\bm 1\{A=a\}}{p_a}
\Bigl(\chi_\theta(Y)-\bb E[\chi_\theta(Y)\mid A=a]\Bigr),
\qquad p_a:=\bb P(A=a).
\]
As mentioned above, by the chain rule (Lemma~\ref{lem:nested}), the efficient influence function of
$m_a(\theta)=T_\theta(\bb P_{Y(a)},\bb P_{Y\mid a})$ is the sum of these two
components. Using that
$\bb E[g_\theta(Y(a))]=m_a(\theta)$ and
$\bb E[\mu_\theta(X)]=\bb E[g_\theta(Y(a))]$,
we obtain
\[
\varphi_{\theta,a}(O)
=
\frac{\bm 1\{A=a\}}{p_a}
\Bigl(\chi_\theta(Y)-\bb E[\chi_\theta(Y)\mid A=a]\Bigr)
+
\frac{\bm 1\{A=a\}}{\pi_a(X)}\bigl(g_\theta(Y)-\mu_\theta(X)\bigr)
+\mu_\theta(X)-\bb E[\mu_\theta(X)],
\]
which is the expression stated in \Cref{thm:eif_fm}.

\end{proof}

\subsubsection{Asymptotic Analysis of Debiased (One-Step) Estimator in \cref{sec:est:debiased}}

\begin{theorem}[Asymptotic expansion and efficiency of the one-step estimator]
\label{thm:efficiency}
Under Assumption~\ref{ass:eif_regularity},  the following expansion holds:
\begin{equation}
\label{eq:asymp_exp_hatphi}
\hat m_{\mathrm{DR}}(\theta_a)-m(\theta_a,\bb P)
=
\underbrace{(\bb P_n-\bb P)\big[\varphi_{\theta_a}(\cdot;\bb P)\big]}_{=:T_1}
+
\underbrace{(\bb P_n-\bb P)\big[\hat\varphi_{\theta_a,1}-\varphi_{\theta_a,1}(\cdot;\bb P)\big]}_{=:T_2}
+
R_n(\theta_a),
\end{equation}
where the von Mises remainder satisfies
\begin{equation}
\label{eq:remainder_bound_wrapper}
R_n(\theta_a)
=
O_p\left(
\|\hat\pi_a-\pi_a\|_{L_2(\bb P_X)}\;
\|\hat h_{a,\theta_a}-h_{a,\theta_a}\|_{L_2(\bb P_X\otimes\bb P_{Y\mid a})}
\right).
\end{equation}
and the empirical process drift term $T_2$ satisfies
\begin{equation}
\label{eq:T2_rate}
T_2
=
O_p\left(
\frac{
\|\hat h_{a,\theta_a}-h_{a,\theta_a}\|_{L_2(\bb P_X\otimes \bb P_{Y\mid a})}
+
\|\hat\pi_a-\pi_a\|_{L_2(\bb P_X)}
}{\sqrt n}
\right),
\end{equation}
\end{theorem}

\begin{proof}
Since all analysis is for a fixed $a, \theta_a$, in what follows we suppress the dependence of all quantities on these objects notationally. 

To start, note that by the definition of the one-step estimator $\hat m$ we can write 
\begin{align}
    m(\bb P) - \hat m & = m(\bb P) - m(\hat {\bb P}) - \bb P_n(\hat \varphi) \nonumber \\
     & = m(\bb P) - m(\hat {\bb P}) - \bb P_n(\hat \varphi_1) - \bb P_n(\hat \varphi_2)  \label{eq:onestep_error}
\end{align}
Now, first note that $\bb P_n(\hat \varphi_2)=0$. To see this, recall that
$\hat \varphi_2 : o \mapsto \hat{\bb P}_{Y\mid A=a}[\hat\psi_2(o,\cdot)]$, where
\(
\hat \psi_2(o',y)
=
\frac{\bm 1\{a'=a\}}{\hat p_a}\big(\hat \chi(y')-\hat \chi(y)\big)
\). Therefore,
\[
\bb P_n(\hat \varphi_2)
=
\bb P_n\left[
\hat{\bb P}_{Y\mid A=a}\left(
\frac{\bm 1\{A=a\}}{\hat p_a}\big(\hat \chi(Y)-\hat \chi(\cdot)\big)
\right)
\right].
\]
By Fubini's theorem and the definition of $\hat{\bb P}_{Y\mid A=a}$,
\[
\bb P_n(\hat \varphi_2)
=
\hat{\bb P}_{Y\mid A=a}\left[
\frac{1}{\hat p_a}\bb P_n\left[\bm 1\{A=a\}\hat \chi(Y)\right]
-
\frac{1}{\hat p_a}\bb P_n[\bm 1\{A=a\}]\hat \chi(y)
\right].
\]
But $\bb P_n[\bm 1\{A=a\}]=\hat p_a$ and
\(
\hat{\bb P}_{Y\mid A=a}[\hat \chi]
= \frac{1}{\hat p_a}\bb P_n[\bm 1\{A=a\}\hat \chi(Y)].
\)
Hence the two terms cancel exactly.

Now, adding and subtracting $\bb P(\hat \varphi_1)$ and $(\bb P_n - \bb P)(\varphi_1 + \varphi_2)$ to Eq.~\eqref{eq:onestep_error} we get 

\[\hat m - m(\bb P) = \underbrace{(\bb P_n - \bb P)(\varphi_1 + \varphi_2)}_{=:T_1} + \underbrace{(\bb P_n - \bb P)(\hat \varphi_1 - \varphi_1)}_{=:T_2} +R_2\]
 where
 \[R_2 := m(\hat {\bb P})  + \bb P(\hat \varphi_1) - m(\bb P) - (\bb P_n - \bb P)(\varphi_2)\]

We now analyze each term.

\textbf{\underline{Empirical Process Term $\bm T_2$}}

For $o'=(x',a',y')\in\mc X\times\mc A\times\mc Y$ and $y\in\mc Y$, define
\[
\Delta(o',y)
:=
\hat\psi_1(o',y)-\psi_1(o',y).
\]

All statements below are conditional on the training fold. In particular,
$\hat h$ and $\hat\pi$ (and hence $\hat\psi_1$ and $\Delta$) are deterministic,
while all expectations are taken with respect to the true law $\bb P$.

By the square–integrability assumptions in \cref{ass:eif_regularity},
we may apply the sample–splitting empirical–process bound established in
\cref{lem:emp_base_plugin_ratio} (Appendix~\ref{app:math:lemmas}), which yields
\begin{equation}
\label{eq:reduce_to_Delta}
(\bb P_n-\bb P)(\hat\varphi-\varphi)
=
O_p\left(
\frac{
\left(
\bb E\left[\int \|\Delta(O,y)\|^2\,\bb P_{Y\mid a}(dy)\right]
\right)^{1/2}
}{\sqrt n}
\right).
\end{equation}
It therefore suffices to bound the $L_2(\bb P\otimes\bb P_{Y\mid a})$ norm of $\Delta$
in terms of the nuisance estimation errors.

For $O=(X,A,Y)$ and $y\in\mc Y$, we have the exact decomposition
\begin{align*}
\Delta(O,y)
&=
\left[
\frac{\bm 1\{A=a\}}{\hat\pi(X)}\big(r(Y,y)-\hat h(X,y)\big)+\hat h(X,y)
\right] \nonumber \\
& \qquad -
\left[
\frac{\bm 1\{A=a\}}{\pi(X)}\big(r(Y,y)-h(X,y)\big)+h(X,y)
\right] \\
&=
\underbrace{\big(\hat h(X,y)-h(X,y)\big)}_{=:D_h(X,y)}
+
\underbrace{
\bm 1\{A=a\}
\left(\frac{1}{\hat\pi(X)}-\frac{1}{\pi(X)}\right)
\big(r(Y,y)-h(X,y)\big)
}_{=:D_\pi(O,y)} \nonumber \\
& \qquad -
\underbrace{
\bm 1\{A=a\}\frac{1}{\hat\pi(X)}\big(\hat h(X,y)-h(X,y)\big)
}_{=:D_{\pi h}(O,y)} .
\end{align*}

By Assumption~\ref{ass:eif_regularity}(i), $\hat\pi$ is clipped so that
$\hat\pi(X)\ge \epsilon/2$ almost surely. Consequently,
\[
\frac{1}{\hat\pi(X)}\le \frac{2}{\epsilon},
\qquad
\left\|\frac{1}{\hat\pi(X)}-\frac{1}{\pi(X)}\right\|
\le
\frac{2}{\epsilon^2}\,\|\hat\pi(X)-\pi(X)\|
\quad\text{a.s.}
\]

Using $\|u+v+w\|^2\le 3(\|u\|^2+\|v\|^2+\|w\|^2)$, we obtain
\begin{align}
\label{eq:Delta_split}
\bb E\left[\int \|\Delta(O,y)\|^2\,\bb P_{Y\mid a}(dy)\right]
\le
3\sum_{\star\in\{h,\pi,\pi h\}}
\bb E\left[\int \|D_\star(O,y)\|^2\,\bb P_{Y\mid a}(dy)\right].
\end{align}
We bound each term in turn.

First, by definition we have
\begin{equation}
\label{eq:Dh_term}
\bb E\left[\int \|D_h(X,y)\|^2\,\bb P_{Y\mid a}(dy)\right]
=
\|\hat h-h\|_{L_2(\bb P_X\otimes \bb P_{Y\mid a})}^2.
\end{equation}

Second, using the uniform bound on $1/\hat\pi$,
\begin{equation}
\label{eq:Dpih_term}
\bb E\left[\int \|D_{\pi h}(O,y)\|^2\,\bb P_{Y\mid a}(dy)\right]
\le
\frac{4}{\epsilon^2}\,
\|\hat h-h\|_{L_2(\bb P_X\otimes \bb P_{Y\mid a})}^2 .
\end{equation}

Finally, for the propensity term, note that
\[
\|r(Y,y)-h(X,y)\|^2 \le 2\|r(Y,y)\|^2+2\|h(X,y)\|^2,
\]
and therefore
\begin{align}
\label{eq:Dpi_term}
& \bb E\left[\int \|D_\pi(O,y)\|^2\,\bb P_{Y\mid a}(dy)\right]
\nonumber \\
& \le
\frac{8}{\epsilon^4}\,
\bb E\left[
\|\hat\pi(X)-\pi(X)\|^2
\bm 1\{A=a\}
\int\big(\|r(Y,y)\|^2+\|h(X,y)\|^2\big)\,
\bb P_{Y\mid a}(dy)
\right] \\
&\le
C_r\,\|\hat\pi-\pi\|_{L_2(\bb P_X)}^2,
\nonumber
\end{align}
for a finite constant $C_r$ depending only on $\epsilon$ and the second moments
of $r$ and $h$, where the last step uses Cauchy--Schwarz and
Assumption~\ref{ass:eif_regularity}(iv).

Substituting Eq.~\eqref{eq:Dh_term}, Eq.~\eqref{eq:Dpih_term}, and Eq.~\eqref{eq:Dpi_term}
into Eq.~\eqref{eq:Delta_split} yields
\[
\bb E\left[\int \|\Delta(O,y)\|^2\,\bb P_{Y\mid a}(dy)\right]
\le
C_1\,\|\hat h-h\|_{L_2(\bb P_X\otimes \bb P_{Y\mid a})}^2
+
C_2\,\|\hat\pi-\pi\|_{L_2(\bb P_X)}^2,
\]
for finite constants $C_1,C_2$. Taking square roots gives
\begin{equation}
\label{eq:Delta_final_bound}
\left(
\bb E\left[\int \|\Delta(O,y)\|^2\,\bb P_{Y\mid a}(dy)\right]
\right)^{1/2}
\le
C_1'\,\|\hat h-h\|_{L_2(\bb P_X\otimes \bb P_{Y\mid a})}
+
C_2'\,\|\hat\pi-\pi\|_{L_2(\bb P_X)} .
\end{equation}

Combining Eq.~\eqref{eq:reduce_to_Delta} and Eq.~\eqref{eq:Delta_final_bound} yields
\[
(\bb P_n-\bb P)(\hat\varphi-\varphi)
=
O_p\left(
\frac{
\|\hat h-h\|_{L_2(\bb P_X\otimes \bb P_{Y\mid a})}
+
\|\hat\pi-\pi\|_{L_2(\bb P_X)}
}{\sqrt n}
\right),
\]
which proves the stated bound for $T_2$.

\textbf{\underline{Part II: Remainder $\bm R_2$}}
We now control the remainder term,
\[
R_2
:=
m(\hat{\bb P})-m(\bb P)+\bb P(\hat\varphi_1) - (\bb P_n - \bb P)(\varphi_2),
\]
From the definition of $\hat \psi_1$ in \cref{def:eif_terms},  we can express this term as 
\[
\hat\varphi_1(O)
=
\frac{\bm 1\{A=a\}}{\hat\pi(X)}
\big(\hat g(Y) - \hat\mu_{\hat h}(X)
\big)
+
\hat\mu_{\hat h}(X)
-
m(\hat{\bb P}).
\]
where  \(\hat\mu_{\hat h}(x) := \hat{\bb P}_{Y\mid A=a}[\hat h(x,\cdot)]\) and \(\hat g(y) = \hat {\bb P}_{Y|a}(r(y,\cdot))\).

Taking the expectation of $\hat \varphi_1$ under the true law $\bb P$,  we get
\[
\bb P(\hat\varphi_1)
=
\bb P\left[
\frac{\pi}{\hat\pi}
\big(
\hat \mu_{ h}-\hat\mu_{\hat h}
\big)
\right]
+
\bb P[\hat\mu_{\hat h}]
-
m(\hat{\bb P}),
\]
where $\hat \mu_{h}(x):=\hat {\bb P}_{Y\mid A=a}[h(x,\cdot)]$. 

Substituting this into the definition of $R_2$ we get
\[
R_2
=
\bb P\left[
\frac{\pi}{\hat\pi}
(\hat \mu_{ h}-\hat\mu_{\hat h})
\right]
+
\bb P[\hat\mu_{\hat h} - \mu_{h}] - (\bb P_n - \bb P)(\varphi_2)
\]

Adding and subtracting $\bb P[\hat \mu_{ h}]$ isolates
(i) the usual doubly-robust remainder around the random target $\mu_{\hat h}$ and
(ii) a drift due to replacing $\bb P_{Y\mid a}$ by $\hat {\bb P}_{Y \mid a}$:
\[
R_2
=
\bb P\underbrace{\left[
\frac{\hat\pi-\pi}{\hat\pi}
(\hat\mu_{\hat h}-\hat \mu_{ h})
\right]}_{B_1}
\;+\;
\underbrace{\bb P(\hat \mu_{h}-\mu_h) - (\bb P_n - \bb P)(\varphi_2)}_{B_2}.
\]

\paragraph{ \underline{II(a) Product term $B_1$}}

Define
\[
w(x):=\frac{\hat\pi(x)-\pi(x)}{\hat\pi(x)},
\qquad
\Delta_h(x,y):=\hat h(x,y)-h(x,y).
\]
By definition of $\hat \mu_{h}$ and $\hat\mu_{\hat h}$,
\[
\hat\mu_{\hat h}(x)-\hat\mu_{h}(x)
=
\int \Delta_h(x,y)\,\hat{\bb P}_{Y\mid A=a}(dy).
\]
Substituting into the remainder yields the exact product measure form
\[
\bb P\left[
\frac{\hat\pi-\pi}{\hat\pi}
(\hat\mu_{\hat h}-\hat \mu_{ h})
\right]
=
(\bb P_X\otimes \hat{\bb P}_{Y\mid A=a})[w\,\Delta_h].
\]

Adding and subtracting $(\bb P_X\otimes \bb P_{Y\mid a})[w\Delta_h]$ gives
\begin{align}
(\bb P_X\otimes \hat{\bb P}_{Y\mid A=a})[w\Delta_h]
&=
\underbrace{(\bb P_X\otimes \bb P_{Y\mid a})[w\Delta_h]}_{C_1}
+
\underbrace{(\bb P_X\otimes (\hat{\bb P}_{Y\mid A=a}-\bb P_{Y\mid a}))[w\Delta_h]}_{C_2}.
\label{eq:split_pop_stoch}
\end{align}

\paragraph{\underline{II(a).1 Population term $\bm C_1$}}
By Cauchy--Schwarz, the first term can be bounded by
\[
\big|(\bb P_X\otimes \bb P_{Y\mid a})[w\Delta_h]\big|
\le
\|w\|_{L_2(\bb P_X)}\;
\|\Delta_h\|_{L_2(\bb P_X\otimes \bb P_{Y\mid a})}.
\]

To bound $\|w\|_{L_2(\bb P_X)}$, note that by Assumption~\ref{ass:eif_regularity}(1),
$\hat\pi(x)\ge \epsilon/2$ almost surely. Conditioning on this measure-one event lets us write
\begin{align}
\|w\|_{L_2(\bb P_X)}
\le
\frac{2}{\epsilon}\,
\|\hat\pi-\pi\|_{L_2(\bb P_X)}. \label{eq:w_bound}
\end{align}

Combining the above bounds yields
\[
\big|(\bb P_X\otimes \bb P_{Y\mid a})[w\Delta_h]\big|
\le
\frac{2}{\epsilon}\,
\|\hat\pi-\pi\|_{L_2(\bb P_X)}\;
\|\hat h-h\|_{L_2(\bb P_X\otimes \bb P_{Y\mid a})}
\]

In particular,
\[
C_1
=
O_p\left(
\|\hat\pi-\pi\|_{L_2(\bb P_X)}\;
\|\hat h-h\|_{L_2(\bb P_X\otimes \bb P_{Y\mid a})}
\right)
\]

\paragraph{\underline{II(a).2 Stochastic base term $\bm C_2$}}
Define the (training-fold measurable) function
\[
F(\tilde y):=\bb E\big[w(X)\Delta_h(X,\tilde y)\big]=\bb P_X[w(\cdot)\Delta_h(\cdot,\tilde y)].
\]
Substituting this definition of $C_2$ and applying Fubini's Theorem, we get
\begin{align*}
C_2 = (\bb P_X\otimes (\hat{\bb P}_{Y\mid A=a}-\bb P_{Y\mid a}))[w\Delta_h]
& =
(\hat{\bb P}_{Y\mid A=a}-\bb P_{Y\mid a})F \\
& =
\frac{1}{\hat p_a}\,(\bb P_n-\bb P)\big[\bm 1_a F\big]
\\
& = \frac{1}{ p_a}\,(\bb P_n-\bb P)\big[\bm 1_a F\big]
+
\Big(\frac{1}{\hat p_a}-\frac{1}{p_a}\Big)\bb P\big[\bm 1_aF\big], 
\end{align*}
where the second last equality used the identities
\(
\hat{\bb P}_{Y\mid A=a}[F]
=
\frac{\bb P_n[\bm 1_aF]}{\hat p_a}\), \(
\bb P_{Y\mid a}[F]
=
\frac{\bb P[\bm 1_aF]}{p_a}.
\) and the last equality used $1/\hat p_a = 1/p_a + (1/\hat p_a - 1/p_a)$. 

To control the first term, note that, conditional on the training fold, $\bm 1\{A=a\}F(Y)$ is a fixed
square-integrable function.
Indeed, for each fixed $\tilde y$, Cauchy--Schwarz in $L_2(\bb P_X)$ yields
\[
\|F(\tilde y)\|
=
\left\|\int w(x)\Delta_h(x,\tilde y)\,\bb P_X(dx)\right\|
\le
\|w\|_{L_2(\bb P_X)}\;
\|\Delta_h(\cdot,\tilde y)\|_{L_2(\bb P_X)}.
\]
Squaring and integrating with respect to $\tilde y\sim\bb P_{Y\mid a}$ gives
\begin{align}
\|F\|^2_{L_2(\bb P_{Y\mid a})}
 = \int \|F(\tilde y)\|^2\,\bb P_{Y\mid a}(d\tilde y)
\le
\|w\|_{L_2(\bb P_X)}^2\;
\|\Delta_h\|_{L_2(\bb P_X\otimes \bb P_{Y\mid a})}^2,
\label{eq:F_bound}\end{align}

Since $\|F\|_{L_2} \geq \|1_a F\|_{L_2}$, we can apply \cref{lem:rand_func_L2}, to bound  $(\bb P_n-\bb P)\big[\bm 1_aF \big]$ asymptotically as
\[
(\bb P_n-\bb P)\big[\bm 1_aF \big]
=
O_p\left(\frac{\|\bm 1\{A=a\}F(Y)\|_{L_2(\bb P)}}{\sqrt n}\right)
=
O_p\left(\frac{\|F\|_{L_2(\bb P_{Y\mid a})}}{\sqrt n}\right).
\]

It remains to control the correction term involving $(1/\hat p_a-1/p_a)$. Since $\bb P[\bm 1_a F]
=
p_a\,\bb P_{Y\mid a}[F]$ and \(
(1/\hat p_a-1/p_a)=O_p(n^{-1/2}),
\) by Jensen's equality  we obtain
\[
\Big(\frac{1}{\hat p_a}-\frac{1}{p_a}\Big)\bb P\big[\bm 1_aF\big] \leq 
\Big(\frac{1}{\hat p_a}-\frac{1}{p_a}\Big)\bb P_{Y\mid a}[F]
\le
\Big(\frac{1}{\hat p_a}-\frac{1}{p_a}\Big)\|F\|_{L_2(\bb P_{Y\mid a})}
\]
By the Central Limit Theorem $\hat p_a-p_a=O_p(n^{-1/2})$  and
$1/\hat p_a=O_p(1)$, so  $(1/\hat p_a-1/p_a)=O_p(n^{-1/2})$ (the latter follows from $\frac 1x - \frac 1y = \frac{1}{xy}(x-y)$). As a consequence we obtain
\[
\Big(\frac{1}{\hat p_a}-\frac{1}{p_a}\Big)\bb P[\bm 1_a F]
=
O_p\left(
\frac{
\|F\|_{L_2(\bb P_{Y\mid a})}
}{\sqrt n}
\right).
\]
Summing the $O_p$ bounds for both terms and applying the inequalities in Eq.~\eqref{eq:F_bound} and Eq.~\eqref{eq:w_bound} gives
\[
C_2
=
O_p\left(
\frac{\|\hat\pi-\pi\|_{L_2(\bb P_X)}\;\|\hat h-h\|_{L_2(\bb P_X\otimes \bb P_{Y\mid a})}}{\sqrt n}
\right)
\]

Putting this together with $C_1$, we get 
\[
B_1
=
O_p\left(
\|\hat\pi-\pi\|_{L_2(\bb P_X)}\;
\|\hat h-h\|_{L_2(\bb P_X\otimes \bb P_{Y\mid a})}
\right)
\]

\paragraph{II(b): Drift term $\bm B_2$.}
Applying Fubini/Tonelli, we can write the first part of the drift term as
\begin{align}
\bb P(\hat\mu_h-\mu_h)
&=
\bb P_X\Big[\hat{\bb P}_{Y\mid A=a}\big[h(X,\cdot)\big]-\bb P_{Y\mid a}\big[h(X,\cdot)\big]\Big]\notag\\
&=
\hat{\bb P}_{Y\mid A=a}\big[\bb P_X[h(\cdot,\cdot)]\big]
-
\bb P_{Y\mid a}\big[\bb P_X[h(\cdot,\cdot)]\big]\notag\\
&=
\hat{\bb P}_{Y\mid A=a}[g]-\bb P_{Y\mid a}[g].
\label{eq:B2_as_conditional_plugin}
\end{align}
Now, by definition \cref{def:eif_terms},
\[
\varphi_2(O)
=
\frac{\bm 1\{A=a\}}{p_a}\Big(\chi(Y)-\bb P_{Y\mid a}[\chi]\Big),
\qquad \bb P[\varphi_2]=0 .
\]
Hence
\begin{align*}
(\bb P_n-\bb P)(\varphi_2)
&=
\bb P_n(\varphi_2)\notag\\
&=
\frac{1}{p_a}\Big(
\bb P_n[\bm 1\{A=a\}\chi(Y)]
-
\bb P_n[\bm 1\{A=a\}]\;\bb P_{Y\mid a}[\chi]
\Big)\notag\\
&=
\frac{\hat p_a}{p_a}\,\hat{\bb P}_{Y\mid A=a}[\chi]
-
\frac{\hat p_a}{p_a}\,\bb P_{Y\mid a}[\chi].
\end{align*}

Using these definitions we get
\begin{align*}
B_2 = \bb P(\hat\mu_h-\mu_h)
-
(\bb P_n-\bb P)(\varphi_2)
&=
\Big(1-\frac{\hat p_a}{p_a}\Big)
\big(\hat{\bb P}_{Y\mid A=a}[\chi]-\bb P_{Y\mid a}[\chi]\big)\notag\\
&=
-\frac{\hat p_a-p_a}{p_a}\,
\big(\hat{\bb P}_{Y\mid A=a}[\chi]-\bb P_{Y\mid a}[\chi]\big).
\end{align*}

Now, by the CLT for the mean of $A$ and conditional mean of $Y$, we have $\hat p_a-p_a=O_p(n^{-1/2})$ and \(
\hat{\bb P}_{Y\mid A=a}[\chi]-\bb P_{Y\mid a}[\chi]
=
O_p(n^{-1/2})
\)
 (since $\chi\in L_2(\bb P_{Y\mid a})$).
Therefore,
\[
\bb P(\hat\mu_h-\mu_h)
-
(\bb P_n-\bb P)(\varphi_2)
=
O_p(n^{-1}).
\tag{B2.4}
\]

Putting this together with $B_1$ yields the stated result for $R_2$ and completes the proof.

\end{proof}

\begin{corollary}[Efficiency bound in the main text]
\label{corr:efficiency}
Suppose the conditions of \Cref{thm:efficiency} hold. In addition, assume
that for each $\theta_a\in\Theta$ the score kernel $r_{\theta_a}$ is uniformly
Lipschitz in its first argument, i.e.\ there exists $L<\infty$ such that for all
$\tilde y\in\mc Y$ and all probability measures $\bb P_1,\bb P_2\in\mc P(\mc Y)$,
\[
\Big|
\int r_{\theta_a}(y,\tilde y)\,\bb P_1(dy)
-
\int r_{\theta_a}(y,\tilde y)\,\bb P_2(dy)
\Big|
\;\le\;
L\,W_1(\bb P_1,\bb P_2).
\]
Then the one-step estimator satisfies
\[
\hat m_{\mathrm{DR}}(\theta_a) - m(\theta_a)
=
O_p(n^{-\frac12})
+
O_p\!\left(
n^{-\frac 12}\big(
\|\hat\pi_a-\pi_a\|_{L_2(\bb P_X)}
+
\Delta_1
\big)
+
\|\hat\pi_a-\pi_a\|_{L_2(\bb P_X)}\,\Delta_1
\right).
\]
 where \(
\Delta_1
:=
\Big\|
W_1\!\big(
\hat{\bb P}_{Y\mid a,X},
\bb P_{Y\mid a,X}
\big)
\Big\|_{L_2(\bb P_X)}.
\)
\end{corollary}

\begin{proof}
By the assumed Lipschitz property of $r_{\theta_a}$, for all $x$ and
$\tilde y$,
\[
\big|
\hat h_{a,\theta_a}(x,\tilde y)-h_{a,\theta_a}(x,\tilde y)
\big|
\le
L\,
W_1\!\big(
\hat{\bb P}_{Y\mid a,X=x},
\bb P_{Y\mid a,X=x}
\big).
\]
Squaring both sides and integrating with respect to
$\bb P_X\otimes\bb P_{Y\mid a}$ yields
\[
\|\hat h_{a,\theta_a}-h_{a,\theta_a}\|_{L_2(\bb P_X\otimes\bb P_{Y\mid a})} \leq L
\Delta_1
\]
The result then follows by substituting this bound into the expansion of
\Cref{thm:efficiency}.
\end{proof}

\end{document}